\documentclass[11pt]{article}

\usepackage[top=0.95in,bottom=0.95in,left=0.95in,right=0.95in]{geometry}

\PassOptionsToPackage{numbers, compress}{natbib}

\usepackage[utf8]{inputenc}

\usepackage[utf8]{inputenc} 
\usepackage[T1]{fontenc}    
\usepackage{url}            
\usepackage{booktabs}       
\usepackage{amsfonts}       
\usepackage{nicefrac}       
\usepackage{microtype}      

\usepackage{subfigure}
\usepackage{algorithm}
\usepackage{algorithmic}

\usepackage[utf8]{inputenc} 
\usepackage[T1]{fontenc}    
\usepackage[colorlinks=true,citecolor=ForestGreen]{hyperref}       
\usepackage{url}            
\usepackage{booktabs}       
\usepackage{amsfonts}       
\usepackage{nicefrac}       
\usepackage{microtype}      

\usepackage{natbib}

\usepackage{float}
\usepackage{graphicx}
\usepackage[dvipsnames]{xcolor}
\usepackage[export]{adjustbox}
\usepackage[inline]{enumitem}

\usepackage{multirow}

\usepackage{tikz}
\usetikzlibrary{positioning}

\usepackage{amsmath,amssymb,amsthm}
\usepackage{newtxtext}
\usepackage[mathscr]{eucal}
\usepackage{stmaryrd}
\usepackage{accents}
\usepackage{upgreek}

\usepackage[most]{tcolorbox}
\usepackage{setspace}

\usepackage{booktabs,colortbl,multirow}

\usepackage{pifont}

\graphicspath{{figs//}}

\allowdisplaybreaks

\newcommand{\defEq}{\stackrel{.}{=}}

\newcommand{\argmax}{{\operatorname{argmax }}}

\renewcommand{\Pr}{\mathbb{P}}
\newcommand{\E}[2]{\underset{#1}{\mathbb{E}}\left[ #2 \right]}

\newcommand{\R}{\mathbf{R}}

\newcommand{\XCal}{\mathscr{X}}
\newcommand{\YCal}{\mathscr{Y}}

\newcommand{\Real}{\mathbb{R}}

\newcommand{\costin}{c_{\rm in}}
\newcommand{\costout}{c_{\rm out}}
\newcommand{\costFN}{c_{\rm fn}}

\newcommand{\PTr}{\Pr_{\mathrm{in}}}
\newcommand{\piTr}{\pi_{\mathrm{in}}}
\newcommand{\PTe}{\Pr_{\mathrm{te}}}

\newcommand{\cX}{\mathcal{X}}

\newcommand{\1}{\mathbf{1}}
\renewcommand{\R}{\mathbb{R}}

\newcommand{\Ex}{\mathbb{E}}

\newcommand{\Pin}{\Pr_{\rm in}}
\newcommand{\Pcomb}{\Pr_{\rm comb}}

\newcommand{\abstain}{\perp}

\theoremstyle{plain}
\newtheorem{theorem}{Theorem}[section]

\newtheorem{lemma}[theorem]{Lemma}

\theoremstyle{definition}

\theoremstyle{remark}

\theoremstyle{remark}
\newtheorem{example}[theorem]{Example}

\usepackage[textsize=tiny]{todonotes}

\usepackage[toc,page,header]{appendix}
\usepackage{minitoc}

\begin{document}

\title{Plugin estimators for selective classification\\with out-of-distribution detection}

%


\author{
Harikrishna Narasimhan\\
Google Research, Mountain View\\
\tt{hnarasimhan@google.com}
\and
Aditya Krishna Menon\\
Google Research, New York\\
\tt{adityakmenon@google.com}
\and
Wittawat Jitkrittum\\
Google Research, New York\\
\tt{wittawat@google.com}
\and
Sanjiv Kumar\\
Google Research, New York\\
\tt{sanjivk@google.com}
}
\maketitle

\doparttoc 
\faketableofcontents 

\part{} 

\begin{abstract}
Real-world classifiers can benefit from 
the option of
\emph{abstaining} 
from predicting on samples where 
they have
low confidence.
Such abstention is 
particularly useful
on samples which are close to the learned decision boundary,
or which are outliers with respect to the training sample.
These settings have been the subject of extensive but disjoint study in the 
\emph{selective classification} (\emph{SC})
and \emph{out-of-distribution} (\emph{OOD}) detection literature.
Recent work on \emph{selective classification with OOD detection} (\emph{SCOD}) has argued for the unified study of these problems;
however,
the formal underpinnings of this problem are still nascent,
and existing techniques are heuristic in nature.
In this paper, we 
propose new 
plugin estimators
for SCOD that are 
theoretically grounded,
effective, 
and generalise existing approaches from the SC and OOD detection literature.
In the course of our analysis, we formally explicate how na\"{i}ve use of existing SC and OOD detection baselines may be inadequate for SCOD.
We empirically demonstrate that our approaches yields competitive SC and OOD detection performance 
compared to baselines from both literatures.



\end{abstract}

\section{Introduction}
Given a training sample drawn i.i.d.\ from a distribution $\Pr_{\rm in}$ (e.g., images of cats and dogs),
the standard classification paradigm concerns learning a classifier that accurately predicts the label for test samples drawn from $\Pr_{\rm in}$.
However, in real-world classifier deployment, one may encounter \emph{out-of-distribution} (\emph{OOD}) test samples, i.e., samples drawn from some distinct distribution $\Pr_{\rm out} \neq \Pr_{\rm in}$ (e.g., images of aeroplanes).
\emph{Out-of-distribution detection} is the problem of accurately identifying such OOD samples, and has received considerable study of late~\citep{Hendrycks:2017,Lee:2018,Hendrycks:2019,Ren:2019,Huang:2021,Huang:2021b,Thulasidasan:2021,wang2022vim,Bitterwolf:2022,Katz:2022,Wei:2022,Sun:2022,Hendrycks:2022}.
An accurate OOD detector
allows one to \emph{abstain} from making a prediction on OOD samples, 
rather than making an egregiously incorrect prediction;
this yields more reliable and trust-worthy classifiers.

The quality of an OOD detector is typically assessed by its ability to distinguish in-distribution (ID) versus OOD samples.
However,
some recent works~\citep{Kim:2021,Xia:2022,Cen:2023} 
argued that to more accurately capture the real-world deployment of OOD detectors, 
it is more natural to consider distinguishing 
\emph{correctly-classified ID}
versus
\emph{OOD and misclassified ID} samples.
Indeed,
it is intuitive for a classifier 
to abstain from predicting on ``hard'' (e.g., ambiguously labelled) ID samples which are likely to be misclassified.
This problem is termed 
\emph{unknown detection} (\emph{UD}) in~\citet{Kim:2021},
and
\emph{selective classification with OOD detection} (\emph{SCOD}) in~\citet{Xia:2022};
we adopt the latter in the sequel.

One may view SCOD as a unification of 
OOD detection
and the classical 
\emph{selective classification} (\emph{SC})
paradigm~\citep{Chow:1970,Bartlett:2008,Cortes:2016,Ramaswamy:2018,Thulasidasan:2019,Ni:2019,Charoenphakdee:2021}.
Both OOD detection and SC have 
well-established formal underpinnings, 
with accompanying principled techniques ~\citep{Bitterwolf:2022,Cortes:2016,Ramaswamy:2018};
however, by contrast,
the understanding of SCOD is still nascent.
In particular,
existing SCOD approaches either employ 
OOD detection baselines~\citep{Kim:2021},
or heuristic design choices~\citep{Xia:2022}.
It remains unclear if 
there are conditions where such approaches may fail, 
and whether there are effective,
theoretically grounded alternatives.

In this paper,
we provide a statistical 
formulation for
the SCOD problem,
and design two
novel plug-in estimators for SCOD that 
operate under different assumptions on available data during training
(Table~\ref{tbl:summary}).
The first estimator addresses
the challenging setting where 
one has
access to \emph{only} ID data, 
and leverages existing techniques for SC and OOD detection in a \emph{black-box} manner.
The second estimator 
addresses the setting
where one additionally access to a ``wild'' sample comprising a mixture of both ID and OOD data~\citep{Katz:2022},
and involves the design of novel \emph{loss functions} 
{with consistency guarantees}.
Both estimators generalise existing approaches from the SC and OOD detection literature, 
and thus offer a unified means of reasoning about both problems.
In sum, our contributions are:
\begin{enumerate}[label=(\roman*),leftmargin=16pt,itemsep=0pt,topsep=0pt]
    \item 
    We provide a statistical formulation 
    for SCOD
    that unifies both the SC and OOD detection problems (\S\ref{sec:relation}), 
    and derive
    the corresponding
    Bayes-optimal solution (Lemma \ref{lem:scod-bayes}).
    Intriguingly this solution is
    a variant of the 
    popular maximum softmax probability baseline for SC and OOD detection~\citep{Chow:1970,Hendrycks:2017},
    using 
    a \emph{sample-dependent} rather than constant threshold.
    
    
    \item  
    Based on the form of the Bayes-optimal solution,
    we propose two new plug-in approaches for SCOD (\S\ref{sec:formulation}).
    These operate in settings with 
    access to only ID data (\S\ref{sec:plug-in-black-box}), and 
    access to 
    a mixture of
    ID and OOD data (\S\ref{sec:plug-in-losses})
    respectively, and 
    generalise existing SC and OOD detection techniques.
    
    \item 
    Experiments on benchmark 
    image classification datasets 
    (\S\ref{sec:expts})
    show that our plug-in approaches yield competitive classification and OOD detection performance at any desired abstention rate, 
    compared to a range of both SC and OOD detection baselines.
\end{enumerate}

\begin{table}[!t]
    \centering
    \renewcommand{\arraystretch}{1.25}
    \footnotesize
    
    \caption{Summary of plug-in estimators to the selective classification with OOD detection (SCOD) problem.
    In SCOD, we seek to learn a classifier capable of rejecting \emph{both} out-of-distribution (OOD) and ``hard'' in-distribution (ID) samples.
    We present 
    two plug-in estimators for SCOD, 
    one of which assumes access to only ID data,
    the other which additionally assumes access to a sample of OOD data.
    Both methods reject samples by suitably combining scores that order samples based on selective classification (SC) or OOD detection criteria.
    The former leverages \emph{any} off-the-shelf scores for these tasks,
    while the latter minimises novel loss functions to estimate these scores.}
    \label{tbl:summary}    
    
    \begin{tabular}{@{}lp{1.875in}p{1.875in}@{}}
        \toprule
         & \textbf{Black-box SCOD} & \textbf{Loss-based SCOD} \\
        \toprule
        \toprule
        \textbf{Training data}           & ID data only & ID + OOD data \\
        \textbf{SC score $s_{\rm sc}$}   & Any off-the-shelf technique, e.g., maximum softmax probability~\citep{Chow:1970} & Minimise~\eqref{eqn:decoupled-surrogate} or~\eqref{eqn:css-surrogate}, obtain $\max_{y \in [L]} f_y( x )$ \\
        \textbf{OOD score $s_{\rm ood}$} & Any off-the-shelf technique, e.g., gradient norm~\citep{Huang:2021} & Minimise~\eqref{eqn:decoupled-surrogate} or~\eqref{eqn:css-surrogate}, obtain $s( x )$ \\
                \midrule
        \textbf{Rejection rule}           & Combine $s_{\rm sc}$, $s_{\rm ood}$ via~\eqref{eqn:plug-in-black-box} & Combine $s_{\rm sc}$, $s_{\rm ood}$ via~\eqref{eqn:plug-in-black-box} or~\eqref{eqn:reject-coupled} \\
        \bottomrule
    \end{tabular}
\end{table}



\section{Background and notation}
\label{sec:background}
We focus on the multi-class classification setting:
given instances $\XCal$, labels $\YCal \defEq [ L ]$, and 
a training sample $S = \{ ( x_n, y_n ) \}_{n \in [N]} \in ( \XCal \times \YCal )^N$ comprising $N$ i.i.d. draws from a \emph{training} (or \emph{inlier}) \emph{distribution} $\Pr_{\rm in}$,
the goal is to
learn a classifier $h \colon \XCal \to \YCal$ with minimal misclassification error 
$\Pr_{\rm te}( y \neq h( x ) )$
for a \emph{test distribution} $\Pr_{\rm te}$.
By default, it is assumed that the training and test distribution coincide, i.e., $\Pr_{\rm te} = \Pr_{\rm in}$.
Typically, $h( x ) = \argmax_{y \in [L]} f_y( x )$,
where $f \colon \XCal \to \Real^L$ scores the affinity of each label to a given instance.
One may learn $f$ 
via minimisation of the \emph{empirical surrogate risk}
$ \hat{R}( f; S, \ell ) \defEq \frac{1}{| S |} \sum_{( x_n, y_n ) \in S} \ell( y_n, f( x_n ) ) $
for \emph{loss function} $\ell \colon [ L ] \times \Real^L \to \Real_+$.

The standard classification setting requires that one make a prediction for \emph{all} test samples.
However, as we now detail, 
it is often prudent to allow the classsifer to \emph{abstain} from predicting on some samples.

\textbf{Selective classification (SC)}.
In
\emph{selective classification (SC)},
also known as 
\emph{learning to reject} or
\emph{learning with abstention}~\citep{Bartlett:2008,Ramaswamy:2018,Charoenphakdee:2021,Gangrade:2021,CorDeSMoh2016,Thulasidasan:2019,Mozannar:2020},
one may \emph{abstain} from predicting on samples where a classifier has low-confidence.
Intuitively, this allows one to abstain on ``hard'' (e.g., ambiguously labelled) samples,
which could then be forwarded to an expert (e.g., a human labeller).
Formally, given a {budget} $b_{\rm rej} \in ( 0, 1 )$ on the fraction samples that can be rejected,
one learns a classifier $h \colon \XCal \to \YCal$ and \emph{rejector} $r \colon \XCal \to \{ 0, 1 \}$
to minimise the misclassification error on non-rejected samples:
\begin{equation}
    \label{eqn:l2r}
    \min_{h, r} \Pr_{\rm in}( y \neq h( x ), r( x ) = 0 ) \colon \Pr_{\rm in}( r( x ) = 1 ) \leq b_{\rm rej}.
\end{equation}
The simplest SC baseline is \emph{confidence-based} rejection~\citep{Chow:1970,Ni:2019}, 
wherein $r$ is constructed by thresholding the maximum of the \emph{softmax probability} $p_y( x ) \propto \exp( f_y( x ) )$.
Alternatively, 
one may modify the training loss $\ell$~\citep{Bartlett:2008,Ramaswamy:2018,Charoenphakdee:2021,Gangrade:2021},
or
learn an explicit rejector jointly with the classifier~\citep{CorDeSMoh2016,Geifman:2019,Thulasidasan:2019,Mozannar:2020}.

%
\textbf{OOD detection}.
In \emph{out-of-distribution} (\emph{OOD}) \emph{detection},
one seeks to identify test samples which are anomalous with respect to the training distribution~\citep{Hendrycks:2017,bendale2016towards,Bitterwolf:2022}.
Intuitively, 
this allows one to abstain from predicting on samples where 
it is unreasonable to expect the classifier to generalise.
Formally,
suppose
$\Pr_{\rm te} \defEq \pi^{*}_{\rm in} \cdot \Pr_{\rm in} + ( 1 - \pi^*_{\rm in} ) \cdot \Pr_{\rm out}$,
for (unknown) distribution $\Pr_{\rm out}$ and mixing weight $\pi^*_{\rm in} \in ( 0, 1 )$.
Samples from $\Pr_{\rm out}$ may be regarded as 
\emph{outliers}
or
\emph{out-of-distribution}
with respect to the inlier distribution (ID) $\Pr_{\rm in}$.
Given a budget $b_{\rm fpr} \in ( 0, 1 ) $ on the {false positive rate}, 
i.e., the fraction of ID samples incorrectly predicted as OOD,
the goal is to learn an \emph{OOD detector} $r \colon \XCal \to \{ 0, 1 \}$ via
\begin{equation}
    \label{eqn:ood}
    \min_{r} \Pr_{\rm out}( r( x ) = 0 ) \colon \Pr_{\rm in}( r( x ) = 1 ) \leq b_{\rm fpr}.
\end{equation}
\emph{Labelled OOD detection}~\citep{Lee:2018,Thulasidasan:2019} additionally accounts for the accuracy of $h$.
OOD detection is a natural task in real-world deployment,
as standard classifiers may produce high-confidence predictions even on completely arbitrary inputs~\citep{Nguyen:2015,Hendrycks:2017}, 
and assign higher scores to OOD compared to ID samples~\citep{Nalisnick:2019}.
Analogous to SC,
a remarkably effective baseline for OOD detection that requires only ID samples is 
the \emph{maximum softmax probability}~\citep{Hendrycks:2017}, possibly with temperature scaling and data augmentation~\citep{Liang:2018}.
Recent works found that the maximum \emph{logit} 
may be preferable~\citep{vaze2021open,Hendrycks:2022,Wei:2022}.
These may be recovered as a limiting case of energy-based approaches~\citep{Liu:2020b}.
More effective detectors can be designed in settings where one additionally has access to an OOD sample~\citep{Hendrycks:2019,Thulasidasan:2019,Dhamija:2018,Katz:2022}.

\textbf{Selective classification with OOD detection ({SCOD})}.
The SC and OOD detection problem both involve abstaining from prediction, but for subtly different reasons:
SC concerns \emph{in-distribution but difficult} samples,
while OOD detection concerns \emph{out-of-distribution} samples.
In practice, one is likely to encounter both types of samples during classifier deployment.
To this end,
\emph{selective classification with OOD detection} (\emph{SCOD})~\citep{Kim:2021,Xia:2022}
allows for abstention on each sample type,
with a user-specified parameter controlling their relative importance.
Formally, 
suppose as before that $\Pr_{\rm te} = \pi^{*}_{\rm in} \cdot \Pr_{\rm in} + ( 1 - \pi^*_{\rm in} ) \cdot \Pr_{\rm out}$.
Given a budget 
$b_{\rm rej} \in ( 0, 1 )$ on the fraction of test samples that can be {rejected},
the goal is to learn a classifier 
$h \colon \XCal \to \YCal$ {and} a {rejector} 
$r \colon \XCal \to \{ 0, 1 \}$
to minimise:
\begin{equation}
    \label{eqn:scod}
    \min_{h, r} \, (1 - \costFN) \cdot \Pr_{\rm in}( y \neq h( x ), r( x ) = 0 ) + \costFN \cdot \Pr_{\rm out}( r( x ) = 0 ) \colon \Pr_{\rm te}( r( x ) = 1 ) \leq b_{\rm rej}.
\end{equation}
Here, 
$\costFN \in [ 0, 1 ]$ is a user-specified cost of 
not rejecting an OOD sample.


%
\textbf{Contrasting SCOD, SC, and OOD detection.}
Before proceeding, it is worth pausing to 
emphasise the distinction between the three problems introduced above.
All problems involve learning a rejector to enable the classifier from abstaining on certain samples.
Crucially, SCOD encourages rejection on both 
ID samples that are likely to be misclassified,
\emph{and} 
OOD samples;
by contrast, the SC and OOD detection problems only focus on one of these cases.
Recent work has observed that standard OOD detectors tend to reject misclassified ID samples~\citep{Cen:2023};
thus,
not considering the latter can lead to overly pessimistic estimates of rejector performance.

Given the practical relevance of SCOD, 
it is of interest to design effective techniques for the problem, 
analogous to those for SC and OOD detection.
Surprisingly, the literature offers only a few instances of such techniques,
most notably the SIRC method of~\citet{Xia:2022}.
While empirically effective, this approach is heuristic is nature.
We seek to design theoretically grounded techniques that are equally effective.
To that end, we begin by investigating a fundamental property of SCOD.

\section{Bayes-optimal selective classification with OOD detection}
\label{sec:relation}
We begin our formal analysis of SCOD by deriving 
its associated
\emph{Bayes-optimal} solution,
which generalises existing results for SC and OOD detection,
and sheds light on potential SCOD strategies.

%
\subsection{Bayes-optimal SCOD rule: sample-dependent confidence thresholding}

Before designing new techniques for SCOD, 
it is prudent to ask:
what are the theoretically optimal choices for $h, r$ that we hope to approximate?
More precisely, we seek to explicate the minimisers of the population SCOD objective~\eqref{eqn:scod} over \emph{all} possible classifiers $h \colon \XCal \to \YCal$, and rejectors $r \colon \XCal \to \{ 0, 1 \}$.
These minimisers will depend on the unknown distributions $\PTr, \PTe$, and are thus not practically realisable as-is;
nonetheless, 
they will subsequently motivate the design of simple, effective, and theoretically grounded solutions to SCOD.
Further, these help study the efficacy of existing baselines.

Via standard Lagrangian analysis, observe that~\eqref{eqn:scod} is equivalent to minimising over $h, r$:
\begin{align}
    \lefteqn{L_{\rm scod}(h, r)}
    \nonumber\\
    &= (1 - \costin - \costout) \cdot \Pr_{\rm in}( y \neq {h}( x ), r( x ) = 0 ) + 
        \costin \cdot \Pr_{\rm in}( r( x ) = 1 ) + \costout \cdot \Pr_{\rm out}( r( x ) = 0 ).
    \label{eqn:scod-soft}
\end{align}
Here, $\costin, \costout \in [ 0, 1 ]$ are distribution-dependent constants 
which encode 
the
false negative outlier cost $\costFN$, 
abstention budget $b_{\rm rej}$,
and the proportion $\pi^{*}_{\rm in}$ of inliers in
$\Pr_{\rm te}$.
We shall momentarily treat these constants as fixed and known; 
we return to the issue of suitable choices for them in~\S\ref{sec:formulation}.
Note that we obtain a soft-penalty version of the SC problem when 
$\costout = 0$,
and the OOD detection problem when 
$\costin + \costout = 1$.
In general,
we have the following Bayes-optimal solution for~\eqref{eqn:scod-soft}.

%
\begin{lemma}
\label{lem:scod-bayes}
Let $( h^*, r^* )$ denote
any minimiser of \eqref{eqn:ood}.
Then,
for any $x \in \XCal$ with $\Pr_{\rm in}( x ) > 0$: 
\begin{equation}
    \label{eqn:scod-bayes}
    r^*( x ) = 1 \iff  
    (1 - \costin - \costout) \cdot \left( 1 - \max_{y \in [L]} \Pr_{\rm in}(y \mid x) \right) + \costout \cdot \frac{\Pr_{\rm out}( x )}{\Pr_{\rm in}( x )} > \costin.
\end{equation}
Further, 
$r^*( x ) = 1$ when $\Pr_{\rm in}( x ) = 0$,
and
$h^*( x ) = \argmax_{y \in [L]} \Pr_{\rm in}( y \mid x )$ when $r^*( x ) = 0$.
\end{lemma}

The optimal classifier $h^*$ has an unsurprising form: for non-rejected samples, we predict the label $y$ with highest inlier class-probability $\Pr_{\rm in}( y \mid x )$.
The Bayes-optimal rejector is more interesting, and involves a comparison between two key quantities:
the \emph{maximum inlier class-probability} $\max_{y \in [L]} \Pr_{\rm in}(y \mid x)$,
and
the \emph{density ratio} $\frac{\Pr_{\rm in}( x )}{\Pr_{\rm out}( x )}$.
These respectively reflect the confidence in the most likely label, 
and the confidence in the sample being an inlier.
Intuitively, when either of these quantities is sufficiently small, a sample is a candidate for rejection.

We now verify that Lemma~\ref{lem:scod-bayes} generalises existing Bayes-optimal rules for SC and OOD detection.

%
\textbf{Special case: SC}.
Suppose $\costout = 0$ and $\costin < 1$.
Then,~\eqref{eqn:scod-bayes} reduces to \emph{Chow's rule}~\citep{Chow:1970,Ramaswamy:2018}:
\begin{equation}
    \label{eqn:chow}
    r^*( x ) = 1 \iff 1 - \max_{y \in [L]} \Pr_{\rm in}(y \mid x) > \frac{\costin}{1 - \costin}.
\end{equation}
Thus, samples with high uncertainty in the label distribution are rejected.


%
\textbf{Special case: OOD detection}.
Suppose $\costin + \costout = 1$ and $\costin < 1$.
Then,~\eqref{eqn:scod-bayes} reduces to \emph{density-based rejection}~\citep{Steinwart:2005,Chandola:2009}:
\begin{equation}
    \label{eqn:ood-bayes}
    r^*( x ) = 1 \iff \frac{\Pr_{\rm out}( x )}{\Pr_{\rm in}( x )} > \frac{\costin}{1 - \costin}.
\end{equation}
Thus, samples with relatively high density under $\Pr_{\rm out}$ are rejected.

%
\subsection{Implication: existing SC and OOD baselines do not suffice for SCOD}

Lemma~\ref{lem:scod-bayes} implies that SCOD cannot be readily solved by
existing SC and OOD detection baselines.
Specifically, consider the
\emph{confidence-based rejection} baseline,
which rejects samples where
$\max_{y \in [L]} \Pr_{\rm in}(y \mid x)$ is lower than a fixed constant.
This is known as 
\emph{Chow's rule}~\eqref{eqn:chow} in 
the SC literature~\citep{Chow:1970,Ramaswamy:2018,Ni:2019}, 
and
the \emph{maximum softmax probability} (\emph{MSP}) in OOD literature~\citep{Hendrycks:2017};
for brevity, we adopt the latter terminology.
The MSP baseline does not suffice for the SCOD problem
in general:
even if $\max_{y \in [L]} \Pr_{\rm in}(y \mid x) \sim 1$,
it may be optimal to reject an input $x \in \mathscr{X}$ if 
$\frac{\Pr_{\rm out}( x )}{\Pr_{\rm in}( x )} \gg 0$.

In fact, the situation is more dire:
the MSP
may result in \emph{arbitrarily bad} rejection decisions.
Surprisingly, this 
even holds
in a special case of OOD detection
wherein there is a strong relationship between $\Pr_{\rm in}$ and $\Pr_{\rm out}$ that \emph{a-priori} would appear favourable to the MSP.
Specifically,
given some distribution $\Pr_{\rm te}$ over $\XCal \times \YCal$,
consider the \emph{open-set classification} (\emph{OSC}) setting~\citep{Scheirer:2013,vaze2021open}:
during training, one only observes samples from a distribution $\Pr_{\rm in}$ over $\XCal \times \YCal_{\rm in}$, 
where $\YCal_{\rm in} \subset \YCal$.
Here, $\Pr_{\rm in}$ is a restriction of $\Pr_{\rm te}$ to a subset of labels.
At evaluation time, 
one seeks to 
accurately classify samples possessing these labels, while 
rejecting samples with unobserved labels $\YCal - \YCal_{\rm in}$.

Under this setup, thresholding
$\max_{y \in \YCal_{\rm in}} \Pr_{\rm in}( y \mid x )$
might appear a reasonable approach.
However, we now demonstrate that it may lead to arbitrarily poor decisions.
In what follows, for simplicity we consider the OSC problem wherein $\YCal_{\rm in} = \YCal - \{ L \}$,
so that there is only one label unobserved in the in-distribution sample.
Further, we focus on the setting where $\costin + \costout = 1$.
We have the following.

%
\begin{lemma}
\label{lemm:bayes-open-set-rewrite}
Under the open-set setting, the Bayes-optimal classifier for the 
SCOD problem
is:
\begin{align*}
    r^*(x) = 1
    &\iff
    {\Pr_{\rm te}( L \mid x )} > t^*_{\rm osc}
    \iff
    \max_{y' \neq L} \PTr( y' \mid x ) \geq \frac{1}{1 - t^*_{\rm osc}} \cdot \max_{y' \neq L} \Pr_{\rm te}( y' \mid x),
\end{align*}
where 
$t^*_{\rm osc} \defEq F\left( \frac{\costin \cdot \Pr_{\rm te}( y = L )}{\costout \cdot \Pr_{\rm te}( y \neq L )} \right)$ for $F \colon z \mapsto {z} / ({1 + z})$.
\end{lemma}


Lemma~\ref{lemm:bayes-open-set-rewrite} shows that the optimal decision is to reject when the maximum softmax probability (with respect to $\PTr$) is \emph{higher} than some (sample-dependent) threshold.
This is the precise \emph{opposite} of the MSP baseline, which rejects when the maximum probability is
\emph{lower} than some threshold.
What is the reason for this stark discrepancy?
Intuitively, the issue is that we would like to threshold 
$\Pr_{\rm te}( y \mid x )$, 
\emph{not} $\PTr( y \mid x )$;
however, these two distributions may not align,
as the latter includes a normalisation term that causes unexpected behaviour when we threshold.
We make this concrete with a simple example; 
see also Figure~\ref{fig:chow-fail-example} (Appendix \ref{app:failure-msp}) for an illustration.

%
\begin{example}[{\it Failure of MSP baseline}]
\label{ex:msp-failure}
Consider a setting where the class probabilities $\Pr_{\rm te}( y' \mid x )$ are equal for all the known classes $y' \ne L$.
This implies that
$\PTr( y' \mid x ) = \frac{1}{L-1}, \forall y' \ne L$.
The Bayes-optimal classifier rejects a sample when 
$\Pr_{\rm te}( L \mid x ) > \frac{\costin}{\costin + \costout}$.
On the other hand,
MSP rejects a sample 
iff the threshold 
$t_{\rm msp} < \frac{1}{L-1}$.
Notice that the rejection decision is \emph{independent} of the unknown class density $\Pr_{\rm te}( L \mid x )$, 
and therefore will not agree with the Bayes-optimal classifier in general.
The following lemma formalizes this observation.
\end{example}

\begin{lemma}
Pick any $t_{\rm msp} \in (0, 1)$,
and consider the corresponding MSP baseline which rejects 
$x \in \XCal$ iff
$\max_{y \neq L} \PTr( y \mid x ) < t_{\rm msp}$.
Then,
there exists a class-probability function 
$\Pr_{\rm te}(y \mid x)$ for which 
the Bayes-optimal rejector 
$\Pr_{\rm te}( L \mid x ) > t^*_{\rm osc}$ disagrees with MSP $\forall t_{\rm msp} \in (0, 1)$.
\label{lem:chow-fail}
\end{lemma}


{One may ask whether using the maximum \emph{logit} rather than softmax probability can prove successful in the open-set setting.
Unfortunately, as this similarly does not include information about $\mathbb{P}_{\rm out}$, 
it can also fail.
For the same reason, other baselines from the OOD and SC literature can also fail;
see Appendix~\ref{app:expts-max-logit}.}
Rather than using existing baselines as-is,
we now consider a more direct approach to estimating the Bayes-optimal SCOD rejector in~\eqref{eqn:scod-bayes}, which has strong empirical performance.

\section{Plug-in estimators to the Bayes-optimal SCOD rule}
\label{sec:formulation}
A minimal requirement
for a reasonable SCOD technique is that its popular minimiser coincides with the Bayes-optimal solution in~\eqref{eqn:scod-bayes}.
Unfortunately,
any attempt at practically realising this solution faces an immediate challenge:
it requires the ability to compute expectations under 
$\Pr_{\rm out}$.
In practice, 
one can scarcely expect to have ready access to $\Pr_{\rm out}$;
indeed, the very premise of OOD detection is that $\Pr_{\rm out}$ comprises samples wholly dissimilar to those used to train the classifier.


In the OOD detection literature, 
this challenge is typically addressed by
designing techniques that exploit only ID information from $\Pr_{\rm in}$,
or
assuming access to a small OOD sample of outliers~\citep{Hendrycks:2019}, possibly mixed with some ID data~\citep{Katz:2022}.
Following this, we now present two techniques that estimate~\eqref{eqn:scod-bayes},
one of which exploits only ID data,
and another which exploits both ID and OOD data.
These techniques come equipped with theoretical guarantees, 
while also generalising existing approaches from the SC and OOD detection literature.

\subsection{Black-box SCOD using only ID data}
\label{sec:plug-in-black-box}

Our first plug-in estimator operates in the setting where one only has access to 
ID samples from
$\Pr_{\rm in}$.
Here, we cannot hope to minimise~\eqref{eqn:scod-soft} directly.
Instead,
we look to approximate the corresponding Bayes-optimal solution~\eqref{eqn:scod-bayes} by leveraging existing 
OOD detection techniques operating in this setting.

Concretely,
suppose 
we have access to \emph{any} existing OOD detection score 
$s_{\rm ood} \colon \XCal \to \Real$ 
that is computed only from ID data,
e.g., the gradient norm score of~\citet{Huang:2021}.
Similarly,
let
$s_{\rm sc} \colon \XCal \to \Real$ 
be \emph{any} existing SC score,
e.g., the maximum softmax probability estimate of~\citet{Chow:1970}.
Then, we propose the following
\emph{black-box rejector}:
\begin{equation}
    \label{eqn:plug-in-black-box}
    r_{\rm BB}( x ) = 1 \iff ( 1 - \costin - \costout ) \cdot s_{\rm sc}( x ) + {\costout} \cdot \vartheta\left( { s_{\rm ood}( x ) } \right) < t_{\rm BB},
\end{equation}
where
$t_{\rm BB} \defEq 1 - 2 \cdot \costin - \costout$,
and $\vartheta \colon z \mapsto -\frac{1}{z}$.
Observe that
Equation~\ref{eqn:plug-in-black-box} 
exactly coincides with the Bayes-optimal rejector~\eqref{eqn:scod-bayes}
when
$s_{\rm sc}, s_{\rm ood}$
equal their Bayes-optimal counterparts
$s^*_{\rm sc}( x ) \defEq \max_{y \in [L]} \Pr_{\rm in}( y \mid x )$ 
and
$s^*_{\rm ood}( x ) \defEq \frac{\Pr_{\rm in}( x )}{\Pr_{\rm out}( x )}$.
Thus, 
as is intuitive,
~\eqref{eqn:plug-in-black-box} can be expected to perform well
when $s_{\rm sc}, s_{\rm ood}$ perform well at the SC and OOD detection task respectively, as shown below.

\begin{lemma}
\label{lemm:black-box-regret}
Suppose we have estimates $\hat{\Pr}_{\rm in}(y \mid x)$ of the inlier class probabilities ${\Pr}_{\rm in}(y \mid x)$, estimates $\hat{s}_{\rm ood}(x)$ of the density ratio $\frac{\Pr_{\rm in}( x )}{\Pr_{\rm out}( x )}$, and \emph{SC} scores $\hat{s}_{\rm sc}(x) = \max_{y \in [L]} \hat{\Pr}_{\rm in}( y \mid x )$.
Let $\hat{h}(x) \in \argmax_{y \in [L]} \hat{\Pr}_{\rm in}(y \mid x)$, 
and $\hat{r}_{\rm BB}$ be a rejector defined according to \eqref{eqn:plug-in-black-box} from 
$\hat{s}_{\rm sc}(x)$ and $\hat{s}_{\rm ood}(x)$. 
Let $\Pr^*(x) = \frac{1}{2}( \Pr_{\rm in}(x) + \Pr_{\rm out}(x) )$. 
Then, for the \emph{SCOD}-risk \eqref{eqn:scod} minimizers $(h^*, r^*)$:
\begin{align*}
\lefteqn{L_{\rm scod}(\hat{h}, \hat{r}_{\rm BB}) \,-\,
L_{\rm scod}(h^*, r^*)}
\\[-3pt]
&\leq
\textstyle
2 \cdot \Ex_{x \sim \Pr^*}\left[\sum_{y \in [L]} \left|\Pr_{\rm in}(y \mid x) - \hat{\Pr}_{\rm in}(y \mid x) \right|  
\,+\,
2\cdot \sum_{y \in [L]} 
{
\left|
\frac{\Pr_{\rm in}( x )}{\Pr_{\rm in}( x ) + \Pr_{\rm out}( x )} - \frac{\hat{s}_{\rm ood}(x)}{1 + \hat{s}_{\rm ood}(x)} \right| 
}
\right].
\end{align*}
\end{lemma}



Interestingly, this black-box rejector can be seen as a principled variant of the SIRC method of~\citet{Xia:2022}.
As with $r_{\rm BB}$,
SIRC works by combining
rejection scores
$s_{\rm sc}( x ), s_{\rm ood}( x )$
for SC and OOD detection respectively.
The key difference is that SIRC employs a \emph{multiplicative} combination:
\begin{equation}
    \label{eqn:sirc}
    r_{\rm SIRC}( x ) = 1 \iff (s_{\rm sc}( x ) - a_1) \cdot \varrho( a_2 \cdot s_{\rm ood}( x ) + a_3 ) < t_{\rm SIRC},
\end{equation}
for constants $a_1, a_2, a_3$, threshold $t_{\rm SIRC}$,
and monotone transform $\varrho \colon z \mapsto 1 + e^{-z}$.
Intuitively, one rejects samples where there is 
sufficient signal that the sample is both 
near the decision boundary, 
\emph{and} likely drawn from the outlier distribution.
While empirically effective, 
it is not hard to see that the Bayes-optimal rejector~\eqref{eqn:scod-bayes} does not take the form of~\eqref{eqn:sirc};
thus, in general, SIRC may be sub-optimal.
We note that this also holds 
for the objective considered in~\citet{Xia:2022}, 
which is a slight variation of~\eqref{eqn:scod} that enforces a constraint on the ID recall.



\subsection{Loss-based SCOD using ID and OOD data}
\label{sec:plug-in-losses}

Our second plug-in estimator operates in the setting where one has access to both ID data, \emph{and} a ``wild'' sample comprising a mixture of ID and OOD data.
Here,
we may seek to directly
minimise the SCOD risk in \eqref{eqn:scod-soft}
via 
novel loss functions.
We shall first present the population risk corresponding to these losses,
before describing 
their instantiation for practical settings.

\textbf{Decoupled loss}.
Our first loss function
builds on the same observation as the previous section:
given estimates $s_{\rm sc}( x )$, $s_{\rm ood}( x )$
of
$\Pr_{\rm in}( y \mid x )$ and
$\frac{\Pr_{\rm in}( x )}{\Pr_{\rm out}( x )}$
respectively,
~\eqref{eqn:plug-in-black-box} 
yields a \emph{plug-in estimator} of the Bayes-optimal rule.
However, rather than leverage black-box estimates
based on ID data
---
which necessarily have limited fidelity
---
we seek to \emph{learn} them by leveraging both the ID and OOD data.

To construct such estimates, we learn
scorers $f \colon \XCal \to \Real^{L}$ and $s \colon \XCal \to \Real$.
Our goal is for a suitable transformation of
$f_y( x )$ and $s( x )$ to approximate ${\Pr}_{\rm in}( y \mid x )$ and 
$\frac{\Pr_{\rm in}( x )}{\Pr_{\rm out}( x )}$.
We propose
to minimise:
\begin{equation}
    \label{eqn:decoupled-surrogate}
    %
    \mathbb{E}_{( x, y ) \sim \mathbb{P}_{\rm in}}\left[ \ell_{\rm mc}( y, {f}( x ) ) \right] + 
    \mathbb{E}_{x \sim \mathbb{P}_{\rm in}}\left[ \ell_{\rm bc}( +1, s( x ) ) \right] + 
    \mathbb{E}_{x \sim \mathbb{P}_{\rm out}}\left[ \ell_{\rm bc}( -1, s( x ) ) \right],%
\end{equation}
where 
$\ell_{\rm mc} \colon [ L ] \times \Real^L \to \Real_+$ 
and $\ell_{\rm bc} \colon \{ \pm 1 \} \times \Real \to \Real_+$
are \emph{strictly proper composite}~\citep{Reid:2010} losses
for multi-class and binary classification respectively.
Canonical instantiations are the softmax 
cross-entropy 
$\ell_{\rm mc}(y, f(x)) = \log\left[ \sum_{y' \in [L]} e^{f_{y'}( x )} \right] - f_{y}( x )$,
and the sigmoid cross-entropy
$\ell_{\rm bc}(z, f(x)) = \log(1 + e^{-z \cdot f(x)})$.
In words, we use a standard 
multi-class classification
loss on the ID samples, 
with an additional loss that discriminates between the ID and OOD samples.
Note that in the last two terms, we do \emph{not} impose separate costs for the OOD detection errors. 
\begin{lemma}
Let $\Pr^*(x, z) = \frac{1}{2}\left(\Pr_{\rm in}(x)\cdot \1(z = 1) + \Pr_{\rm out}(x)\cdot \1(z = -1) \right)$ denote a joint ID-OOD distribution, with $z = -1$ indicating an OOD sample.  
Suppose $\ell_{\rm mc}, \ell_{\rm bc}$
correspond to the softmax and sigmoid cross-entropy.
Let $(f^*, r^*)$ be the minimizer of
the decoupled loss in \eqref{eqn:decoupled-surrogate}. 
For any scorers $f, s$,
with transformations $p_y(x) = \frac{ \exp( f_y(x) ) }{ \sum_{y'} \exp( f_{ y' }(x) ) }$ and $p_{\perp}(x) = \frac{1}{1 + \exp( -s(x) )}$:
\begin{align*}
\textstyle
    \mathbb{E}_{x \sim \mathbb{P}_{\rm in}}\left[ \sum_{y \in [L]} \big|
    p_y(x) - \Pr_{\rm in}(y \mid x) \big| \right]
    &\textstyle
    \leq 
    \sqrt{2}\sqrt{
    \mathbb{E}_{( x, y ) \sim \mathbb{P}_{\rm in}}\left[ \ell_{\rm mc}( y, f( x ) )
    \right] 
    \,-\,
    \mathbb{E}_{( x, y ) \sim \mathbb{P}_{\rm in}}\left[ \ell_{\rm mc}( y, f^*( x ) )
    \right]
    }\\
    \textstyle
    \mathbb{E}_{x \sim \mathbb{P}^*}\left[
    \left|
    p_\perp(x) - \frac{\Pr_{\rm in}( x )}{\Pr_{\rm in}( x ) + \Pr_{\rm out}( x )} \right| \right]
    &\textstyle
    \leq 
    \frac{1}{\sqrt{2}}\sqrt{
    \mathbb{E}_{( x, z ) \sim \mathbb{P}^*}\left[ \ell_{\rm bc}( z, s( x ) )
    \right] 
    \,-\,
    \mathbb{E}_{( x, z ) \sim \mathbb{P}^*}\left[ \ell_{\rm bc}(z , s^*( x ) )
    \right].
    }
\end{align*}
\end{lemma}

Note that in the first term of the decoupled loss in \eqref{eqn:decoupled-surrogate}, we only use classification scores $f_y(x)$, and
exclude the rejector score $s( x )$.
The classifier and rejector losses are thus \emph{decoupled}.
We may introduce coupling 
\emph{implicitly}, 
by parameterising $f_{y'}( x ) = w_{y'}^\top \Phi( x )$ and $s( x ) = u_{y'}^\top \Phi( x )$ for shared embedding $\Phi$;
or \emph{explicitly}, 
as follows.

\textbf{Coupled loss}.
We propose a second loss function that
seeks to learn an augmented scorer $\bar{f} \colon \XCal \to \Real^{L+1}$, with the additional score corresponding to a ``reject class'', denoted by $\perp$, and 
takes the form of a standard multi-class classification loss applied jointly to both the classification and rejection logits:
\begin{equation}
    \label{eqn:css-surrogate}
    \E{(x,y) \sim \Pr_{\rm in}}{ \ell_{\rm mc}( y, \bar{f}( x ) ) } + (1 - \costin) \cdot \E{x \sim \Pr_{\rm in}}{ \ell_{\rm mc}( \perp, \bar{f}( x ) ) } + \costout \cdot \E{x \sim \Pr_{\rm out}}{ \ell_{\rm mc}( \perp, \bar{f}( x ) ) }.    
\end{equation}
This yields an alternate plug-in estimator of the Bayes-optimal rule, which we discuss in Appendix \ref{app:coupled-loss}.


\textbf{Practical algorithm: SCOD in the wild}.
\label{sec:practical-algo}
The losses in~\eqref{eqn:decoupled-surrogate} and~\eqref{eqn:css-surrogate}
require estimating expectations under $\Pr_{\rm out}$.
While obtaining access to a sample drawn from $\Pr_{\rm out}$ may be challenging,
we adopt a similar strategy to~\citet{Katz:2022}, and assume access to two sets of \emph{unlabelled} samples:
\begin{enumerate}[leftmargin=24pt,topsep=0pt,itemsep=0pt]
    \item[{\bf (A1)}]
$S_{\rm mix}$,
consisting of a mixture of inlier and outlier samples drawn i.i.d.\ from 
a mixture
$\Pr_{\rm mix} = \pi_{\rm mix} \cdot \Pr_{\rm in} + (1-\pi_{\rm mix}) \cdot \Pr_{\rm out}$
of samples observed in the \emph{wild} (e.g., during deployment)

    \item[{\bf (A2)}] $S^*_{\rm in}$,
    consisting of samples certified to be \emph{strictly inlier}, i.e., with $\Pr_{\rm out}(x) = 0, \forall x \in S^*_{\rm in}$
\end{enumerate}
Assumption {\bf (A1)} was employed in~\citet{Katz:2022}, 
and may be implemented in practice by collecting
samples encountered ``in the wild'' during deployment of the SCOD classifier and rejector.
Assumption {\bf (A2)} merely requires identifying samples that are clearly \emph{not} OOD,
and is not difficult to satisfy:
it
may be implemented in practice by either 
identifying prototypical training samples,
or by
simply selecting a random subset of the training sample.
We follow the latter in our experiments.

Equipped with $S_{\rm mix}$, 
following~\citet{Katz:2022},
we propose to use it to approximate expectations under $\Pr_{\rm out}$.
One challenge 
is that 
the rejection logit 
will now estimate $\frac{\Pr_{\rm in}( x )}{\Pr_{\rm mix}( x )}$,
rather than $\frac{\Pr_{\rm in}( x )}{\Pr_{\rm out}( x )}$.
To resolve this, 
it is not hard to show that by {\bf(A2)}, 
one can estimate the latter via a simple transformation (see Appendix~\ref{app:noise_correction}).
Plugging these estimates into~\eqref{eqn:plug-in-black-box} then gives us an approximation to the Bayes-optimal solution.
We summarise this procedure in Algorithm~\ref{algo:plug-in} for the decoupled loss.

In Appendix \ref{app:relation-existing-losses}, we explain how our losses relate to existing losses for OOD detection.~\\[-10pt]


%
\begin{figure}[t]
\vspace{-6pt}
    \begin{algorithm}[H]
    \caption{Loss-based SCOD using a mixture of ID and OOD data}
    \label{algo:plug-in}
        \begin{algorithmic}[1]
        \STATE \textbf{Input:} Labeled set $S_{\rm in} \sim \Pr_{\rm in}$, Unlabeled set $S_{\rm mix} \sim \Pr_{\rm mix}$, Strictly inlier set $S^*_{\rm in}$ with $\Pr_{\rm out}(x) = 0, \forall x\in S^*_{\rm in}$
        \\[-10pt]
        \STATE \textbf{Parameters:} Costs $\costin, \costout$
        \STATE \textbf{Surrogate loss:} Find minimizers  $\hat{f}: \XCal \rightarrow \R^{L+1}$ and $\hat{s}: \XCal \rightarrow \R$ of the decoupled loss:
        \begin{equation*}
            \resizebox{\linewidth}{!}{
            $\displaystyle \frac{1}{|S_{\rm in}|} \sum_{(x, y) \in S_{\rm in}} \ell_{\rm mc}(y, f(x)) +
            \frac{1}{|S_{\rm in}|}\sum_{(x, y) \in S_{\rm in}} \ell_{\rm bc}(+1, s(x)) +
            \frac{1}{|S_{\rm mix}|}\sum_{x \in S_{\rm mix}} \ell_{\rm bc}(-1, s(x))$
            }
        \end{equation*}
        \vspace{-10pt}
        \STATE \textbf{Inlier class probabilities:} $\hat{\Pr}_{\rm in}(y|x) \defEq \frac{ \exp( \hat{f}_y(x)) }{ \sum_{y'} \exp( \hat{f}_{y'}(x)) }$
        \STATE \textbf{Mixture proportion:} $\hat{\pi}_{\rm mix} \defEq \frac{1}{|S^*_{\rm in}|}\sum_{x \in S^{*}_{\rm in}} \exp( -\hat{s}(x))$
        \STATE \textbf{Density ratio:}  $\hat{s}_{\rm ood}(x) \defEq \big( \frac{1}{1-\hat{\pi}_{\rm mix}} \cdot \left( \exp( -\hat{s}(x) ) - \hat{\pi}_{\rm mix} \right) \big)^{-1}$
        \STATE \textbf{Plug-in classifier:} Plug estimates $\hat{\Pr}_{\rm in}(y|x)$, $\hat{s}_{\rm ood}(x)$, and costs $\costin, \costout$ into \eqref{eqn:plug-in-black-box}, and construct classifier $\hat{h}$, rejector $\hat{r}$
        \STATE \textbf{Output:} $\hat{h}, \hat{r}$
        \end{algorithmic}
    \end{algorithm}
\vspace{-17pt}
\end{figure}

%



Thus far, we have focused on minimising~\eqref{eqn:scod-soft}, 
which applies a soft penalty on making incorrect reject decisions.
This requires specifying costs $\costin, \costout \in [ 0, 1 ]$, 
which respectively control the importance of
not rejecting ID samples
and
rejecting OOD samples
compared to misclassifying non-rejected ID samples.
These user-specified parameters may be set based on any available domain knowledge.

In practice, 
it may be more natural for a user to specify 
the relative cost $\costFN \in [ 0, 1 ]$
of making an incorrect rejection decision on 
OOD samples,
and a budget $b_{\rm rej} \in [ 0, 1 ]$ on the total fraction of abstentions,
as in~\eqref{eqn:scod}; this is the setting we consider in our experiments in the next section.
Our plugin estimators 
easily accommodate such an explicit constraint, via standard Lagrangian;
see Appendix~\ref{app:plug-in-budget}.

\section{Experimental results}
\label{sec:expts}
\vspace{-3pt}

We now demonstrate the efficacy of our proposed plug-in approaches
to SCOD
on a range of image classification benchmarks
from the OOD detection and SCOD literature~\citep{Bitterwolf:2022, Katz:2022, Xia:2022}.

\textbf{Datasets.} 
We use 
CIFAR-100~\citep{Krizhevsky:2009} and 
ImageNet~\citep{Deng:2009} as the in-distribution (ID) datasets, 
and 
SVHN \citep{netzer2011reading}, 
Places365 \citep{zhou2017places}, 
LSUN \citep{yu2015lsun} (original and resized), 
Texture \cite{cimpoi2014describing}, 
CelebA \citep{liu2015faceattributes}, 
300K Random Images \citep{Hendrycks:2019}, 
OpenImages \citep{krasin2017openimages}, 
OpenImages-O \citep{wang2022vim},
 iNaturalist-O~\citep{Huang:2021b} 
 and Colorectal~\citep{kather2016multi}
 as the OOD datasets. 
For training, we use labeled ID samples and (optionally) an unlabeled ``wild'' mixture of ID and OOD samples ($\Pr_{\rm mix} = \pi_{\rm mix} \cdot \Pr_{\rm in} + (1-\pi_{\rm mix}) \cdot \Pr^{\rm tr}_{\rm out}$). 
For testing, we use OOD samples ($\Pr^{\rm te}_{\rm out}$) that may be different from those used in training ($\Pr^{\rm tr}_{\rm out}$). 
We train a ResNet-56 on CIFAR, and use a pre-trained BiT ResNet-101  on ImageNet (hyper-parameter details in Appendix \ref{app:expts}).

In experiments where we use both ID and OOD samples for training, the training set comprises of equal number of ID samples and  wild samples. 
We hold out 5\% of the original ID test set and use it as the ``strictly inlier'' sample needed to estimate 
$\pi_{\rm mix}$ for Algorithm~\ref{algo:plug-in}. 
Our final test set contains equal proportions of ID and OOD samples; we report results with other 
choices in Appendix \ref{app:expts}.


\textbf{Evaluation metrics.}
Recall that our goal is to solve the constrained objective in ~\eqref{eqn:scod}.
One way to measure performance with respect to this objective is
to 
measure
the area under the risk-coverage curve (AUC-RC),
as considered in prior work~\citep{Kim:2021, Xia:2022}.
Concretely, we
plot the joint risk in~\eqref{eqn:scod} as a function of samples abstained, and evaluate the area under the curve. 
This summarizes the performance of a rejector on both selective classification and OOD detection.
For a fixed fraction 
$\hat{b}_{\rm rej} 
= \frac{1}{|S_{\rm all}|}
\sum_{x \in S_{\rm all}}\1( r(x) = 1 )$
of abstained samples, we measure the joint risk as:
\begin{align*}
\textstyle
  {
  \frac{1}{Z}\Big( (1 - \costFN) \cdot\sum_{(x, y) \in S_{\rm in}}\1( y \ne h(x), r(x) = 0 )
  }
  {~+~ \costFN \cdot
        \sum_{x \in S_{\rm out}}\1( r(x) = 0 ) \Big)
    },
\end{align*}
where $Z = \sum_{x \in S_{\rm all}}\1( r(x) = 0 )$ conditions the risk on non-rejected samples, and
$S_{\rm all} = \{x: (x, y) \in S_{\rm in}\} \cup S_{\rm out}$ is the combined ID-OOD dataset. 
See Appendix~\ref{app:plug-in-budget} for details of how our plug-in estimators handle this constrained objective.
We set $\costFN = 0.75$ here, and explore other cost parameters in Appendix \ref{app:expts}.
We additionally report 
the ID accuracy, and
the precision, recall, 
ROC-AUC and FPR@95TPR for OOD detection, and provide plots of risk-coverage curves.

\begin{table}[t]
    \centering
    \caption{Area Under the Risk-Coverage Curve (AUC-RC) for  methods trained with CIFAR-100 as the ID sample and a mix of CIFAR-100 and either 300K Random Images or Open Images as the wild sample ($c_{\rm fn} = 0.75$). 
    The wild set contains 10\% ID and 90\% OOD. 
    Base model is ResNet-56. 
    A * against a method indicates that it uses both ID and OOD samples for training. 
    \emph{Lower} values are \emph{better}.}
    \scriptsize
    \begin{tabular}{@{}lccccccccccc@{}}
        \toprule
        & \multicolumn{5}{c}{
        ID + OOD training with
        $\Pr^{\rm tr}_{\rm out}$ = Random300K}
        &~~&
        \multicolumn{5}{c}{
        ID + OOD training with
        $\Pr^{\rm tr}_{\rm out}$ = OpenImages}
        \\
        Method / $\Pr^{\rm te}_{\rm out}$  & SVHN & Places & LSUN & LSUN-R
        & Texture & ~~ & SVHN & Places & LSUN & LSUN-R & Texture
        \\
        \toprule
        MSP & 0.318 & 0.337 & 0.325 & 0.392 & 0.350 &  & 
0.321 & 0.301 & 0.322 & 0.291 & 0.334 \\[3pt]
        MaxLogit & 0.284 & 0.319 & 0.297 & 0.365 & 0.332 &  & 
0.295 & 0.247 & 0.283 & 0.237 & 0.302
        \\[3pt]
        Energy & 0.285 & 0.320 & 0.296 & 0.364 & 0.328 &  & 
0.295 & 0.246 & 0.282 & 0.233 & 0.299
        \\
        \midrule
        SIRC [$L_1$] & 0.295 & 0.330 & 0.300 & 0.387 & 0.325 &  & 
0.307 & 0.273 & 0.294 & 0.257 & 0.308
        \\[3pt]
        SIRC [Res] & 0.270 & 0.333 & 0.289 & 0.387 & 0.355 &  & 
0.280 & 0.288 & 0.283 & 0.273 & 0.336
        \\
        \midrule
        CCE* & 0.287 & 0.314 & 0.254 & 0.212 & 0.257 &  & 
0.303 & 0.209 & 0.246 & 0.210 & 0.277 \\[2pt]
        DCE* & 0.294 & 0.325 & 0.246 & 0.211 & 0.258 &  & 
0.352 & 0.213 & 0.263 & 0.214 & 0.292 \\[3pt]
        OE* & 0.312 & 0.305 & 0.260 & 0.204 & 0.259 &  & 
0.318 & 0.202 & 0.259 & 0.204 & 0.297 \\
        \midrule
        Plug-in BB [$L_1$]  & 0.223 & \textbf{0.286} & \textbf{0.226} & 0.294 & \textbf{0.241} &  & 0.248 & 0.211 & \textbf{0.221} & {0.202} & \textbf{0.232}
        \\[2pt]
        Plug-in BB [Res] & \textbf{0.204} & 0.308 & 0.234 & 0.296 & 0.461 &  & 
\textbf{0.212} & 0.240 & \textbf{0.221} & 0.219 & 0.447
        \\ [2pt]
        Plug-in LB* & 0.289 & 0.305 & 0.243 & \textbf{0.187} & 0.249 &  & 
0.315 & \textbf{0.182} & 0.267 & \textbf{0.186} & 0.292\\
        \bottomrule
    \end{tabular}
    \vspace{-3pt}
    \label{tab:auc-rc-cf100-random-100k}
\end{table}

\begin{table}[t]
    \centering
    \caption{AUC-RC ($\downarrow$) for CIFAR-100 as ID, and a ``wild'' comprising of 90\% ID and \emph{only} 10\% OOD. 
    The OOD part of the wild set is drawn from the \emph{same} OOD dataset from which the test set is drawn.
    }
    \vspace{-2pt}
    \scriptsize
    \begin{tabular}{@{}lccccccc@{}}
        \toprule
        & \multicolumn{6}{c}{
        ID + OOD training with
        $\Pr^{\rm tr}_{\rm out} = \Pr^{\rm te}_{\rm out}$} 
        \\[2pt]
        Method / $\Pr^{\rm te}_{\rm out}$ & SVHN & Places & LSUN & LSUN-R & Texture & OpenImages & CelebA 
        \\
        \toprule
        MSP & 0.313 & 0.287 & 0.325 & 0.300 & 0.402 & 0.281 & 0.267 \\[3pt]
        MaxLogit & 0.254 & 0.232 & 0.286 & 0.250 & 0.391 & 0.243 & 0.234 \\[3pt]
        Energy & 0.250 & 0.232 & 0.284 & 0.247 & 0.389 & 0.243 & 0.231 \\
        \midrule
        SIRC [$L_1$] & 0.254 & 0.257 & 0.289 & 0.276 & 0.378 & 0.257 & 0.229 \\[3pt]
        SIRC [Res] & 0.249 & 0.270 & 0.292 & 0.289 & 0.408 & 0.269 & 0.233 \\
        \midrule
    CCE* & 0.238 & 0.227 & 0.231 & 0.235 & 0.239 & 0.243 & 0.240 \\[3pt]
    DCE* & 0.235 & 0.220 & 0.226 & 0.230 & 0.235 & 0.241 & 0.227 \\[3pt]
    OE* & 0.245 & 0.245 & 0.254 & 0.241 & 0.264 & 0.255 & 0.239 \\
        \midrule
        Plug-in BB [$L_1$] & \textbf{0.196} & 0.210 & 0.226 & 0.223 & 0.318 & \textbf{0.222} & 0.227 \\[2pt]
        Plug-in BB [Res] & 0.198 & 0.236 & 0.244 & 0.250 & 0.470 & 0.251 & 0.230 \\ [2pt]
        Plug-in LB* & 0.221 & \textbf{0.199} & \textbf{0.209} & \textbf{0.215} & \textbf{0.218} & 0.225 & \textbf{0.205} \\
        \bottomrule
    \end{tabular}
    \vspace{-2pt}
    \label{tab:auc-rc-cf100-wild}
\end{table}

\begin{table}[t]
    \centering
    \scriptsize
    \caption{AUC-RC ($\downarrow$)  for methods trained with ImageNet as the inlier dataset and \emph{without} OOD samples. The base model is a pre-trained BiT ResNet-101. 
    \emph{Lower} values are \emph{better}.
    }
    \vspace{-2pt}
    \begin{tabular}{@{}lcccccccc@{}}
        \toprule
        & \multicolumn{7}{c}{ID-only training}\\
        Method / $\Pr^{\rm te}_{\rm out}$ & Places & LSUN & CelebA & Colorectal &
        iNaturalist-O & Texture &
        OpenImages-O &  ImageNet-O \\
        \toprule
        MSP &  0.227 & 0.234 & 0.241 & 0.218 & 0.195 & \textbf{0.220} & 0.203 & 0.325 \\[3pt]
        MaxLogit & 0.229 & 0.239 & 0.256 & 0.204 & 0.195 & 0.223 & 0.202 & 0.326 \\[3pt]
        Energy & 0.235 & 0.246 & 0.278 & 0.204 & 0.199 & 0.227 & 0.210 & 0.330 \\
        \midrule
        SIRC [$L_1$] & 0.222 & 0.229 & 0.248 & 0.220 & 0.196 & 0.226 & \textbf{0.200} & \textbf{0.313} \\[2pt]
        SIRC [Res]  & 0.211 & 0.198 & 0.178 & 0.161 & 0.175 & \textbf{0.219} & \textbf{0.201} & 0.327 \\[1pt]
        \midrule
        Plug-in BB [$L_1$] & 0.261 & 0.257 & 0.337 & 0.283 & 0.219 & 0.270 & 0.222 & 0.333 \\[1pt]
        Plug-in BB [Res] & \textbf{0.191} & \textbf{0.170} & \textbf{0.145} & \textbf{0.149} & \textbf{0.162} & 0.252 & 0.215 & 0.378  \\
        \bottomrule
    \end{tabular}
    \label{tab:auc-rc-imagenet-id-only}
\end{table}

\textbf{Baselines.} 
Our \emph{primary competitor is SIRC}, the only prior method that jointly tackles both selective classification and OOD detection. 
We compare with two variants of this method, 
which respectively use the $L_1$-norm of the embeddings as the OOD detection score, 
and a residual score \cite{wang2022vim} instead. 

We additionally compare with representative methods from the OOD detection and SCOD literature. This includes ones that train only on the ID samples, namely, MSP \citep{Chow:1970}, MaxLogit \citep{Hendrickx:2021}, energy-based scorer \citep{Hendrickx:2021}, and SIRC \citep{Xia:2022}, and those which additionally use OOD samples, namely, the coupled CE loss (CCE) \citep{Thulasidasan:2021}, the de-coupled CE loss (DCE) \citep{Bitterwolf:2022}, and the outlier exposure (OE) \citep{Hendrycks:2019}.  
In Appendix \ref{app:expts}, we also compare against  cost-sensitive softmax (CSS) loss~\citep{Mozannar:2020}, a representative SC baseline, and ODIN \citep{Liang:2018}.
With each method, we tune the threshold or cost parameter to achieve a given rate of abstention, and aggregate performance across different abstention rates (details in Appendix \ref{app:expts}). 
\textbf{Plug-in estimators.} We evaluate three variants of our proposed estimators: (i)  black-box rejector in \eqref{eqn:plug-in-black-box} using the $L_1$ scorer of \citet{Xia:2022} for $s_{\rm ood}$, (ii) black-box rejector in \eqref{eqn:plug-in-black-box} using their residual scorer, and (iii)  loss-based rejector using the de-coupled (DC) loss in \eqref{eqn:decoupled-surrogate}. Of these, (i) and (ii) use only ID samples for training; (iii) uses both ID and OOD samples for training.

\textbf{Results.}\ 
Our first experiments use CIFAR-100 as the ID sample. Table \ref{tab:auc-rc-cf100-random-100k} reports results for a setting where the OOD samples used (as a part of the wild set) during training  are different from those used for testing ($\Pr^{\rm tr}_{\rm out} \ne \Pr^{\rm te}_{\rm out}$). Table
\ref{tab:auc-rc-cf100-wild} contains results for a setting where they are the same ($\Pr^{\rm tr}_{\rm out} = \Pr^{\rm te}_{\rm out}$). 

In both cases, \textit{one among the three plug-in estimators yields the lowest AUC-RC}. Interestingly, when $\Pr^{\rm tr}_{\rm out} \ne \Pr^{\rm te}_{\rm out}$, the two black-box (BB) plug-in estimators that use only ID-samples for training often fare better than the loss-based (LB) one which uses both ID and wild samples for training. This is likely due to the mismatch between the training and test OOD distributions resulting in the decoupled loss yielding poor estimates of  $\frac{\Pr_{\rm in}(x) }{\Pr_{\rm out}(x)}$.
When $\Pr^{\rm tr}_{\rm out} = \Pr^{\rm te}_{\rm out}$, the LB estimator often performs the best.

In Table \ref{tab:auc-rc-imagenet-id-only}, we present results with ImageNet as ID, and no OOD samples for training. The  BB plug-in estimator (residual) yields notable gains on 5/8 OOD datasets. On the remaining, even the SIRC baselines are often only marginally better than MSP; this is because the grad-norm  scorers used by them (and also by our estimators) are not very effective in detecting OOD samples for these datasets.
\section{Discussion and future work}
\vspace{-3pt}
We have provided theoretically grounded plug-in estimators for 
SCOD
and demonstrated their efficacy on both settings that train with only ID samples, and those that additionally use a noisy OOD sample. 
A key element in our approach is an estimator for the ID-OOD density ratio, for which we  used grad-norm based scorers \cite{wang2022vim} as representative methods. In the future, we wish to explore other approaches for estimating the density ratio (e.g., \citep{Ren:2019}). We also wish to study the fairness implications of our approach on rare subgroups \citep{jonesselective}; we discuss  this and other limitations in Appendix \ref{app:limitations}.

\bibliographystyle{plainnat}
\bibliography{ood_references,l2r_references}

\begin{thebibliography}{55}
\providecommand{\natexlab}[1]{#1}
\providecommand{\url}[1]{\texttt{#1}}
\expandafter\ifx\csname urlstyle\endcsname\relax
  \providecommand{\doi}[1]{doi: #1}\else
  \providecommand{\doi}{doi: \begingroup \urlstyle{rm}\Url}\fi

\bibitem[Bartlett and Wegkamp(2008)]{Bartlett:2008}
Peter~L. Bartlett and Marten~H. Wegkamp.
\newblock Classification with a reject option using a hinge loss.
\newblock \emph{Journal of Machine Learning Research}, 9\penalty0
  (59):\penalty0 1823--1840, 2008.

\bibitem[Bendale and Boult(2016)]{bendale2016towards}
Abhijit Bendale and Terrance~E Boult.
\newblock Towards open set deep networks.
\newblock In \emph{Proceedings of the IEEE conference on computer vision and
  pattern recognition}, pages 1563--1572, 2016.

\bibitem[Bitterwolf et~al.(2022)Bitterwolf, Meinke, Augustin, and
  Hein]{Bitterwolf:2022}
Julian Bitterwolf, Alexander Meinke, Maximilian Augustin, and Matthias Hein.
\newblock Breaking down out-of-distribution detection: Many methods based on
  {OOD} training data estimate a combination of the same core quantities.
\newblock In Kamalika Chaudhuri, Stefanie Jegelka, Le~Song, Csaba Szepesvari,
  Gang Niu, and Sivan Sabato, editors, \emph{Proceedings of the 39th
  International Conference on Machine Learning}, volume 162 of
  \emph{Proceedings of Machine Learning Research}, pages 2041--2074. PMLR,
  17--23 Jul 2022.

\bibitem[Cen et~al.(2023)Cen, Luan, Zhang, Pei, Zhang, Zhao, Shen, and
  Chen]{Cen:2023}
Jun Cen, Di~Luan, Shiwei Zhang, Yixuan Pei, Yingya Zhang, Deli Zhao, Shaojie
  Shen, and Qifeng Chen.
\newblock The devil is in the wrongly-classified samples: Towards unified
  open-set recognition.
\newblock In \emph{The Eleventh International Conference on Learning
  Representations}, 2023.
\newblock URL \url{https://openreview.net/forum?id=xLr0I_xYGAs}.

\bibitem[Chandola et~al.(2009)Chandola, Banerjee, and Kumar]{Chandola:2009}
Varun Chandola, Arindam Banerjee, and Vipin Kumar.
\newblock Anomaly detection: A survey.
\newblock \emph{ACM Comput. Surv.}, 41\penalty0 (3), jul 2009.
\newblock ISSN 0360-0300.
\newblock \doi{10.1145/1541880.1541882}.
\newblock URL \url{https://doi.org/10.1145/1541880.1541882}.

\bibitem[Charoenphakdee et~al.(2021)Charoenphakdee, Cui, Zhang, and
  Sugiyama]{Charoenphakdee:2021}
Nontawat Charoenphakdee, Zhenghang Cui, Yivan Zhang, and Masashi Sugiyama.
\newblock Classification with rejection based on cost-sensitive classification.
\newblock In Marina Meila and Tong Zhang, editors, \emph{Proceedings of the
  38th International Conference on Machine Learning}, volume 139 of
  \emph{Proceedings of Machine Learning Research}, pages 1507--1517. PMLR,
  18--24 Jul 2021.

\bibitem[Chow(1970)]{Chow:1970}
C.~Chow.
\newblock On optimum recognition error and reject tradeoff.
\newblock \emph{IEEE Transactions on Information Theory}, 16\penalty0
  (1):\penalty0 41--46, 1970.
\newblock \doi{10.1109/TIT.1970.1054406}.

\bibitem[Cimpoi et~al.(2014)Cimpoi, Maji, Kokkinos, Mohamed, and
  Vedaldi]{cimpoi2014describing}
Mircea Cimpoi, Subhransu Maji, Iasonas Kokkinos, Sammy Mohamed, and Andrea
  Vedaldi.
\newblock Describing textures in the wild.
\newblock In \emph{Proceedings of the IEEE conference on computer vision and
  pattern recognition}, pages 3606--3613, 2014.

\bibitem[Cortes et~al.(2016{\natexlab{a}})Cortes, DeSalvo, and
  Mohri]{CorDeSMoh2016}
Corinna Cortes, Giulia DeSalvo, and Mehryar Mohri.
\newblock Boosting with abstention.
\newblock \emph{Advances in Neural Information Processing Systems},
  29:\penalty0 1660--1668, 2016{\natexlab{a}}.

\bibitem[Cortes et~al.(2016{\natexlab{b}})Cortes, DeSalvo, and
  Mohri]{Cortes:2016}
Corinna Cortes, Giulia DeSalvo, and Mehryar Mohri.
\newblock Learning with rejection.
\newblock In \emph{ALT}, 2016{\natexlab{b}}.

\bibitem[Deng et~al.(2009)Deng, Dong, Socher, Li, Li, and Fei-Fei]{Deng:2009}
Jia Deng, Wei Dong, Richard Socher, Li-Jia Li, Kai Li, and Li~Fei-Fei.
\newblock Imagenet: A large-scale hierarchical image database.
\newblock In \emph{2009 IEEE conference on computer vision and pattern
  recognition}, pages 248--255. Ieee, 2009.

\bibitem[Dhamija et~al.(2018)Dhamija, G\"{u}nther, and Boult]{Dhamija:2018}
Akshay~Raj Dhamija, Manuel G\"{u}nther, and Terrance~E. Boult.
\newblock Reducing network agnostophobia.
\newblock In \emph{Proceedings of the 32nd International Conference on Neural
  Information Processing Systems}, NIPS'18, page 9175–9186, Red Hook, NY,
  USA, 2018. Curran Associates Inc.

\bibitem[Elkan(2001)]{Elkan:2001}
Charles Elkan.
\newblock The foundations of cost-sensitive learning.
\newblock In \emph{In Proceedings of the Seventeenth International Joint
  Conference on Artificial Intelligence}, pages 973--978, 2001.

\bibitem[Gangrade et~al.(2021)Gangrade, Kag, and Saligrama]{Gangrade:2021}
Aditya Gangrade, Anil Kag, and Venkatesh Saligrama.
\newblock Selective classification via one-sided prediction.
\newblock In Arindam Banerjee and Kenji Fukumizu, editors, \emph{Proceedings of
  The 24th International Conference on Artificial Intelligence and Statistics},
  volume 130 of \emph{Proceedings of Machine Learning Research}, pages
  2179--2187. PMLR, 13--15 Apr 2021.
\newblock URL \url{https://proceedings.mlr.press/v130/gangrade21a.html}.

\bibitem[Garg et~al.(2021)Garg, Wu, Smola, Balakrishnan, and
  Lipton]{garg2021mixture}
Saurabh Garg, Yifan Wu, Alexander~J Smola, Sivaraman Balakrishnan, and Zachary
  Lipton.
\newblock Mixture proportion estimation and pu learning: A modern approach.
\newblock \emph{Advances in Neural Information Processing Systems},
  34:\penalty0 8532--8544, 2021.

\bibitem[Geifman and El-Yaniv(2019)]{Geifman:2019}
Yonatan Geifman and Ran El-Yaniv.
\newblock {S}elective{N}et: A deep neural network with an integrated reject
  option.
\newblock In Kamalika Chaudhuri and Ruslan Salakhutdinov, editors,
  \emph{Proceedings of the 36th International Conference on Machine Learning},
  volume~97 of \emph{Proceedings of Machine Learning Research}, pages
  2151--2159. PMLR, 09--15 Jun 2019.

\bibitem[Hendrickx et~al.(2021)Hendrickx, Perini, der Plas, Meert, and
  Davis]{Hendrickx:2021}
Kilian Hendrickx, Lorenzo Perini, Dries~Van der Plas, Wannes Meert, and Jesse
  Davis.
\newblock Machine learning with a reject option: {A} survey.
\newblock \emph{CoRR}, abs/2107.11277, 2021.

\bibitem[Hendrycks and Gimpel(2017)]{Hendrycks:2017}
Dan Hendrycks and Kevin Gimpel.
\newblock A baseline for detecting misclassified and out-of-distribution
  examples in neural networks.
\newblock In \emph{International Conference on Learning Representations}, 2017.
\newblock URL \url{https://openreview.net/forum?id=Hkg4TI9xl}.

\bibitem[Hendrycks et~al.(2018)Hendrycks, Mazeika, Wilson, and
  Gimpel]{Hendrycks:2018}
Dan Hendrycks, Mantas Mazeika, Duncan Wilson, and Kevin Gimpel.
\newblock Using trusted data to train deep networks on labels corrupted by
  severe noise.
\newblock In S.~Bengio, H.~Wallach, H.~Larochelle, K.~Grauman, N.~Cesa-Bianchi,
  and R.~Garnett, editors, \emph{Advances in Neural Information Processing
  Systems}, volume~31. Curran Associates, Inc., 2018.
\newblock URL
  \url{https://proceedings.neurips.cc/paper_files/paper/2018/file/ad554d8c3b06d6b97ee76a2448bd7913-Paper.pdf}.

\bibitem[Hendrycks et~al.(2019)Hendrycks, Mazeika, and
  Dietterich]{Hendrycks:2019}
Dan Hendrycks, Mantas Mazeika, and Thomas Dietterich.
\newblock Deep anomaly detection with outlier exposure.
\newblock \emph{Proceedings of the International Conference on Learning
  Representations}, 2019.

\bibitem[Hendrycks et~al.(2022)Hendrycks, Basart, Mazeika, Zou, Kwon,
  Mostajabi, Steinhardt, and Song]{Hendrycks:2022}
Dan Hendrycks, Steven Basart, Mantas Mazeika, Andy Zou, Joseph Kwon,
  Mohammadreza Mostajabi, Jacob Steinhardt, and Dawn Song.
\newblock Scaling out-of-distribution detection for real-world settings.
\newblock In Kamalika Chaudhuri, Stefanie Jegelka, Le~Song, Csaba Szepesvari,
  Gang Niu, and Sivan Sabato, editors, \emph{Proceedings of the 39th
  International Conference on Machine Learning}, volume 162 of
  \emph{Proceedings of Machine Learning Research}, pages 8759--8773. PMLR,
  17--23 Jul 2022.
\newblock URL \url{https://proceedings.mlr.press/v162/hendrycks22a.html}.

\bibitem[Huang and Li(2021)]{Huang:2021b}
Rui Huang and Yixuan Li.
\newblock Mos: Towards scaling out-of-distribution detection for large semantic
  space.
\newblock In \emph{Proceedings of the IEEE/CVF Conference on Computer Vision
  and Pattern Recognition}, 2021.

\bibitem[Huang et~al.(2021)Huang, Geng, and Li]{Huang:2021}
Rui Huang, Andrew Geng, and Yixuan Li.
\newblock On the importance of gradients for detecting distributional shifts in
  the wild.
\newblock In A.~Beygelzimer, Y.~Dauphin, P.~Liang, and J.~Wortman Vaughan,
  editors, \emph{Advances in Neural Information Processing Systems}, 2021.
\newblock URL \url{https://openreview.net/forum?id=fmiwLdJCmLS}.

\bibitem[Jones et~al.(2021)Jones, Sagawa, Koh, Kumar, and
  Liang]{jonesselective}
Erik Jones, Shiori Sagawa, Pang~Wei Koh, Ananya Kumar, and Percy Liang.
\newblock Selective classification can magnify disparities across groups.
\newblock In \emph{International Conference on Learning Representations}, 2021.

\bibitem[Kather et~al.(2016)Kather, Weis, Bianconi, Melchers, Schad, Gaiser,
  Marx, and Z{\"o}llner]{kather2016multi}
Jakob~Nikolas Kather, Cleo-Aron Weis, Francesco Bianconi, Susanne~M Melchers,
  Lothar~R Schad, Timo Gaiser, Alexander Marx, and Frank~Gerrit Z{\"o}llner.
\newblock Multi-class texture analysis in colorectal cancer histology.
\newblock \emph{Scientific reports}, 6\penalty0 (1):\penalty0 1--11, 2016.

\bibitem[Katz-Samuels et~al.(2022)Katz-Samuels, Nakhleh, Nowak, and
  Li]{Katz:2022}
Julian Katz-Samuels, Julia~B Nakhleh, Robert Nowak, and Yixuan Li.
\newblock Training {OOD} detectors in their natural habitats.
\newblock In Kamalika Chaudhuri, Stefanie Jegelka, Le~Song, Csaba Szepesvari,
  Gang Niu, and Sivan Sabato, editors, \emph{Proceedings of the 39th
  International Conference on Machine Learning}, volume 162 of
  \emph{Proceedings of Machine Learning Research}, pages 10848--10865. PMLR,
  17--23 Jul 2022.

\bibitem[Kim et~al.(2021)Kim, Koo, and Hwang]{Kim:2021}
Jihyo Kim, Jiin Koo, and Sangheum Hwang.
\newblock A unified benchmark for the unknown detection capability of deep
  neural networks, 2021.

\bibitem[Krasin et~al.(2017)Krasin, Duerig, Alldrin, Ferrari, Abu-El-Haija,
  Kuznetsova, Rom, Uijlings, Popov, Veit, et~al.]{krasin2017openimages}
Ivan Krasin, Tom Duerig, Neil Alldrin, Vittorio Ferrari, Sami Abu-El-Haija,
  Alina Kuznetsova, Hassan Rom, Jasper Uijlings, Stefan Popov, Andreas Veit,
  et~al.
\newblock Openimages: A public dataset for large-scale multi-label and
  multi-class image classification.
\newblock \emph{Dataset available from https://github. com/openimages},
  2\penalty0 (3):\penalty0 18, 2017.

\bibitem[Krizhevsky(2009)]{Krizhevsky:2009}
Alex Krizhevsky.
\newblock Learning multiple layers of features from tiny images.
\newblock Technical report, University of Toronto, 2009.

\bibitem[Lee et~al.(2018)Lee, Lee, Lee, and Shin]{Lee:2018}
Kimin Lee, Honglak Lee, Kibok Lee, and Jinwoo Shin.
\newblock Training confidence-calibrated classifiers for detecting
  out-of-distribution samples.
\newblock In \emph{International Conference on Learning Representations}, 2018.
\newblock URL \url{https://openreview.net/forum?id=ryiAv2xAZ}.

\bibitem[Liang et~al.(2018)Liang, Li, and Srikant]{Liang:2018}
Shiyu Liang, Yixuan Li, and R.~Srikant.
\newblock Enhancing the reliability of out-of-distribution image detection in
  neural networks.
\newblock In \emph{International Conference on Learning Representations}, 2018.
\newblock URL \url{https://openreview.net/forum?id=H1VGkIxRZ}.

\bibitem[Liu et~al.(2020{\natexlab{a}})Liu, Zhou, Zhao, Wang, Deng, and
  Ju]{Liu:2020}
Weijie Liu, Peng Zhou, Zhe Zhao, Zhiruo Wang, Haotang Deng, and Qi~Ju.
\newblock {FastBERT}: a self-distilling bert with adaptive inference time.
\newblock In \emph{Proceedings of ACL 2020}, 2020{\natexlab{a}}.

\bibitem[Liu et~al.(2020{\natexlab{b}})Liu, Wang, Owens, and Li]{Liu:2020b}
Weitang Liu, Xiaoyun Wang, John Owens, and Yixuan Li.
\newblock Energy-based out-of-distribution detection.
\newblock In H.~Larochelle, M.~Ranzato, R.~Hadsell, M.F. Balcan, and H.~Lin,
  editors, \emph{Advances in Neural Information Processing Systems}, volume~33,
  pages 21464--21475. Curran Associates, Inc., 2020{\natexlab{b}}.
\newblock URL
  \url{https://proceedings.neurips.cc/paper/2020/file/f5496252609c43eb8a3d147ab9b9c006-Paper.pdf}.

\bibitem[Liu et~al.(2015)Liu, Luo, Wang, and Tang]{liu2015faceattributes}
Ziwei Liu, Ping Luo, Xiaogang Wang, and Xiaoou Tang.
\newblock Deep learning face attributes in the wild.
\newblock In \emph{Proceedings of International Conference on Computer Vision
  (ICCV)}, December 2015.

\bibitem[Mozannar and Sontag(2020)]{Mozannar:2020}
Hussein Mozannar and David Sontag.
\newblock Consistent estimators for learning to defer to an expert.
\newblock In Hal~Daumé III and Aarti Singh, editors, \emph{Proceedings of the
  37th International Conference on Machine Learning}, volume 119 of
  \emph{Proceedings of Machine Learning Research}, pages 7076--7087. PMLR,
  13--18 Jul 2020.

\bibitem[Nalisnick et~al.(2019)Nalisnick, Matsukawa, Teh, G{\"{o}}r{\"{u}}r,
  and Lakshminarayanan]{Nalisnick:2019}
Eric~T. Nalisnick, Akihiro Matsukawa, Yee~Whye Teh, Dilan G{\"{o}}r{\"{u}}r,
  and Balaji Lakshminarayanan.
\newblock Do deep generative models know what they don't know?
\newblock In \emph{7th International Conference on Learning Representations,
  {ICLR} 2019, New Orleans, LA, USA, May 6-9, 2019}. OpenReview.net, 2019.
\newblock URL \url{https://openreview.net/forum?id=H1xwNhCcYm}.

\bibitem[Netzer et~al.(2011)Netzer, Wang, Coates, Bissacco, Wu, and
  Ng]{netzer2011reading}
Yuval Netzer, Tao Wang, Adam Coates, Alessandro Bissacco, Bo~Wu, and Andrew~Y
  Ng.
\newblock Reading digits in natural images with unsupervised feature learning.
\newblock NIPS Workshop on Deep Learning and Unsupervised Feature Learning
  2011, 2011.

\bibitem[Nguyen et~al.(2015)Nguyen, Yosinski, and Clune]{Nguyen:2015}
Anh Nguyen, Jason Yosinski, and Jeff Clune.
\newblock Deep neural networks are easily fooled: High confidence predictions
  for unrecognizable images.
\newblock In \emph{2015 IEEE Conference on Computer Vision and Pattern
  Recognition (CVPR)}, pages 427--436, 2015.
\newblock \doi{10.1109/CVPR.2015.7298640}.

\bibitem[Ni et~al.(2019)Ni, Charoenphakdee, Honda, and Sugiyama]{Ni:2019}
Chenri Ni, Nontawat Charoenphakdee, Junya Honda, and Masashi Sugiyama.
\newblock On the calibration of multiclass classification with rejection.
\newblock In Hanna~M. Wallach, Hugo Larochelle, Alina Beygelzimer, Florence
  d'Alch{\'{e}}{-}Buc, Emily~B. Fox, and Roman Garnett, editors, \emph{Advances
  in Neural Information Processing Systems 32: Annual Conference on Neural
  Information Processing Systems 2019, NeurIPS 2019, December 8-14, 2019,
  Vancouver, BC, Canada}, pages 2582--2592, 2019.

\bibitem[Patrini et~al.(2017)Patrini, Rozza, Menon, Nock, and Qu]{Patrini:2017}
Giorgio Patrini, Alessandro Rozza, Aditya~Krishna Menon, Richard Nock, and
  Lizhen Qu.
\newblock Making deep neural networks robust to label noise: a loss correction
  approach.
\newblock In \emph{Computer Vision and Pattern Recognition ({CVPR})}, pages
  2233--2241, 2017.

\bibitem[Ramaswamy et~al.(2018)Ramaswamy, Tewari, and Agarwal]{Ramaswamy:2018}
Harish~G. Ramaswamy, Ambuj Tewari, and Shivani Agarwal.
\newblock {Consistent algorithms for multiclass classification with an abstain
  option}.
\newblock \emph{Electronic Journal of Statistics}, 12\penalty0 (1):\penalty0
  530 -- 554, 2018.
\newblock \doi{10.1214/17-EJS1388}.

\bibitem[Reid and Williamson(2010)]{Reid:2010}
Mark~D. Reid and Robert~C. Williamson.
\newblock Composite binary losses.
\newblock \emph{Journal of Machine Learning Research}, 11:\penalty0 2387--2422,
  2010.

\bibitem[Ren et~al.(2019)Ren, Liu, Fertig, Snoek, Poplin, DePristo, Dillon, and
  Lakshminarayanan]{Ren:2019}
Jie Ren, Peter~J. Liu, Emily Fertig, Jasper Snoek, Ryan Poplin, Mark~A.
  DePristo, Joshua~V. Dillon, and Balaji Lakshminarayanan.
\newblock \emph{Likelihood Ratios for Out-of-Distribution Detection}, pages
  14707–--14718.
\newblock Curran Associates Inc., Red Hook, NY, USA, 2019.

\bibitem[Scheirer et~al.(2013)Scheirer, de~Rezende~Rocha, Sapkota, and
  Boult]{Scheirer:2013}
Walter~J. Scheirer, Anderson de~Rezende~Rocha, Archana Sapkota, and Terrance~E.
  Boult.
\newblock Toward open set recognition.
\newblock \emph{IEEE Transactions on Pattern Analysis and Machine
  Intelligence}, 35\penalty0 (7):\penalty0 1757--1772, 2013.
\newblock \doi{10.1109/TPAMI.2012.256}.

\bibitem[Steinwart et~al.(2005)Steinwart, Hush, and Scovel]{Steinwart:2005}
Ingo Steinwart, Don Hush, and Clint Scovel.
\newblock A classification framework for anomaly detection.
\newblock \emph{Journal of Machine Learning Research}, 6\penalty0 (8):\penalty0
  211--232, 2005.
\newblock URL \url{http://jmlr.org/papers/v6/steinwart05a.html}.

\bibitem[Sun et~al.(2022)Sun, Ming, Zhu, and Li]{Sun:2022}
Yiyou Sun, Yifei Ming, Xiaojin Zhu, and Yixuan Li.
\newblock Out-of-distribution detection with deep nearest neighbors.
\newblock In Kamalika Chaudhuri, Stefanie Jegelka, Le~Song, Csaba Szepesvari,
  Gang Niu, and Sivan Sabato, editors, \emph{Proceedings of the 39th
  International Conference on Machine Learning}, volume 162 of
  \emph{Proceedings of Machine Learning Research}, pages 20827--20840. PMLR,
  17--23 Jul 2022.
\newblock URL \url{https://proceedings.mlr.press/v162/sun22d.html}.

\bibitem[Thulasidasan et~al.(2019)Thulasidasan, Bhattacharya, Bilmes,
  Chennupati, and Mohd-Yusof]{Thulasidasan:2019}
Sunil Thulasidasan, Tanmoy Bhattacharya, Jeff Bilmes, Gopinath Chennupati, and
  Jamal Mohd-Yusof.
\newblock Combating label noise in deep learning using abstention.
\newblock In Kamalika Chaudhuri and Ruslan Salakhutdinov, editors,
  \emph{Proceedings of the 36th International Conference on Machine Learning},
  volume~97 of \emph{Proceedings of Machine Learning Research}, pages
  6234--6243, Long Beach, California, USA, 09--15 Jun 2019. PMLR.

\bibitem[Thulasidasan et~al.(2021)Thulasidasan, Thapa, Dhaubhadel, Chennupati,
  Bhattacharya, and Bilmes]{Thulasidasan:2021}
Sunil Thulasidasan, Sushil Thapa, Sayera Dhaubhadel, Gopinath Chennupati,
  Tanmoy Bhattacharya, and Jeff~A. Bilmes.
\newblock An effective baseline for robustness to distributional shift.
\newblock \emph{CoRR}, abs/2105.07107, 2021.
\newblock URL \url{https://arxiv.org/abs/2105.07107}.

\bibitem[Vaze et~al.(2021)Vaze, Han, Vedaldi, and Zisserman]{vaze2021open}
Sagar Vaze, Kai Han, Andrea Vedaldi, and Andrew Zisserman.
\newblock Open-set recognition: A good closed-set classifier is all you need.
\newblock \emph{arXiv preprint arXiv:2110.06207}, 2021.

\bibitem[Verma and Nalisnick(2022)]{verma2022calibrated}
Rajeev Verma and Eric Nalisnick.
\newblock Calibrated learning to defer with one-vs-all classifiers.
\newblock \emph{arXiv preprint arXiv:2202.03673}, 2022.

\bibitem[Wang et~al.(2022)Wang, Li, Feng, and Zhang]{wang2022vim}
Haoqi Wang, Zhizhong Li, Litong Feng, and Wayne Zhang.
\newblock Vim: Out-of-distribution with virtual-logit matching.
\newblock In \emph{Proceedings of the IEEE/CVF Conference on Computer Vision
  and Pattern Recognition}, pages 4921--4930, 2022.

\bibitem[Wei et~al.(2022)Wei, Xie, Cheng, Feng, An, and Li]{Wei:2022}
Hongxin Wei, Renchunzi Xie, Hao Cheng, Lei Feng, Bo~An, and Yixuan Li.
\newblock Mitigating neural network overconfidence with logit normalization.
\newblock In Kamalika Chaudhuri, Stefanie Jegelka, Le~Song, Csaba Szepesvari,
  Gang Niu, and Sivan Sabato, editors, \emph{Proceedings of the 39th
  International Conference on Machine Learning}, volume 162 of
  \emph{Proceedings of Machine Learning Research}, pages 23631--23644. PMLR,
  17--23 Jul 2022.
\newblock URL \url{https://proceedings.mlr.press/v162/wei22d.html}.

\bibitem[Xia and Bouganis(2022)]{Xia:2022}
Guoxuan Xia and Christos-Savvas Bouganis.
\newblock Augmenting softmax information for selective classification with
  out-of-distribution data.
\newblock \emph{ArXiv}, abs/2207.07506, 2022.

\bibitem[Yu et~al.(2015)Yu, Seff, Zhang, Song, Funkhouser, and
  Xiao]{yu2015lsun}
Fisher Yu, Ari Seff, Yinda Zhang, Shuran Song, Thomas Funkhouser, and Jianxiong
  Xiao.
\newblock Lsun: Construction of a large-scale image dataset using deep learning
  with humans in the loop.
\newblock \emph{arXiv preprint arXiv:1506.03365}, 2015.

\bibitem[Zhou et~al.(2017)Zhou, Lapedriza, Khosla, Oliva, and
  Torralba]{zhou2017places}
Bolei Zhou, Agata Lapedriza, Aditya Khosla, Aude Oliva, and Antonio Torralba.
\newblock Places: A 10 million image database for scene recognition.
\newblock \emph{IEEE transactions on pattern analysis and machine
  intelligence}, 40\penalty0 (6):\penalty0 1452--1464, 2017.

\end{thebibliography}

\newpage
\appendix

\addcontentsline{toc}{section}{Appendix} 
\part{Appendix} 
\parttoc 

\section{Proofs}

\begin{proof}[Proof of Lemma~\ref{lem:scod-bayes}]
We first define a joint marginal distribution $\Pcomb$ that samples from $\Pr_{\rm in}(x)$ and $\Pr_{\rm out}(x)$ with equal probabilities.
We then rewrite the objective in \eqref{eqn:scod-soft} in terms of the joint marginal distribution:
\begin{align*}
    {L_{\rm scod}(h, r)} &= \Ex_{x \sim \Pcomb}\left[ T_1( h( x ), r( x ) ) + T_2( h( x ), r( x ) ) \right] \\
    T_1( h( x ), r( x ) ) &= ( 1 - \costin - \costout ) \cdot \Ex_{y|x \sim \Pr_{\rm in}}\left[\frac{\Pr_{\rm in}(x)}{\Pcomb(x)} \cdot \1( y \neq h( x ), h( x ) \neq \abstain) \right] \\
    &= ( 1 - \costin - \costout ) \cdot \sum_{y \in [L]}\Pr_{\rm in}(y|x) \cdot \frac{\Pr_{\rm in}(x)}{\Pcomb(x)} \cdot \1( y \neq h( x ), h( x ) \neq \abstain) \\
    T_2( h( x ), r( x ) ) &= \costin \cdot \frac{\Pr_{\rm in}(x)}{\Pcomb(x)} \cdot\1( h( x ) = \abstain ) + + \costout \cdot \1( h( x ) \neq \abstain ).
\end{align*}

The conditional risk that a classifier $h$ incurs when abstaining (i.e., predicting $r( x ) = 1$) on a fixed instance $x$ is given by:
\[
\costin \cdot \frac{\Pr_{\rm in}(x)}{\Pcomb(x)}.
\]

The conditional risk associated with predicting a base class $y \in [L]$ on instance $x$ is given by:
\[
( 1 - \costin - \costout ) \cdot \frac{\Pr_{\rm in}(x)}{\Pcomb(x)} \cdot \left( 1 - \Pr_{\rm in}(y|x) \right)
        + \costout \cdot \frac{\Pr_{\rm out}(x)}{\Pcomb(x)} 
\]
The Bayes-optimal classifier then predicts the label with the lowest conditional risk. 
When $\Pr_{\rm in}( x ) = 0$, this amounts to predicting abstain ($r( x ) = 1$). 
When $\Pr_{\rm in}( x ) > 0$, the optimal classifier predicts $r( x ) = 1$ when:
\begin{align*}
    &\costin \cdot \frac{\Pr_{\rm in}(x)}{\Pcomb(x)}
     < ( 1 - \costin - \costout ) \cdot \frac{\Pr_{\rm in}(x)}{\Pcomb(x)} \cdot \min_{y \in [L]}\left( 1 - \Pr_{\rm in}(y|x) \right)
        + \costout \cdot \frac{\Pr_{\rm out}(x)}{\Pcomb(x)} 
    \\
    &\iff 
    \costin \cdot {\Pr_{\rm in}(x)}
     < ( 1 - \costin - \costout ) \cdot {\Pr_{\rm in}(x)} \cdot \min_{y \in [L]}\left( 1 - \Pr_{\rm in}(y|x) \right)
        + \costout \cdot {\Pr_{\rm out}(x)}
    \\    
    &\iff
    \costin \cdot {\Pr_{\rm in}(x)}
     < ( 1 - \costin - \costout ) \cdot {\Pr_{\rm in}(x)} \cdot \left( 1 - \max_{y \in [L]}\Pr_{\rm in}(y|x) \right)
        + \costout \cdot {\Pr_{\rm out}(x)} \\
    &\iff
    \costin
     < ( 1 - \costin - \costout ) \cdot \left( 1 - \max_{y \in [L]}\Pr_{\rm in}(y|x) \right)
        + \costout \cdot \frac{\Pr_{\rm out}(x)}{\Pr_{\rm in}(x)}.
\end{align*}
Otherwise, the classifier does not abstain ($r( x ) = 0$),
and predicts $\argmax_{y \in [L]}\, \Pr_{\rm in}(y|x)$, as desired.
\end{proof}

\begin{proof}[Proof of Lemma~\ref{lemm:bayes-open-set-rewrite}]
Recall that in open-set classification,
the outlier distribution is $\Pr_{\rm out}( x ) = \Pr_{\rm te}(x \mid y=L)$,
while
the training distribution is
\begin{align*}
    \PTr( x \mid y ) &= \mathbb{P}_{\rm te}( x \mid y ) \\
    \piTr( y ) &= \Pr_{\rm in}( y ) \\
    &= \frac{1( y \neq L )}{1 - \pi_{\rm te}( L )} \cdot \pi_{\rm te}( y ). 
\end{align*}
We will find it useful to derive the following quantities.
\begin{align*}
    \PTr( x, y ) &= \piTr( y ) \cdot \PTr( x \mid y ) \\
    &= \frac{1( y \neq L ) }{1 - \pi_{\rm te}( L )} \cdot \pi_{\rm te}( y ) \cdot \mathbb{P}_{\rm te}( x \mid y ) \\
    &= \frac{1( y \neq L ) }{1 - \pi_{\rm te}( L )} \cdot \mathbb{P}_{\rm te}( x, y ) \\
    \PTr( x ) &= \sum_{y \in [L]} \PTr( x, y ) \\
    &= \sum_{y \in [L]} \piTr( y ) \cdot \PTr( x \mid y ) \\
    &= \frac{1}{1 - \pi_{\rm te}( L )} \sum_{y \neq L} \pi_{\rm te}( y ) \cdot \mathbb{P}_{\rm te}( x \mid y ) \\
    &= \frac{1}{1 - \pi_{\rm te}( L )} \sum_{y \neq L} \mathbb{P}_{\rm te}( y \mid x ) \cdot \mathbb{P}_{\rm te}( x ) \\
    &= \frac{\mathbb{P}_{\rm te}( y \neq L \mid x )}{1 - \pi_{\rm te}( L )} \cdot \mathbb{P}_{\rm te}( x ) \\
    \PTr( y \mid x ) &= \frac{\PTr( x, y )}{\PTr( x )} \\
    &= \frac{1( y \neq L ) }{1 - \pi_{\rm te}( L )} \cdot \frac{1 - \pi_{\rm te}( L )}{\mathbb{P}_{\rm te}( y \neq L \mid x )} \cdot \frac{\mathbb{P}_{\rm te}( x, y )}{\mathbb{P}_{\rm te}( x )}  \\
    &= \frac{1( y \neq L )}{\mathbb{P}_{\rm te}( y \neq L \mid x )} \cdot \mathbb{P}_{\rm te}( y \mid x ).
\end{align*}

The first part follows from standard results in cost-sensitive learning \citep{Elkan:2001}:
\begin{align*}
    r^*(x) = 1
    &\,\iff\,
    \costin \cdot \Pr_{\rm in}( x ) - \costout \cdot \Pr_{\rm out}( x ) < 0 \\
    &\,\iff\,
    \costin \cdot \Pr_{\rm in}( x ) < \costout \cdot \Pr_{\rm out}( x ) \\
    &\,\iff\,
    \costin \cdot \Pr_{\rm te}( x \mid y \neq L ) < \costout \cdot \Pr_{\rm te}( x \mid y = L ) \\
    &\,\iff\,
    \costin \cdot \Pr_{\rm te}( y \neq L \mid x ) \cdot \Pr_{\rm te}( y = L ) < \costout \cdot \Pr_{\rm te}( y = L \mid x ) \cdot \Pr_{\rm te}( y \neq L ) \\
    &\,\iff\,
    \frac{\costin \cdot \Pr_{\rm te}( y = L )}{\costout \cdot \Pr_{\rm te}( y \neq L )} < \frac{\Pr_{\rm te}( y = L \mid x )}{\Pr_{\rm te}( y \neq L \mid x )} \\
    &\,\iff\,
    \mathbb{P}_{\rm te}( y = L \mid x ) > F\left( \frac{\costin \cdot \Pr_{\rm te}( y = L )}{\costout \cdot \Pr_{\rm te}( y \neq L )} \right).
\end{align*}
We further have for threshold $t^*_{\rm osc} \defEq F\left( \frac{\costin \cdot \Pr_{\rm te}( y = L )}{\costout \cdot \Pr_{\rm te}( y \neq L )} \right)$,
\begin{align*}
    \mathbb{P}_{\rm te}( y = L \mid x ) \geq t^*_{\rm osc} &\iff \mathbb{P}_{\rm te}( y \neq L \mid x ) \leq 1 - t^*_{\rm osc} \\
    &\iff \frac{1}{\mathbb{P}_{\rm te}( y \neq L \mid x )} \geq \frac{1}{1 - t^*_{\rm osc}} \\
    &\iff \frac{\max_{y' \neq L} \mathbb{P}_{\rm te}( y' \mid x)}{\mathbb{P}_{\rm te}( y \neq L \mid x )} \geq \frac{\max_{y' \neq L} \mathbb{P}_{\rm te}( y' \mid x)}{1 - t^*_{\rm osc}} \\
    &\iff \max_{y' \neq L} \PTr( y' \mid x ) \geq \frac{\max_{y' \neq L} \mathbb{P}_{\rm te}( y' \mid x)}{1 - t^*_{\rm osc}}.
\end{align*}
That is, we want to reject when the maximum softmax probability is \emph{higher} than some (sample-dependent) threshold.
\end{proof}

\begin{proof}[Proof of Lemma \ref{lem:chow-fail}]
Fix $\epsilon \in (0,1)$. 
We consider two cases for threshold $t_{\rm msp}$:

Case (i): $t_{\rm msp} \leq \frac{1}{L-1}$. Consider a distribution where for all instances $x$, $\mathbb{P}_{\rm te}( y = L \mid x ) = 1 - \epsilon$ and $\mathbb{P}_{\rm te}( y' \mid x) = \frac{\epsilon}{L-1}, \forall y' \ne L$. Then the Bayes-optimal classifier  accepts any instance $x$ for all thresholds $t \in \big(0, 1-\epsilon\big)$. In contrast, Chow's rule would compute $\max_{y \ne L}\PTr( y \mid x) = \frac{1}{L-1},$ and thus reject all instances $x$. 

Case (ii): $t_{\rm msp} > \frac{1}{L-1}$. Consider a distribution where for all instances $x$, $\mathbb{P}_{\rm te}( y = L \mid x ) = \epsilon$ and $\mathbb{P}_{\rm te}( y' \mid x) = \frac{1-\epsilon}{L-1}, \forall y' \ne L$. Then the Bayes-optimal classifier would reject any instance $x$ for thresholds $t \in \big(\epsilon, 1\big)$, whereas Chow's rule would accept all instances. 

Taking $\epsilon \rightarrow 0$ completes the proof. 
\end{proof}


%
\begin{proof}[Proof of Lemma~\ref{lemm:black-box-regret}]
Let $\Pr^*$ denote the joint distribution that draws a sample from $\Pr_{\rm in}$ and $\Pr_{\rm out}$ with equal probability. Denote $\gamma_{\rm in}(x) = \frac{ \Pr_{\rm in}(x) }{ \Pr_{\rm in}(x) + \Pr_{\rm out}(x) }$.

The joint risk in \eqref{eqn:scod-soft} can be written as:
\begin{align*}
\lefteqn{L_{\rm scod}(h, r) }\\
&= 
    (1 - \costin - \costout) \cdot \Pr_{\rm in}( y \neq {h}( x ), r( x ) = 0 ) + 
    \costin \cdot \Pr_{\rm in}( r( x ) = 1 ) + \costout \cdot \Pr_{\rm out}( r( x ) = 0 )\\
    &= 
    \Ex_{x \sim \Pr^*}\Big[ (1 - \costin - \costout) \cdot \gamma_{\rm in}(x)
     \cdot
    \sum_{y \ne h(x)} \Pr_{\rm in}( y \mid x) \cdot \1( r( x ) = 0 )  \\[-5pt]
    & \hspace{3.5cm} + 
    \costin \cdot \gamma_{\rm in}(x) \cdot \1( r( x ) = 1 )
    + \costout \cdot (1 - \gamma_{\rm in}(x)) \cdot \1( r( x ) = 0 ) \Big].
\end{align*}

 For class probability estimates $\hat{\Pr}_{\rm in}(y \mid x) \approx \Pr_{\rm in}(y \mid x)$, and scorers 
 $\hat{s}_{\rm sc}(x) = \max_{y \in [L]} \hat{\Pr}_{\rm in}(y \mid x)$ and
$\hat{s}_{\rm ood}(x) \approx \frac{ \Pr_{\rm in}(x) }{ \Pr_{\rm out}(x)}$, we construct a classifier $\hat{h}(x) \in \argmax_{y \in [L]} \hat{\eta}_y(x)$ and black-box rejector:
\begin{equation}
    \label{eqn:plug-in-black-box-rewritten}
    \hat{r}_{\rm BB}( x ) = 1 \iff ( 1 - \costin - \costout ) \cdot (1 - \hat{s}_{\rm sc}( x )) + {\costout} \cdot \left( \frac{ 1 }{ \hat{s}_{\rm ood}( x ) } \right) > c_{\rm in}.
\end{equation}

Let $(h^*, r^*)$ denote the optimal classifier and rejector as defined in \eqref{eqn:scod-bayes}. 
We then wish to bound the following regret:
\begin{align*}
    L_{\rm scod}(\hat{h}, \hat{r}_{\rm BB}) - L_{\rm scod}(h^*, r^*) &= 
    \underbrace{L_{\rm scod}(\hat{h}, \hat{r}_{\rm BB}) - L_{\rm scod}(h^*, \hat{r}_{\rm BB})}_{ \text{term}_1 } + \underbrace{L_{\rm scod}(h^*, \hat{r}_{\rm BB}) - L_{\rm scod}(h^*, r^*)}_{ \text{term}_2 }. 
\end{align*}
We first bound the first term:
\begin{align*}
    \text{term}_1 &= 
    \Ex_{x \sim \Pr^*}\left[ (1 - \costin - \costout) \cdot \gamma_{\rm in}(x)
     \cdot \1( \hat{r}_{\rm BB}( x ) = 0 ) \cdot
    \Big( \sum_{y \ne \hat{h}(x)} \Pr_{\rm in}( y \mid x) 
    - \sum_{y \ne h^*(x)} \Pr_{\rm in}( y \mid x) \Big)
     \right] \\
    &=  \Ex_{x \sim \Pr^*}\left[ \omega(x) \cdot
    \Big( \sum_{y \ne \hat{h}(x)} \Pr_{\rm in}( y \mid x) 
    - \sum_{y \ne h^*(x)} \Pr_{\rm in}( y \mid x) \Big)
     \right],
\end{align*}
where we denote $\omega(x) = (1 - \costin - \costout) \cdot \gamma_{\rm in}(x)
     \cdot \1( \hat{r}_{\rm BB}( x ) = 0 )$.
     
Furthermore, we can write:
\begin{align*}
\lefteqn{\text{term}_1}\\
    &=
    \Ex_{x \sim \Pr^*}\left[ \omega(x) \cdot
    \Big( \sum_{y \ne \hat{h}(x)} \Pr_{\rm in}( y \mid x) 
    -
    \sum_{y \ne h^*(x)} \hat{\Pr}_{\rm in}( y \mid x) 
    +
    \sum_{y \ne h^*(x)} \hat{\Pr}_{\rm in}( y \mid x) 
    - \sum_{y \ne h^*(x)} \Pr_{\rm in}( y \mid x) \Big)
     \right] \\
    &\leq
    \Ex_{x \sim \Pr^*}\left[ \omega(x) \cdot
    \Big( \sum_{y \ne \hat{h}(x)} \Pr_{\rm in}( y \mid x) 
    -
    \sum_{y \ne \hat{h}(x)} \hat{\Pr}_{\rm in}( y \mid x) 
    +
    \sum_{y \ne h^*(x)} \hat{\Pr}_{\rm in}( y \mid x) 
    - \sum_{y \ne h^*(x)} \Pr_{\rm in}( y \mid x) \Big)
     \right]\\
    &\leq
    2 \cdot \Ex_{x \sim \Pr^*}\left[ \omega(x) \cdot
    \sum_{y  \in [L]}
    \left| \Pr_{\rm in}( y \mid x) 
    -
     \hat{\Pr}_{\rm in}( y \mid x) \right|
     \right]\\
     &\leq
    2 \cdot \Ex_{x \sim \Pr^*}\left[
    \sum_{y  \in [L]}
    \left| \Pr_{\rm in}( y \mid x) 
    -
     \hat{\Pr}_{\rm in}( y \mid x) \right|
     \right],
\end{align*}
where the third step uses the definition of $\hat{h}$ and the fact that $\omega(x) > 0$; the last step uses the fact that $\omega(x) \leq 1$.

We bound the second term now. For this, we first define:
\begin{align*}
 L_{\rm rej}(r) 
&= 
\Ex_{x \sim \Pr^*}\bigg[
\left(
    (1 - \costin - \costout) \cdot \gamma_{\rm in}(x) \cdot ( 1 - \max_{y \in [L]} \Pr_{\rm in}(y \mid x) ) + \costout \cdot (1 - \gamma_{\rm in}(x)) \right) \cdot \1(r(x) = 0)\\
& \hspace{10cm}
+ \costin \cdot \gamma_{\rm in}(x) \cdot \1(r(x) = 1) \bigg].
\end{align*}
and
\begin{align*}
{ \hat{L}_{\rm rej}(r)  }
&= 
\Ex_{x \sim \Pr^*}\bigg[
\left(
    (1 - \costin - \costout) \cdot \hat{\gamma}_{\rm in}(x) \cdot ( 1 - \max_{y \in [L]} \hat{\Pr}_{\rm in}(y \mid x) ) + \costout \cdot (1 - \hat{\gamma}_{\rm in}(x)) \right) \cdot \1(r(x) = 0)\\
& \hspace{10cm}
+ \costin \cdot \hat{\gamma}_{\rm in}(x) \cdot \1(r(x) = 1) \bigg],
\end{align*}
where we denote $\hat{\gamma}_{\rm in}(x) = \frac{\hat{s}_{\rm ood}(x)}{1 + \hat{s}_{\rm ood}(x)}$.

Notice that $r^*$ minimizes $L(r)$ over all rejectors $r: \mathcal{X} \rightarrow \{0, 1\}$. Similarly, note that $\hat{r}_{\rm BB}$ minimizes $\hat{L}(r)$ over all rejectors $r: \mathcal{X} \rightarrow \{0, 1\}$.

Then the second term can be written as:
\begin{align*}
\text{term}_2 
    &=  
    L_{\rm rej}(\hat{r}_{\rm BB}) - L_{\rm rej}(r^*)\\
    &=  L_{\rm rej}(\hat{r}_{\rm BB}) - 
    \hat{L}_{\rm rej}(r^*)
    +
    \hat{L}_{\rm rej}(r^*)
    -
    L_{\rm rej}(r^*) 
    \\
    &\leq
    L_{\rm rej}(\hat{r}_{\rm BB}) - 
        \hat{L}_{\rm rej}(\hat{r}_{\rm BB})
        +
        \hat{L}_{\rm rej}(r^*)
        -
        L_{\rm rej}(r^*)\\
    &\leq 
    2 \cdot(1 - \costin - \costout) \cdot  \left|\max_{y \in [L]} \Pr_{\rm in}(y \mid x) - \max_{y \in [L]} \hat{\Pr}_{\rm in}(y \mid x) \right|\cdot|\gamma_{\rm in}(x) - \hat{\gamma}_{\rm in}(x)|\\
    &
    \hspace{6cm}+
        2 \cdot \big(
    (1 - \costin - \costout)   + \costout + \costin \big) \cdot |\gamma_{\rm in}(x) - \hat{\gamma}_{\rm in}(x)|
    \\
    &\leq 2 \cdot(1 - \costin - \costout) \cdot  (1) \cdot|\gamma_{\rm in}(x) - \hat{\gamma}_{\rm in}(x)| + 
    2 \cdot  (1) \cdot|\gamma_{\rm in}(x) - \hat{\gamma}_{\rm in}(x)|
    \\
    &\leq 4 \cdot |\gamma_{\rm in}(x) - \hat{\gamma}_{\rm in}(x)|\\
    &= 4 \cdot \left|
    \frac{ \Pr_{\rm in}(x) }{ \Pr_{\rm in}(x) + \Pr_{\rm out}(x) } - \frac{\hat{s}_{\rm ood}(x)}{1 + \hat{s}_{\rm ood}(x)}
    \right|,
\end{align*}
where the third step follows from $\hat{r}_{\rm BB}$ being a minimizer of $\hat{L}_{\rm rej}(r)$, the fourth step uses the fact that $\left|\max_{y \in [L]} \Pr_{\rm in}(y \mid x) - \max_{y \in [L]} \hat{\Pr}_{\rm in}(y \mid x) \right| \leq 1$, and the fifth step uses the fact that $c_{\rm in} + c_{\rm out} \leq 1$.

Combining the bounds on $\text{term}_1$ and $\text{term}_2$ completes the proof.
\end{proof}

\begin{proof}[Proof of Lemma 4.2]
We first note that $f^*(x) \propto \log(\mathbb{P}_{\rm in}(y \mid x))$ and $s^*(x) = \log\big( \frac{ \mathbb{P}^*(z = 1 \mid x) }{ \mathbb{P}^*(z = 0 \mid x) } \big)$.

\textbf{Regret Bound 1}: We start with the first regret bound. We expand the multi-class cross-entropy loss to get:
\begin{align*}
    \mathbb{E}_{( x, y ) \sim \mathbb{P}_{\rm in}}\left[ \ell_{\rm mc}( y, f( x ) )
    \right]  &= 
    \mathbb{E}_{x \sim \mathbb{P}_{\rm in}}\left[ -\sum_{y \in [L]} \mathbb{P}_{\rm in}(y \mid x) \cdot \log\left( p_y( x ) \right)
    \right] \\
    \mathbb{E}_{( x, y ) \sim \mathbb{P}_{\rm in}}\left[ \ell_{\rm mc}( y, f^*( x ) )
    \right]  &= 
    \mathbb{E}_{x \sim \mathbb{P}_{\rm in}}\left[ -\sum_{y \in [L]} \mathbb{P}_{\rm in}(y \mid x) \cdot \log\left( \mathbb{P}_{\rm in}(y \mid x) \right)
    \right].
\end{align*}
The right-hand side of the first bound can then be expanded as:
\begin{align}
    \mathbb{E}_{( x, y ) \sim \mathbb{P}_{\rm in}}\left[ \ell_{\rm mc}( y, f( x ) )
    \right] - 
    \mathbb{E}_{( x, y ) \sim \mathbb{P}_{\rm in}}\left[ \ell_{\rm mc}( y, f^*( x ) )
    \right]
    &= 
    \mathbb{E}_{x \sim \mathbb{P}_{\rm in}}\left[ \sum_{y \in [L]} \mathbb{P}_{\rm in}(y \mid x) \cdot \log\left( \frac{ \mathbb{P}_{\rm in}(y \mid x) }{ p_y( x ) } \right)
    \right],
    \label{eqn:kl-rewritten}
\end{align}
which the KL-divergence between $\mathbb{P}_{\rm in}(y \mid x) $ and $p_y(x)$. 

The KL-divergence between two probability mass functions $p$ and $q$ over $\mathcal{U}$ can be lower bounded by:
\begin{equation}
\text{KL}(p || q) \geq \frac{1}{2} \left( \sum_{u \in \mathcal{U}} |p(u) - q(u)| \right)^2.
\label{eqn:kld-bound}
\end{equation}
Applying \eqref{eqn:kld-bound} to \eqref{eqn:kl-rewritten}, we have:
\begin{align*}
    \sum_{y \in [L]} \mathbb{P}_{\rm in}(y \mid x) \cdot \log\left( \frac{ \mathbb{P}_{\rm in}(y \mid x) }{ p_y( x ) } \right)
    &\geq
    \frac{1}{2}\left(
    \sum_{y \in [L]} \left|\mathbb{P}_{\rm in}(y \mid x) -  p_y( x ) \right|
     \right)^2,
\end{align*}
and therefore:
\begin{align*}
    \mathbb{E}_{( x, y ) \sim \mathbb{P}_{\rm in}}\left[ \ell_{\rm mc}( y, f( x ) )
    \right] - 
    \mathbb{E}_{( x, y ) \sim \mathbb{P}_{\rm in}}\left[ \ell_{\rm mc}( y, f^*( x ) )
    \right]
    &\geq 
    \frac{1}{2}\cdot\mathbb{E}_{x \sim \mathbb{P}_{\rm in}}\left[ \left(
    \sum_{y \in [L]} \left|\mathbb{P}_{\rm in}(y \mid x) -  p_y( x ) \right|
     \right)^2
    \right] \\
    &\geq
    \frac{1}{2}
    \left(
    \mathbb{E}_{x \sim \mathbb{P}_{\rm in}}\left[ 
    \sum_{y \in [L]} \left|\mathbb{P}_{\rm in}(y \mid x) -  p_y( x ) \right|
     \right]
     \right)^2,
\end{align*}
or 
\[
    \mathbb{E}_{x \sim \mathbb{P}_{\rm in}}\left[
    \sum_{y \in [L]} \big|
    \Pr_{\rm in}(y \mid x) - p_y(x) \big| \right]
    ~\textstyle
    \leq 
    \sqrt{2}\sqrt{
    \mathbb{E}_{( x, y ) \sim \mathbb{P}_{\rm in}}\left[ \ell_{\rm mc}( y, f( x ) )
    \right] 
    \,-\,
    \mathbb{E}_{( x, y ) \sim \mathbb{P}_{\rm in}}\left[ \ell_{\rm mc}( y, f^*( x ) )
    \right]
    }.
\]

\textbf{Regret Bound 2}: We expand the binary  sigmoid cross-entropy loss to get:
\begin{align*}
\mathbb{E}_{( x, z ) \sim \mathbb{P}^*}\left[ \ell_{\rm bc}(z , s( x ) )\right]
&=
\mathbb{E}_{x \sim \mathbb{P}^*}\left[ -\mathbb{P}^*(z = 1 \mid x) \cdot \log\left( p_\perp( x ) \right)
\,-\,
\mathbb{P}^*(z = -1 \mid x) \cdot \log\left( 1 - p_\perp( x ) \right)
    \right]\\
\mathbb{E}_{( x, z ) \sim \mathbb{P}^*}\left[ \ell_{\rm bc}(z , s^*( x ) )\right]
&=
\mathbb{E}_{x \sim \mathbb{P}^*}\left[ -\mathbb{P}^*(z = 1 \mid x) \cdot \log\left( \mathbb{P}^*(z = 1 \mid x) \right)
\,-\,
\mathbb{P}^*(z = -1 \mid x) \cdot \log\left( \mathbb{P}^*(z = -1 \mid x) \right)
    \right],
\end{align*}
and furthermore
\begin{align*}
\lefteqn{
    \mathbb{E}_{( x, z ) \sim \mathbb{P}^*}\left[ \ell_{\rm bc}( z, s( x ) )
    \right] 
    \,-\,
    \mathbb{E}_{( x, z ) \sim \mathbb{P}^*}\left[ \ell_{\rm bc}(z , s^*( x ) )\right]}\\
&=
\mathbb{E}_{x \sim \mathbb{P}^*}\left[
\mathbb{P}^*(z = 1 \mid x) \cdot \log\left( \frac{ \mathbb{P}^*(z = 1 \mid x) }{ p_\perp(x) } \right) \,+\,
\mathbb{P}^*(z = -1 \mid x) \cdot \log\left( \frac{ \mathbb{P}^*(z = -1 \mid x) }{ 1 - p_\perp(x) } \right)
\right]
\\
&
\geq 
\mathbb{E}_{x \sim \mathbb{P}^*}\left[
\frac{1}{2}\left( |\mathbb{P}^*(z = 1 \mid x) - p_\perp(x)| +  |\mathbb{P}^*(z = -1 \mid x) - (1 - p_\perp(x))|\right)^2
\right]
\\
&
= 
\mathbb{E}_{x \sim \mathbb{P}^*}\left[
\frac{1}{2}\left( |\mathbb{P}^*(z = 1 \mid x) - p_\perp(x)| +  |(1 - \mathbb{P}^*(z = 1 \mid x)) - (1 - p_\perp(x))|\right)^2
\right]
\\
&=
2\cdot 
\mathbb{E}_{x \sim \mathbb{P}^*}\left[
|\mathbb{P}^*(z = 1 \mid x) - p_\perp(x)|^2
\right]\\
&\geq
2\cdot \left(\mathbb{E}_{x \sim \mathbb{P}^*}\left[
|\mathbb{P}^*(z = 1 \mid x) - p_\perp(x)|
\right]\right)^2,
\end{align*}
where the second step uses the bound in \eqref{eqn:kld-bound} and the last step uses Jensen's inequality. Taking square-root on both sides and noting that $\mathbb{P}^*(z = 1 \mid x) = \frac{\Pr_{\rm in}( x )}{\Pr_{\rm in}( x ) + \Pr_{\rm out}( x )}$ completes the proof.
\end{proof}


\section{Technical details: Coupled loss}
\label{app:coupled-loss}
Our second loss function 
seeks to learn an augmented scorer $\bar{f} \colon \XCal \to \Real^{L+1}$, with the additional score corresponding to a ``reject class'', denoted by $\perp$, and is based on the following simple observation:
define
$$ z_{y'}( x ) = \begin{cases}
(1 - \costin - \costout) \cdot \Pr_{\rm in}( y \mid x ) & \text{ if } y' \in [ L ] \\
(1 - 2 \cdot \costin - \costout) + \costout \cdot \frac{\Pr_{\rm out}( x )}{\Pr_{\rm in}( x )} & \text{ if } y' = \perp,
\end{cases} $$
and let
$\zeta_{y'}( x ) = \frac{ z_{y'}( x ) }{ Z( x ) }$ 
for 
$Z( x ) \defEq {\sum_{y'' \in [ L ] \cup \{ \perp \}} z_{y''}( x ) }$.
Now suppose that one has an estimate
$\hat{\zeta}$ of $\zeta$.
This yields an alternate plug-in estimator of the Bayes-optimal SCOD rule~\eqref{eqn:scod-bayes}:
\begin{equation}
    \label{eqn:reject-coupled}
    \hat{r}( x ) = 1 \iff \max_{y' \in [L]} \hat{\zeta}_{y'}( x ) < \hat{\zeta}_{\perp}( x ).
\end{equation}
One may readily estimate 
$\zeta_{y'}$ 
with a standard multi-class loss $\ell_{\rm mc}$,
with suitable modification:
\begin{equation}
    \label{eqn:css-surrogate-repeat}
    \E{(x,y) \sim \Pr_{\rm in}}{ \ell_{\rm mc}( y, \bar{f}( x ) ) } + (1 - \costin) \cdot \E{x \sim \Pr_{\rm in}}{ \ell_{\rm mc}( \perp, \bar{f}( x ) ) } + \costout \cdot \E{x \sim \Pr_{\rm out}}{ \ell_{\rm mc}( \perp, \bar{f}( x ) ) }.    
\end{equation}
Compared to the decoupled loss~\eqref{eqn:decoupled-surrogate}, the key difference is that the penalties on the rejection logit $\bar{f}_\perp( x )$ involve the classification logits as well.

\section{Technical details: Estimating the OOD mixing weight $\pi_{\rm mix}$}
\label{app:noise_correction}

To obtain the latter, 
we apply a simple transformation as follows.

\begin{lemma} 
Suppose 
$\Pr_{\rm mix} = \pi_{\rm mix} \cdot \Pr_{\rm in} + (1-\pi_{\rm mix}) \cdot \Pr_{\rm out}$ with $\pi_{\rm mix} < 1$.
Then, 
if $\Pr_{\rm in}( x ) > 0$,
$$
    \frac{ \Pr_{\rm out}(x) }{ \Pr_{\rm in}(x) } =  \frac{1}{1-\pi_{\rm mix}} \cdot \left( \frac{ \Pr_{\rm mix}(x) }{ \Pr_{\rm in}(x) } - \pi_{\rm mix} \right).
$$
\label{lem:wild-noise-correction}
\vspace{-0.5cm}
\end{lemma}

The above transformation requires knowing the mixing proportion $\pi_{\rm mix}$ of inlier samples in the unlabeled dataset.
However, as it measures the fraction of OOD samples during deployment,
$\pi_{\rm mix}$ is typically \emph{unknown}.
We may however estimate this with {\bf (A2)}.
Observe that for a strictly inlier example $x \in S^*_{\rm in}$, 
we have $\frac{ \Pr_{\rm mix}(x) }{ \Pr_{\rm in}(x)} = \pi_{\rm mix}$, i.e., $\exp( -\hat{s}(x) ) \approx \pi_{\rm mix}$. 
Therefore, we can estimate 
\begin{align*}
    \hat{s}_{\rm ood}(x) = \left(\frac{1}{1-\hat{\pi}_{\rm mix}} \cdot \left( \exp( -\hat{s}(x) ) - \hat{\pi}_{\rm mix} \right)\right)^{-1}
    \quad
    \text{where}
    \quad
    \hat{\pi}_{\rm mix} = \frac{1}{|S^*_{\rm in}|}\sum_{x \in S^*_{\rm in}} \exp(-\hat{s}(x)).
\end{align*}

We remark here that this problem is roughly akin to class prior estimation for PU learning~\citep{garg2021mixture}, 
and noise rate estimation for label noise~\citep{Patrini:2017}.
As in those literatures, 
estimating $\pi_{\rm mix}$ without any assumptions is challenging.
Our assumption on the existence of a Strict Inlier set $S^*_{\rm in}$ is analogous to assuming the existence of a golden label set in the label noise literature~\citep{Hendrycks:2018}.

\begin{proof}[Proof of Lemma \ref{lem:wild-noise-correction}]
Expanding the right-hand side, we have:
\begin{align*}
    \frac{1}{1-\pi_{\rm mix}} \cdot \left( \frac{ \Pr_{\rm mix}(x) }{ \Pr_{\rm in}(x) } - \pi_{\rm mix} \right)
        &= \frac{1}{1-\pi_{\rm mix}} \cdot \left( \frac{ \pi_{\rm mix} \cdot \Pr_{\rm in}(x) + (1-\pi_{\rm mix}) \cdot \Pr_{\rm out}(x) }{ \Pr_{\rm in}(x) } - \pi_{\rm mix} \right)\\
    &= \frac{ \Pr_{\rm out}(x) }{ \Pr_{\rm in}(x) },
\end{align*}
as desired.
\end{proof}

\section{Technical details: Plug-in estimators with an abstention budget}
\label{app:plug-in-budget}
\label{app:katz-samuels}

Observe that~\eqref{eqn:scod} is equivalent to solving the Lagrangian:
\begin{align}
 \label{eqn:budget-constrainted-ood}
 \min_{h, r} \max_{\lambda} \left[ F( h, r; \lambda ) \right]& \\
 \nonumber
 F( h, r; \lambda ) \defEq ( 1 - \costFN ) \cdot \Pr_{\rm in}( y \neq {h}( x ), r( x ) = 0 ) & + 
 \costin(\lambda) \cdot \Pr_{\rm out}( {r}( x ) = 0 ) + 
 \costout(\lambda) \cdot \Pr_{\rm in}( r( x ) = 1 ) + 
 \nu_\lambda
 \\
 \nonumber
 \left( \costin(\lambda), \costout(\lambda), \nu_\lambda \right) \defEq ( \costFN - \lambda \cdot (1 - \pi^*_{\rm in}), 
 &
 \lambda \cdot \pi^*_{\rm in}, \lambda \cdot (1 - \pi^*_{\rm in}) - \lambda \cdot b_{\rm rej} ).
\end{align}
Solving~\eqref{eqn:budget-constrainted-ood} requires optimising over both $(h, r)$ and $\lambda$.
Suppose momentarily that $\lambda$ is fixed.
Then, $F( h, r; \lambda )$ is exactly a scaled version of 
the soft-penalty objective~\eqref{eqn:scod-soft}.
Thus,
we can use Algorithm~\ref{algo:plug-in} to construct a plug-in classifier that minimizes the above joint risk. 
To find the optimal $\lambda$, 
we only need to implement the surrogate minimisation step in Algorithm~\ref{algo:plug-in} \emph{once} to estimate the relevant probabilities.
We can then construct multiple 
plug-in classifiers for different values of $\lambda$, 
and perform an inexpensive threshold search:
amongst the classifiers satisfying the budget constraint, 
we pick
the one that minimises~\eqref{eqn:budget-constrainted-ood}.

The above requires estimating $\pi^*_{\rm in}$, the fraction of inliers observed during deployment.
Following {\bf(A2)}, one plausible estimate is $\pi_{\rm mix}$, the fraction of inliers in the ``wild'' mixture set $S_{\rm mix}$.

\textbf{Remark.} The previous work of \citet{Katz:2022} for OOD detection also seeks to solve an optimization problem with explicit constraints on abstention rates. 
However, there are some subtle, but important, technical differences between their formulation and ours.

Like us,~\citet{Katz:2022} also seek to jointly learn a classifier and an OOD scorer, with  constraints  on the classification and abstention rates, given access to samples from $\Pr_{\rm in}$ and $\Pr_{\rm mix}$. 
For a joint classifier $h: \XCal \rightarrow [L]$ and rejector $r: \XCal \rightarrow \{0, 1\}$, their formulation can be written as:
\begin{align}
   \lefteqn{
   \min_{h}~
   \Pr_{\rm out}\left( r(x) = 0 \right) }
   \label{eq:ks-original}
   \\
    \text{s.t.}\hspace{20pt}
    & \Pin\left( r(x) = 1 \right) \leq \kappa
    \nonumber
    \\
    & \Pin\left( {h}(x) \ne y,\, r(x) = 0 \right)  \leq \tau,
    \nonumber
\end{align}
for given targets $\kappa, \tau \in (0,1)$. 

While $\Pr_{\rm out}$ is not directly available, 
\citeauthor{Katz:2022} provide a simple solution to solving  \eqref{eq:ks-original} using only access to $\Pr_{\rm mix}$ and $\Pr_{\rm in}$. They show that under some mild assumptions,  replacing  $\Pr_{\rm out}$ with $\Pr_{\rm mix}$ in the above problem does not alter the optimal solution. The intuition behind this is that when the first constraint on the inlier abstention rate is satisfied with equality, we have $\Pr_{\rm mix}\left( r(x) = 0 \right) = \pi_{\rm mix} \cdot (1 - \costin) + (1 - \pi_{\rm mix}) \cdot \Pr_{\rm out}\left( r(x) = 0 \right)$, and minimizing this objective is equivalent to minimizing the OOD objective in \eqref{eq:ks-original}. 

This simple trick of replacing $\Pr_{\rm out}$ with $\Pr_{\rm mix}$ will only work when we have an explicit constraint on the inlier abstention rate, and will not work for the formulation we are interested in \eqref{eqn:budget-constrainted-ood}. This is because in our formulation, we impose a budget on the overall abstention rate (as this is a more intuitive quantity that a practitioner may want to constraint), and do not explicitly control the abstention rate on $\Pr_{\rm in}$. 

In comparison to \citet{Katz:2022}, the plug-in based approach we prescribe is more general, and can be applied to optimize any objective that involves as a weighted combination of the mis-classification error and the abstention rates on the inlier and OOD samples. This includes both the  budget-constrained problem we consider in \eqref{eqn:budget-constrainted-ood}, and the constrained problem of \citeauthor{Katz:2022} in \eqref{eq:ks-original}.

\section{Technical details: Relation of proposed losses to existing losses}
\label{app:relation-existing-losses}
Equation~\ref{eqn:decoupled-surrogate} generalises 
several existing proposals in the SC and OOD detection literature.
In particular,
it
reduces to the
loss proposed in~\citet{verma2022calibrated},
when $\Pr_{\rm in} = \Pr_{\rm out}$,
i.e., when one only wishes to abstain on low confidence ID samples.
Interestingly, 
this also corresponds to the decoupled loss
for OOD detection
in~\citet{Bitterwolf:2022};
crucially, however,
they 
reject only based on whether $\bar{f}_{\perp}( x ) < 0$,
rather than comparing $\bar{f}_{\perp}( x )$ and $\max_{y' \in [L]} \bar{f}_{y'}( x )$.
The latter is essential to match the Bayes-optimal predictor in~\eqref{eqn:scod-bayes}.
Similarly,
the coupled loss in~\eqref{eqn:css-surrogate} reduces to the 
\emph{cost-sensitive softmax cross-entropy}
in~\citet{Mozannar:2020}
when $\costout = 0$,
and the OOD detection 
loss of~\citet{Thulasidasan:2021} when $\costin = 0, \costout = 1$.

\section{Additional experiments}
\label{app:expts}
We provide details about the hyper-parameters and dataset splits used in the experiments, as well as, additional experimental results and plots that were not included in the main text. The in-training experimental results are \textbf{averaged over 5 random trials}.

%
\subsection{Hyper-parameter choices}
\label{app:hparams}

We provide details of the learning rate (LR) schedule and other hyper-parameters used in our experiments.
\begin{center}
\begin{tabular}{>{\raggedright}p{4cm}lcccc}
\toprule 
Dataset & Model & LR & Schedule & Epochs & Batch size\tabularnewline
\toprule
{CIFAR-40/100} 
 & CIFAR ResNet 56 & 1.0 & anneal & 256 & 1024\tabularnewline
\bottomrule
\end{tabular}
\end{center}

We use SGD with momentum as the optimization
algorithm for all models. For annealing schedule, the specified learning
rate (LR) is the initial rate, which is then decayed by a factor of
ten after each epoch in a specified list. For CIFAR, these epochs
are 15, 96, 192 and 224. 

\subsection{Baseline details}
We provide further details about the baselines we compare with. The following baselines are trained on only the inlier data.
\begin{itemize}[itemsep=0pt,topsep=0pt,leftmargin=16pt]
    \item \textit{MSP or Chow's rule}: Train a scorer $f: \cX \rightarrow \R^L$ using CE loss, and threshold the MSP 
    to decide to abstain \citep{Chow:1970, Hendrycks:2017}. 
    \item \textit{MaxLogit}: Same as above, but instead threshold the maximum logit $\max_{y \in [L]} f_y(x)$  \citep{Hendrickx:2021}.
    \item \textit{Energy score}: Same as above, but threshold the energy function $-\log\sum_y \exp(f_y(x))$
    \citep{Liu:2020}.
    \item \textit{ODIN}: Train a scorer $f: \cX \rightarrow \R^L$ using CE loss, and uses a combination of input noise and temperature-scaled MSP to decide when to abstain \cite{Hendrickx:2021}.
    \item \textit{SIRC}: Train a scorer $f: \cX \rightarrow \R^L$ using CE loss, and compute a post-hoc deferral rule that combines the MSP score with either the $L_1$-norm or the residual score of the embedding layer from the scorer $f$ \citep{Xia:2022}.
     \item \textit{CSS}: Minimize the cost-sensitive softmax L2R loss of \citet{Mozannar:2020} using only the inlier dataset to learn a scorer $f \colon \XCal \to \Real^{L + 1}$, augmented with a rejection score $f_\perp( x )$, and abstain iff $f_{\perp}( x ) > \max_{y' \in [ L ]} f_{y'}( x ) + t$, for threshold $t$.
\end{itemize}
The following baselines additional use the unlabeled data containing a mix of inlier and OOD samples.
\begin{itemize}[itemsep=0pt,topsep=0pt,leftmargin=16pt]
    \item \textit{Coupled CE (CCE)}: Train a scorer $f \colon \XCal \to \Real^{L + 1}$, augmented with a rejection score $f_\perp( x )$ by optimizing the CCE loss of \citet{Thulasidasan:2021}, and abstain iff $f_{\perp}( x ) > \max_{y' \in [ L ]} f_{y'}( x ) + t$, for threshold $t$.
    \item \textit{De-coupled CE (DCE)}: Same as above but uses the DCE loss of \citet{Bitterwolf:2022} for training.
    \item \textit{Outlier Exposure (OE):} Train a scorer using the OE loss of \citet{Hendrycks:2019} and threshold the MSP.
\end{itemize}

\subsection{Data split details}
For the CIFAR-100 experiments 
where we use a wild sample containing a mix of ID and OOD examples, we split the original CIFAR-100 training set into two halves, use one half as the inlier sample and the other half to construct the wild sample. For evaluation, we combine the orignal CIFAR-100 test set with the respective OOD test set. In each case, the larger of the ID and OOD dataset is down-sampled to match the desired ID-OOD ratio. The experimental results are \textbf{averaged over 5 random trials}.

For the pre-trained ImageNet experiments, we sample equal number of examples from the ImageNet validation sample and the OOD dataset, and annotate them with the pre-trained model. The number of samples is set to  the smaller of the size of the OOD dataset or 5000.

%
%

\subsection{Comparison to CSS and ODIN baselines}
We present some representative results in Table \ref{tab:css-odin-comparisons} comparing our proposed methods against the cost-sensitive softmax (CSS) of \citet{Mozannar:2020}, a representative learning-to-reject baseline, and the ODIN method of \citet{Hendrickx:2021}, an OOD detection baseline. As expected, the CSS baseline, which does not have OOD detection capabilities is seen to under-perform. The ODIN, baseline, on the other hand, is occasionally seen to be competitive.

\begin{table}[t]
    \centering
    \caption{AUC-RC ($\downarrow$) for CIFAR-100 as ID, and a ``wild'' comprising of 90\% ID and \emph{only} 10\% OOD. 
    The OOD part of the wild set is drawn from the \emph{same} OOD dataset from which the test set is drawn.
    We compare the proposed methods with the cost-sensitive softmax (CSS) learning-to-reject loss of \citet{Mozannar:2020} and the ODIN method of \citet{Hendrickx:2021}. 
    We set $c_{\rm fn} = 0.75$. 
    }
    \begin{tabular}{@{}lccc@{}}
        \toprule
        & \multicolumn{3}{c}{
        ID + OOD training with
        $\Pr^{\rm tr}_{\rm out} = \Pr^{\rm te}_{\rm out}$} 
        \\[2pt]
        Method / $\Pr^{\rm te}_{\rm out}$ & SVHN & Places & OpenImages 
        \\
        \toprule
        CSS & 0.286 & 0.263 & 0.254 \\[3pt]
        ODIN & 0.218 & 0.217 & \textbf{0.217} \\
        \midrule
        Plug-in BB [$L_1$] & \textbf{0.196} & 0.210 &  {0.222}  \\[2pt]
        Plug-in BB [Res] & 0.198 & 0.236 & 0.251  \\ [2pt]
        Plug-in LB* & 0.221 & \textbf{0.199} &  0.225  \\
        \bottomrule
    \end{tabular}
    \label{tab:css-odin-comparisons}
\end{table}

\subsection{Experimental plots}
\label{app:expt-additional-plots}

We present  experimental plots in Figure \ref{fig:expt-results-extra} of the joint risk in Section \ref{sec:expts} as a function of the fraction of samples abstained. We also plot the inlier accuracy, the OOD precision, and the OOD recall as a function of samples abstained. These metrics are described below:
\begin{align*}
    \text{inlier-accuracy}(h, r) &= \frac{
        \sum_{(x,y) \in S_{\rm in}}\1(y = \bar{h}(x), r(x) = 0)
        }
        {
        \sum_{x \in S_{\rm all}}\1( r(x) = 0 )
        }\\
    \text{ood-precision}(h, r) &= \frac{
        \sum_{(x,y) \in S_{\rm out}}\1( r(x) = 1)
        }
        {
        \sum_{x \in S_{\rm all}}\1( r(x) = 1)
        }\\
    \text{ood-recall}(\bar{h}) &= \frac{
        \sum_{x \in S_{\rm out}}\1( r(x) = 1)
        }
        {
        |S_{\rm out}|
        },
\end{align*}
where $S_{\rm all} = \{x: (x, y) \in S_{\rm in}\} \cup S_{\rm out}$ is the combined set of ID and OOD instances. 

One can see a few general trends.
The joint risk decreases with more abstentions; the inlier accuracy increases with abstentions.  
The OOD precision is the highest initially when the abstentions are on the OOD samples, but decreases when the OOD samples are exhausted, and the abstentions are on the inlier samples; the opposite is true for OOD recall.



%
\begin{figure*}[!ht]
    \centering

    \subfigure[$\Pr_{\rm in}$: CIFAR-100;~~~ $\Pr_{\rm out}$: SVHN]{
    \includegraphics[scale=0.27]{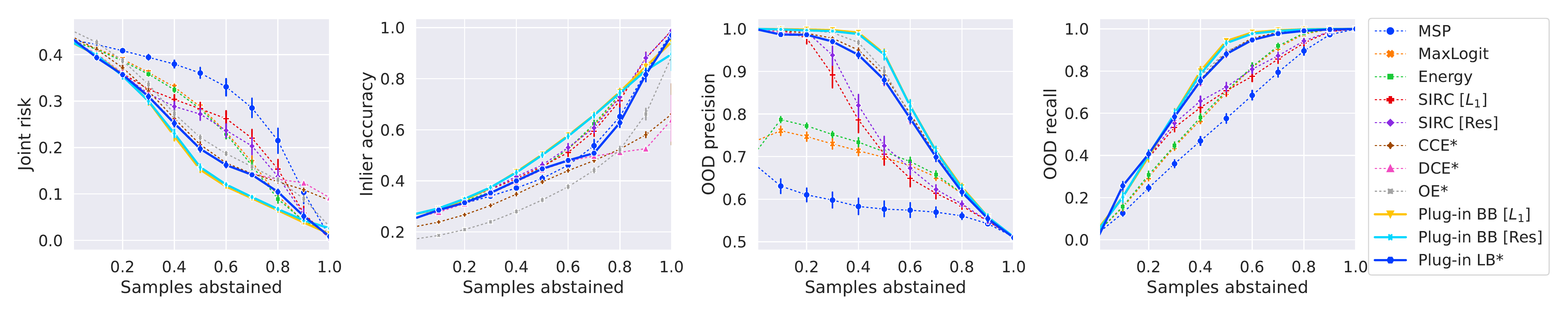}}
    \\[-2pt]
    \subfigure[$\Pr_{\rm in}$: CIFAR-100;~~~ $\Pr_{\rm out}$: Places365]{
    \includegraphics[scale=0.27]{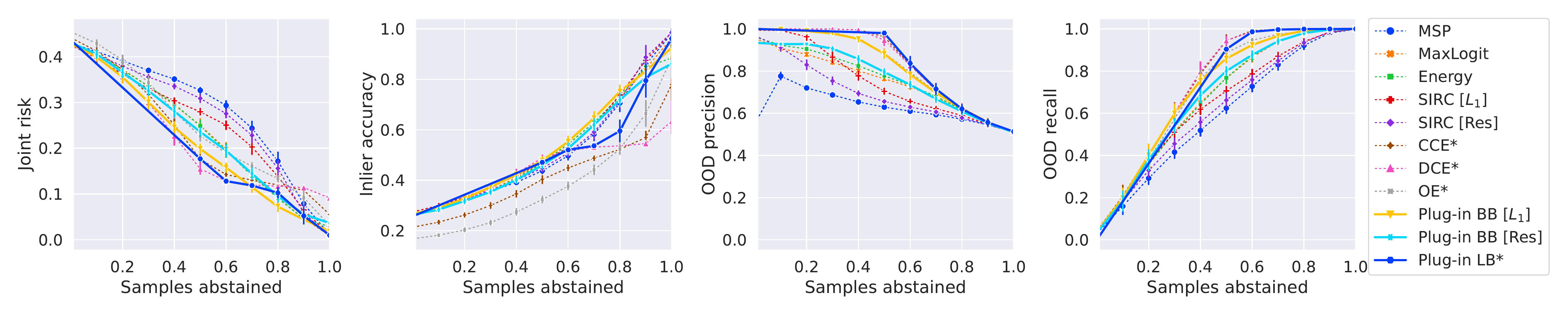}
    }
    \\[-2pt]
    \subfigure[$\Pr_{\rm in}$: CIFAR100;~~~ $\Pr_{\rm out}$: LSUN]{
    \includegraphics[scale=0.27]{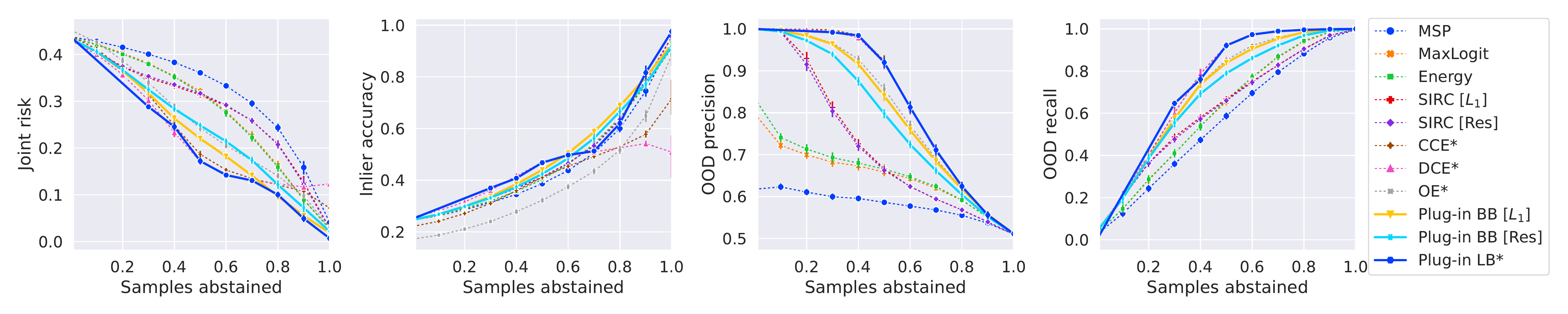}
    }
    \\[-2pt]
    \subfigure[$\Pr_{\rm in}$: CIFAR-100;~~~ $\Pr_{\rm out}$: Texture]{
    \includegraphics[scale=0.27]{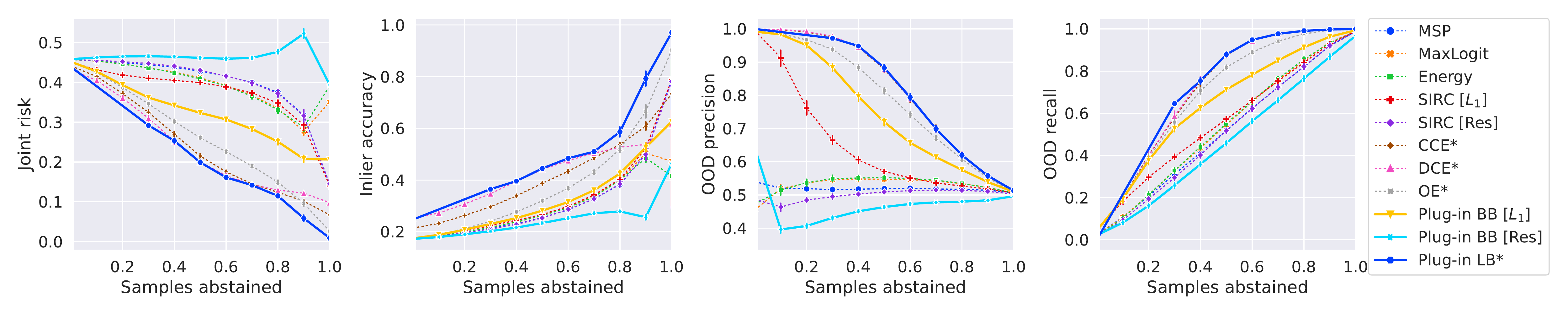}
    }
    \\[-2pt]
    \subfigure[$\Pr_{\rm in}$: CIFAR-100;~~~ $\Pr_{\rm out}$: Open Images]{
    \includegraphics[scale=0.27]{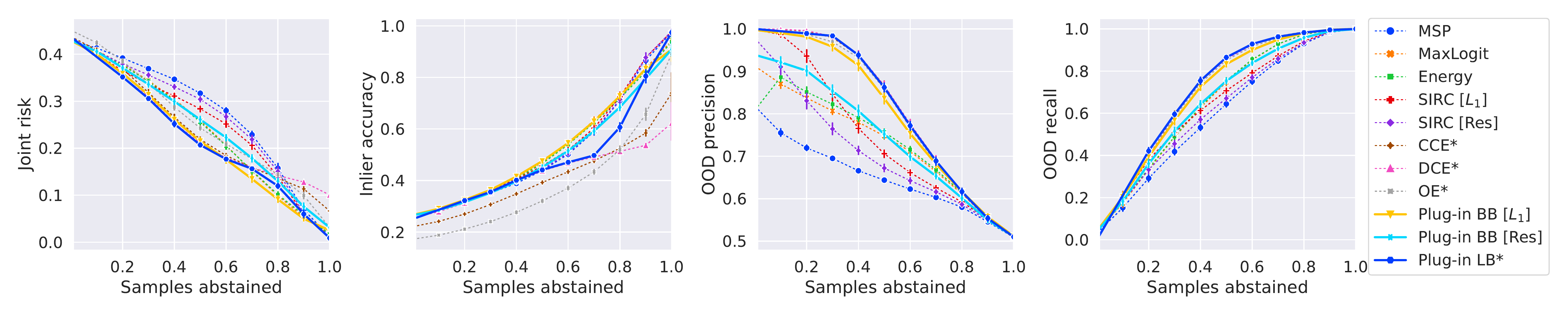}
    }
    \\[-2pt]
    \subfigure[$\Pr_{\rm in}$: CIFAR-100;~~~ $\Pr_{\rm out}$: CelebA]{
    \includegraphics[scale=0.27]{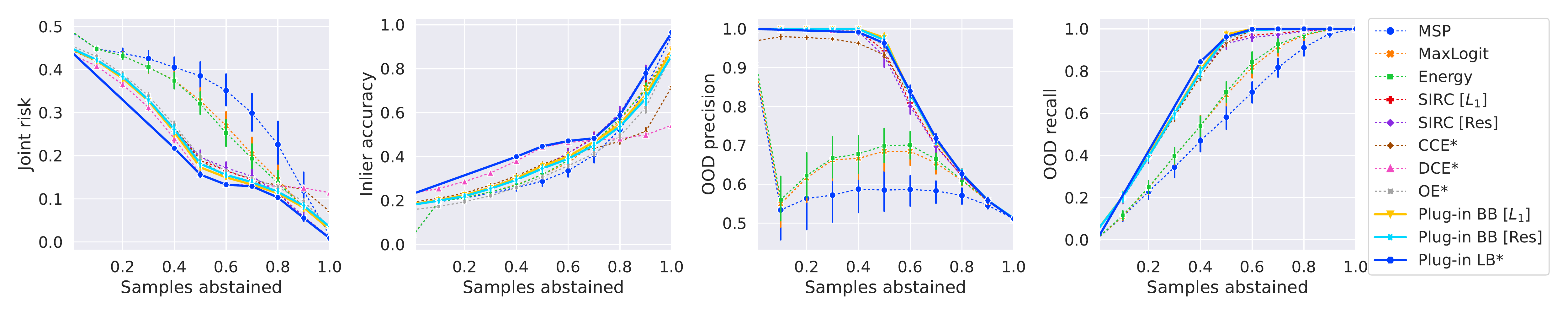}
    }
    \caption{Plots of classification and OOD detection metrics as a function of the fraction of abstained samples (averaged over 5 trials). We use CIFAR-100 as the ID sample, and a mix of CIFAR-100 and each of SVHN, Places265, LSUN, LSUN-R, Texture, Open Images and CelebA as the wild sample, and evaluate on the respective OOD dataset. The wild sample contains 90\% ID and only 10\% OOD samples. The test contains equal proportions of ID and OOD samples. For the joint risk, \emph{lower values are better}. For all other metrics, \textit{higher values are better}. 
    We set $c_{\rm fn} = 0.75$.  }
    \label{fig:expt-results-extra}
\end{figure*}


\subsection{Varying OOD mixing proportion in test set}
We repeat the experiments in Table \ref{tab:auc-rc-cf100-random-100k} on CIFAR-100 and 100K Random Images with varying proportions of OOD samples in the test set, and present the results in Table \ref{tab:auc-rc-cf100-random-100k-varying-ood-mixing-proportion}. One among the proposed plug-in methods continues to perform the best.

\begin{table}[t]
    \centering
    \caption{Area Under the Risk-Coverage Curve (AUC-RC) for  methods trained with CIFAR-100 as the ID sample and a mix of CIFAR-100 and 300K Random Images as the wild sample, and with the proportion of OOD samples in test set varied. 
    The wild set contains 10\% ID and 90\% OOD. 
    Base model is ResNet-56. 
    We set $c_{\rm fn} = 0.75$. 
    A * against a method indicates that it uses both ID and OOD samples for training. 
    \emph{Lower} values are \emph{better}.}
    \scriptsize
    \begin{tabular}{@{}lccccccccccc@{}}
        \toprule
        & \multicolumn{5}{c}{
        Test OOD proportion = 0.25}
        &~~&
        \multicolumn{5}{c}{
        Test OOD proportion = 0.75}
        \\
        Method / $\Pr^{\rm te}_{\rm out}$  & SVHN & Places & LSUN & LSUN-R
        & Texture & ~~ & SVHN & Places & LSUN & LSUN-R & Texture
        \\
        \toprule
        MSP & 0.171 & 0.186 & 0.176 & 0.222 & 0.192 &  & 
0.501 & 0.518 & 0.506 & 0.564 & 0.532 \\[3pt]
MaxLogit & 0.156 & 0.175 & 0.163 & 0.204 & 0.183 &  & 
0.464 & 0.505 & 0.478 & 0.545 & 0.512 \\[3pt]
Energy & 0.158 & 0.177 & 0.162 & 0.206 & 0.181 &  & 
0.467 & 0.502 & 0.477 & 0.538 & 0.509 \\\midrule
SIRC [$L_1$] & 0.158 & 0.181 & 0.159 & 0.218 & 0.180 &  & 
0.480 & 0.513 & 0.485 & 0.560 & 0.509 \\[3pt]
SIRC [Res] & 0.141 & 0.181 & 0.152 & 0.219 & 0.194 &  & 
0.456 & 0.516 & 0.476 & 0.561 & 0.535 \\
\midrule
CCE* & 0.175 & 0.191 & 0.153 & 0.131 & 0.154 &  & 
0.460 & 0.487 & 0.425 & 0.374 & 0.429 \\[3pt]
DCE* & 0.182 & 0.200 & 0.155 & 0.136 & 0.162 &  & 
0.467 & 0.498 & 0.414 & 0.372 & 0.428 \\[3pt]
OE* & 0.179 & 0.174 & 0.147 & 0.117 & 0.148 &  & 
0.492 & 0.487 & 0.440 & 0.371 & 0.440 \\\midrule
Plug-in BB [$L_1$] & 0.127 & \textbf{0.164} & \textbf{0.128} & 0.180 & \textbf{0.134} &  & 
0.395 & \textbf{0.457} & \textbf{0.397} & 0.448 & \textbf{0.414} \\[3pt]
Plug-in BB [Res] & \textbf{0.111} & 0.175 & \textbf{0.129} & 0.182 & 0.248 &  & 
\textbf{0.377} & 0.484 & 0.407 & 0.449 & 0.645 \\[3pt]
Plug-in LB* & 0.160 & 0.169 & 0.133 & \textbf{0.099} & \textbf{0.132} &  & 
0.468 & 0.489 & 0.418 & \textbf{0.351} & 0.430\\
        \bottomrule
    \end{tabular}
    \vspace{-3pt}
    \label{tab:auc-rc-cf100-random-100k-varying-ood-mixing-proportion}
\end{table}

\subsection{Varying OOD cost parameter}
We repeat the experiments in Table \ref{tab:auc-rc-cf100-random-100k} on CIFAR-100 and 100K Random Images with varying values of cost parameter $c_{\rm fn}$, and present the results in Table \ref{tab:auc-rc-cf100-random-100k-varying-ood-cost}. One among the proposed plug-in methods continues to perform the best, although the gap between the best and second-best methods increases with $c_{\rm fn}$.

\begin{table}[t]
    \centering
    \caption{Area Under the Risk-Coverage Curve (AUC-RC) for  methods trained with CIFAR-100 as the ID sample and a mix of CIFAR-100 and 300K Random Images as the wild sample, and for different values of cost parameter $c_{\rm fn}$. 
    The wild set contains 10\% ID and 90\% OOD. 
    Base model is ResNet-56. 
    }
    \scriptsize
    \begin{tabular}{@{}lccccccccccc@{}}
        \toprule
        & \multicolumn{5}{c}{
        $c_{\rm fn} = 0.5$}
        &~~&
        \multicolumn{5}{c}{
        $c_{\rm fn} = 0.9$}
        \\
        Method / $\Pr^{\rm te}_{\rm out}$  & SVHN & Places & LSUN & LSUN-R
        & Texture & ~~ & SVHN & Places & LSUN & LSUN-R & Texture
        \\
        \toprule
        MSP & 0.261 & 0.271 & 0.265 & 0.299 & 0.278 &  & 
0.350 & 0.374 & 0.360 & 0.448 & 0.394 \\[3pt]
MaxLogit & 0.253 & 0.271 & 0.259 & 0.293 & 0.277 &  & 
0.304 & 0.350 & 0.318 & 0.410 & 0.360 \\[3pt]
Energy & 0.254 & 0.273 & 0.262 & 0.293 & 0.277 &  & 
0.303 & 0.349 & 0.317 & 0.407 & 0.359 \\\midrule
SIRC [$L_1$] & 0.252 & 0.270 & 0.257 & 0.298 & 0.267 &  & 
0.319 & 0.368 & 0.327 & 0.440 & 0.358 \\[3pt]
SIRC [Res] & 0.245 & 0.270 & 0.251 & 0.297 & 0.282 &  & 
0.286 & 0.371 & 0.311 & 0.440 & 0.397 \\\midrule
CCE* & 0.296 & 0.307 & 0.283 & 0.269 & 0.286 &  & 
0.282 & 0.318 & 0.233 & 0.179 & 0.240 \\[3pt]
DCE* & 0.303 & 0.317 & 0.285 & 0.270 & 0.292 &  & 
0.289 & 0.331 & 0.225 & 0.177 & 0.238 \\[3pt]
OE* & 0.287 & 0.283 & 0.270 & 0.255 & 0.272 &  & 
0.327 & 0.315 & 0.252 & 0.173 & 0.251 \\\midrule
Plug-in BB [$L_1$] & 0.237 & \textbf{0.258} & \textbf{0.239} & 0.267 & \textbf{0.244} &  & 
0.207 & \textbf{0.280} & \textbf{0.207} & 0.266 & \textbf{0.226} \\[3pt]
Plug-in BB [Res] & \textbf{0.228} & {0.266} & \textbf{0.241} & 0.269 & 0.321 &  & 
\textbf{0.185} & 0.322 & 0.218 & 0.266 & 0.599 \\[3pt]
Plug-in LB* & 0.256 & {0.265} & 0.243 & \textbf{0.222} & \textbf{0.245} &  & 
0.299 & 0.326 & 0.234 & \textbf{0.165} & 0.246 \\\bottomrule
    \end{tabular}
    \vspace{-3pt}
    \label{tab:auc-rc-cf100-random-100k-varying-ood-cost}
\end{table}

\subsection{Confidence intervals}
In Table \ref{tab:auc-rc-cf100-random-100k-std}, we report 95\% confidence intervals for the experiments on CIFAR-100 and 100K Random Images
 from Table \ref{tab:auc-rc-cf100-random-100k}. In each case, the differences between the best performing plug-in method and the baselines are \emph{statistically significant}.
\begin{table}[t]
    \centering
    \caption{Area Under the Risk-Coverage Curve (AUC-RC) for  methods trained with CIFAR-100 as the ID sample and a mix of CIFAR-100 and 300K Random Images as the wild sample, with 95\% \textbf{confidence intervals} included. 
    The wild set contains 10\% ID and 90\% OOD. 
    The test sets contain 50\% ID and 50\% OOD samples. 
    Base model is ResNet-56. 
    We set $c_{\rm fn} = 0.75$. 
    }
    \scriptsize
    \begin{tabular}{@{}lccccc@{}}
        \toprule
        Method / $\Pr^{\rm te}_{\rm out}$  & SVHN & Places & LSUN & LSUN-R
        & Texture
        \\
        \toprule
        MSP & 0.317 $\pm$ 0.023 & 0.336 $\pm$ 0.010 & 0.326 $\pm$ 0.005 & 0.393 $\pm$ 0.018 & 0.350 $\pm$ 0.004 \\[3pt]
MaxLogit & 0.286 $\pm$ 0.012 & 0.321 $\pm$ 0.011 & 0.299 $\pm$ 0.009 & 0.365 $\pm$ 0.016 & 0.329 $\pm$ 0.013 \\[3pt]
Energy & 0.286 $\pm$ 0.012 & 0.320 $\pm$ 0.013 & 0.296 $\pm$ 0.008 & 0.364 $\pm$ 0.015 & 0.326 $\pm$ 0.014 \\\midrule
SIRC [$L_1$] & 0.294 $\pm$ 0.021 & 0.331 $\pm$ 0.010 & 0.300 $\pm$ 0.007 & 0.387 $\pm$ 0.017 & 0.326 $\pm$ 0.006 \\[3pt]
SIRC [Res] & 0.270 $\pm$ 0.019 & 0.332 $\pm$ 0.009 & 0.289 $\pm$ 0.007 & 0.384 $\pm$ 0.019 & 0.353 $\pm$ 0.003 \\[3pt]
CCE* & 0.288 $\pm$ 0.017 & 0.315 $\pm$ 0.018 & 0.252 $\pm$ 0.004 & 0.213 $\pm$ 0.001 & 0.255 $\pm$ 0.004 \\[3pt]
DCE* & 0.295 $\pm$ 0.015 & 0.326 $\pm$ 0.028 & 0.246 $\pm$ 0.004 & 0.212 $\pm$ 0.001 & 0.260 $\pm$ 0.005 \\[3pt]
OE* & 0.313 $\pm$ 0.015 & 0.304 $\pm$ 0.006 & 0.261 $\pm$ 0.001 & 0.204 $\pm$ 0.002 & 0.260 $\pm$ 0.002 \\\midrule
Plug-in BB [$L_1$] & 0.223 $\pm$ 0.004 & \textbf{0.286 $\pm$ 0.013} & \textbf{0.227 $\pm$ 0.007} & 0.294 $\pm$ 0.021 & \textbf{0.240 $\pm$ 0.006} \\[3pt]
Plug-in BB [Res] & \textbf{0.205 $\pm$ 0.002} & 0.309 $\pm$ 0.009 & 0.235 $\pm$ 0.005 & 0.296 $\pm$ 0.012 & 0.457 $\pm$ 0.008 \\[3pt]
Plug-in LB* & 0.290 $\pm$ 0.017 & 0.306 $\pm$ 0.016 & 0.243 $\pm$ 0.003 & \textbf{0.186 $\pm$ 0.001} & 0.248 $\pm$ 0.006 \\
        \bottomrule
    \end{tabular}
    \vspace{-3pt}
    \label{tab:auc-rc-cf100-random-100k-std}
\end{table}

\subsection{AUC and FPR95 metrics for OOD scorers}
\label{app:ranking-metrics}
Table \ref{tab:ood-auc-fpr95} reports the AUC-ROC and FPR@95TPR metrics for the OOD scorers used by different methods, treating OOD samples as positives and ID samples as negatives. Note that the CCE, DCE and OE methods which are trained with both ID and OOD samples are seen to perform the best on these metrics. However, this superior performance in OOD detection doesn't often translate to good performance on the SCOD problem (as measured by AUC-RC). This is because these methods abstain solely based on the their estimates of the ID-OOD density ratio, and do not trade-off between accuracy and OOD detection performance.

\begin{table}[t]
    \centering
    \caption{AUC-ROC (($\uparrow$)) and FPR@95TPR ($\downarrow$) metrics for OOD scorers used by different  methods trained. We use CIFAR-100 as the ID sample and a mix of 50\% CIFAR-100 and 50\% 300K Random Images as the wild sample. 
    Base model is ResNet-56. 
    We set $c_{\rm fn} = 0.75$ in the plug-in methods. The CCE, DCE and OE methods which are trained with both ID and OOD samples are seen to perform the best on these metrics. However, this superior performance in OOD detection doesn't often translate to good performance on the SCOD problem (as measured by AUC-RC in Table \ref{tab:auc-rc-cf100-random-100k}).
    }
    \scriptsize
    \begin{tabular}{@{}lccccccccccc@{}}
        \toprule
        & \multicolumn{5}{c}{
        OOD AUC-ROC}
        &~~&
        \multicolumn{5}{c}{
        OOD FPR95}
        \\
        Method / $\Pr^{\rm te}_{\rm out}$  & SVHN & Places & LSUN & LSUN-R
        & Texture & ~~ & SVHN & Places & LSUN & LSUN-R & Texture
        \\
        \toprule
        MSP & 0.629 & 0.602 & 0.615 & 0.494 & 0.579 &  & 
0.813 & 0.868 & 0.829 & 0.933 & 0.903 \\[3pt]
MaxLogit & 0.682 & 0.649 & 0.692 & 0.564 & 0.634 &  & 
0.688 & 0.846 & 0.754 & 0.916 & 0.864 \\[3pt]
Energy & 0.685 & 0.654 & 0.698 & 0.568 & 0.645 &  & 
0.680 & 0.843 & 0.742 & 0.915 & 0.850 \\\midrule
SIRC [$L_1$] & 0.699 & 0.621 & 0.700 & 0.516 & 0.663 &  & 
0.788 & 0.871 & 0.819 & 0.930 & 0.882 \\[3pt]
SIRC [Res] & 0.777 & 0.613 & 0.735 & 0.513 & 0.566 &  & 
0.755 & 0.870 & 0.800 & 0.929 & 0.900 \\\midrule
CCE* & 0.772 & \textbf{0.725} & 0.878 & 0.995 & 0.883 &  & 
0.647 & 0.775 & 0.520 & 0.022 & 0.570 \\[3pt]
DCE* & 0.770 & 0.709 & \textbf{0.905} & \textbf{0.998} & \textbf{0.888} &  & 
0.693 & 0.807 & \textbf{0.466} & \textbf{0.007} & \textbf{0.562} \\[3pt]
OE* & 0.699 & \textbf{0.725} & 0.861 & \textbf{0.998} & 0.873 &  & 
0.797 & 0.792 & 0.689 & \textbf{0.004} & 0.706 \\\midrule
Plug-in BB [$L_1$] & 0.897 & 0.718 & 0.896 & 0.684 & 0.876 &  & 
0.473 & \textbf{0.716} & 0.496 & 0.717 & 0.580 \\[3pt]
Plug-in BB [Res] & \textbf{0.963} & 0.667 & 0.885 & 0.680 & 0.432 &  & 
\textbf{0.251} & 0.777 & 0.559 & 0.726 & 0.996 \\[3pt]
Plug-in LB* & 0.710 & 0.683 & 0.860 & \textbf{0.997} & 0.853 &  & 
0.749 & 0.801 & 0.653 & 0.009 & 0.697 \\\bottomrule
    \end{tabular}
    \vspace{-3pt}
    \label{tab:ood-auc-fpr95}
\end{table}

\subsection{Results on CIFAR-40 ID sample}
Following  \citet{Kim:2021}, we present in Table \ref{tab:auc-rc-cf40-id-only} results of experiments where we use CIFAR-40 (a subset of CIFAR-100 with 40 classes) as the ID-only training dataset, and we evaluate on CIFAR-60 (the remainder with 60 classes), SVHN, Places, LSUN and LSUN-R as OOD datasets.

%
\begin{table}[!t]
    \centering
    \caption{Area Under the Risk-Coverage Curve (AUC-RC) for different methods with CIFAR-40 as the inlier dataset and the training set comprising of only inlier samples, when evaluated on the following OOD datasets: CIFAR60, SVHN, Places, LSUN-C and LSUN-R. The test sets contain 50\% ID samples and 50\% OOD samples. We set $c_{\rm fn} = 0.75$. The last three rows contain results for the proposed methods.}
    \begin{tabular}{@{}lccccc@{}}
        \toprule
        & \multicolumn{5}{c}{Test OOD dataset}\\
        Method & CIFAR60 & SVHN & Places & LSUN-C & LSUN-R \\
        \toprule
        MSP & 0.262 & 0.238 & 0.252 & 0.282 & 0.243 \\[3pt]
        MaxLogit & 0.272 & 0.223 & \textbf{0.242} & 0.252 & 0.231 \\[3pt]
        Energy & 0.266 & 0.221 & 0.244 & 0.248 & \textbf{0.230} \\
        \midrule
        SIRC [$\|z\|_1$] & 0.263 & 0.226 & 0.249 & 0.266 & 0.241 \\[2pt]
        SIRC [Res]  & \textbf{0.258} & 0.209 & 0.250 & 0.244 & 0.241 \\[3pt]
        \midrule
        SIRC [$\|z\|_1$, Bayes-opt] &  0.290 & {0.195} & \textbf{0.243} & \textbf{0.191} & \textbf{0.228} \\[1pt]
        SIRC [Res, Bayes-opt] & 0.309 & \textbf{0.175} & 0.279 & 0.204 & 0.247 \\
        \bottomrule
    \end{tabular}
    \label{tab:auc-rc-cf40-id-only}
\end{table}

\subsection{Additional results on pre-trained ImageNet models}
Following \citet{Xia:2022}, we present additional results with pre-trained models with ImageNet-200 (a subset of ImageNet with 200 classes) as the inlier dataset in Table \ref{tab:auc-rc-imagenet-id-only-res50}. The base model is a ResNet-50.

\begin{table}[t]
    \centering
    \scriptsize
    \caption{AUC-RC ($\downarrow$)  for methods trained with ImageNet-200 as the inlier dataset and \emph{without} OOD samples. The base model is a pre-trained ResNet-50 model. 
    \emph{Lower} values are \emph{better}.
    }
    \vspace{-2pt}
    \begin{tabular}{@{}lccccccccc@{}}
        \toprule
        & \multicolumn{8}{c}{ID-only training}\\
        Method / $\Pr^{\rm te}_{\rm out}$ & Places & LSUN & CelebA & Colorectal &
        iNaturalist-O & Texture &
        ImageNet-O & Food32\\
        \toprule
        MSP & 0.183 & 0.186 & 0.156 & 0.163 & 0.161 & 0.172 & 0.217 & 0.181  \\[3pt]
MaxLogit & \textbf{0.173} & 0.184 & 0.146 & 0.149 & 0.166 & \textbf{0.162} & \textbf{0.209} & 0.218  \\[3pt]
Energy & 0.176 & 0.185 & 0.145 & 0.146 & 0.172 & 0.166 & 0.211 & 0.225  \\
\midrule
SIRC [$L_1$] & 0.185 & 0.195 & 0.155 & 0.165 & 0.166 & 0.172 & 0.214 & 0.184  \\[3pt]
SIRC [Res] & {0.180} & 0.179 & 0.137 & 0.140 & 0.151 & 0.167 & 0.219 & 0.174 \\
\midrule
Plug-in BB [$L_1$] & 0.262 & 0.261 & 0.199 & 0.225 & 0.228 & 0.270 & 0.298 & 0.240 \\[3pt]
Plug-in BB [Res] & 0.184 & \textbf{0.172} & \textbf{0.135} & \textbf{0.138} & \textbf{0.145} & 0.194 & 0.285 & \textbf{0.164}  \\
        \bottomrule
    \end{tabular}
    \vspace{3pt}
    \begin{tabular}{@{}lcccc@{}}
        \toprule
        & \multicolumn{4}{c}{ID-only training}\\
        Method / $\Pr^{\rm te}_{\rm out}$ & Near-ImageNet-200
        & Caltech65 & Places32 & Noise\\
        \toprule
        MSP &  0.209 & 0.184 & 0.176 & 0.188 \\[3pt]
MaxLogit &  0.220 & \textbf{0.171} & \textbf{0.170} & 0.192 \\[3pt]
Energy &  0.217 & 0.175 & \textbf{0.169} & 0.190 \\
\midrule
SIRC [$L_1$] & \textbf{0.205} & 0.182 & 0.174 & 0.191 \\[3pt]
SIRC [Res] &  \textbf{0.204} & 0.177 & 0.173 & \textbf{0.136} \\
\midrule
Plug-in BB [$L_1$] &  0.264 & 0.242 & 0.256 & 0.344 \\[3pt]
Plug-in BB [Res] & 0.247 & 0.202 & \textbf{0.171} & \textbf{0.136} \\
        \bottomrule
    \end{tabular}
    \label{tab:auc-rc-imagenet-id-only-res50}
\end{table}

\section{Illustrating the failure of MSP for OOD detection}
\label{app:failure-msp}

\subsection{Illustration of MSP failure for open-set classification}

Figure~\ref{fig:chow-fail-example} shows a graphical illustration of the example discussed in Example~\ref{ex:msp-failure},
wherein the MSP baseline can fail for open-set classification.

\begin{figure*}[!t]
    \centering
    
    \resizebox{\linewidth}{!}{
        \subfigure[]{
        \includegraphics[scale=0.25]{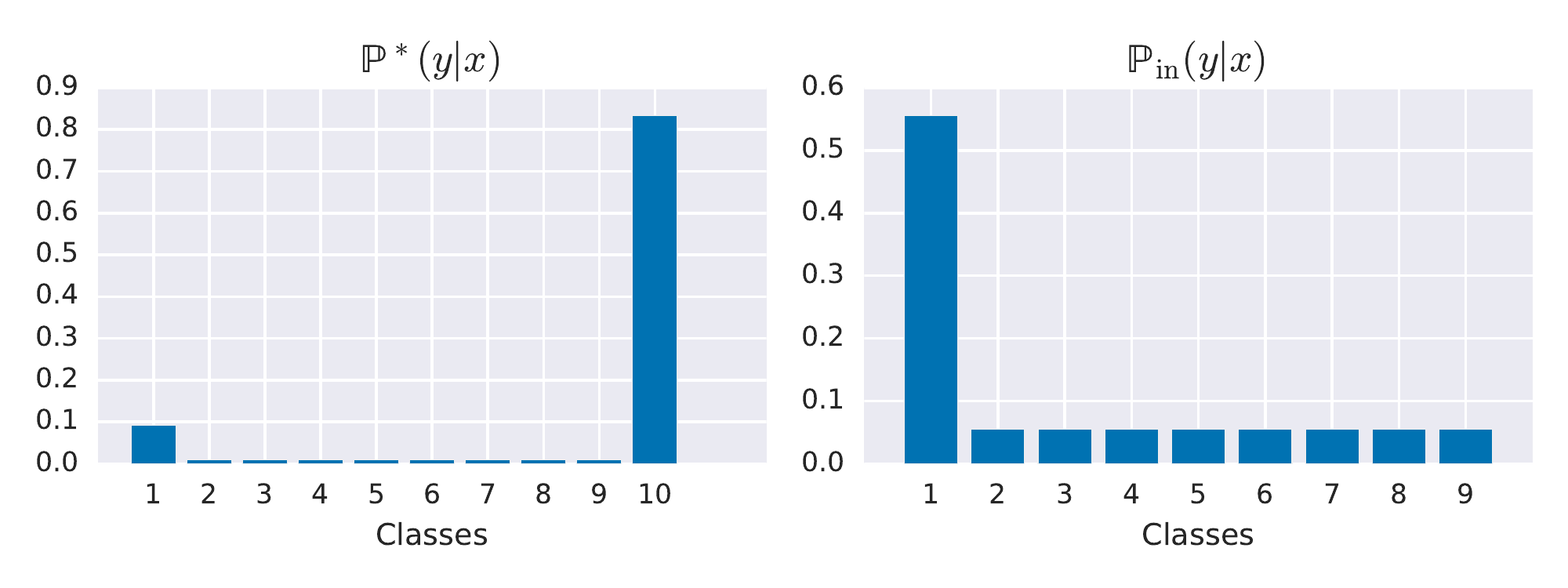}
        }
        \subfigure[]{
        \includegraphics[scale=0.25]{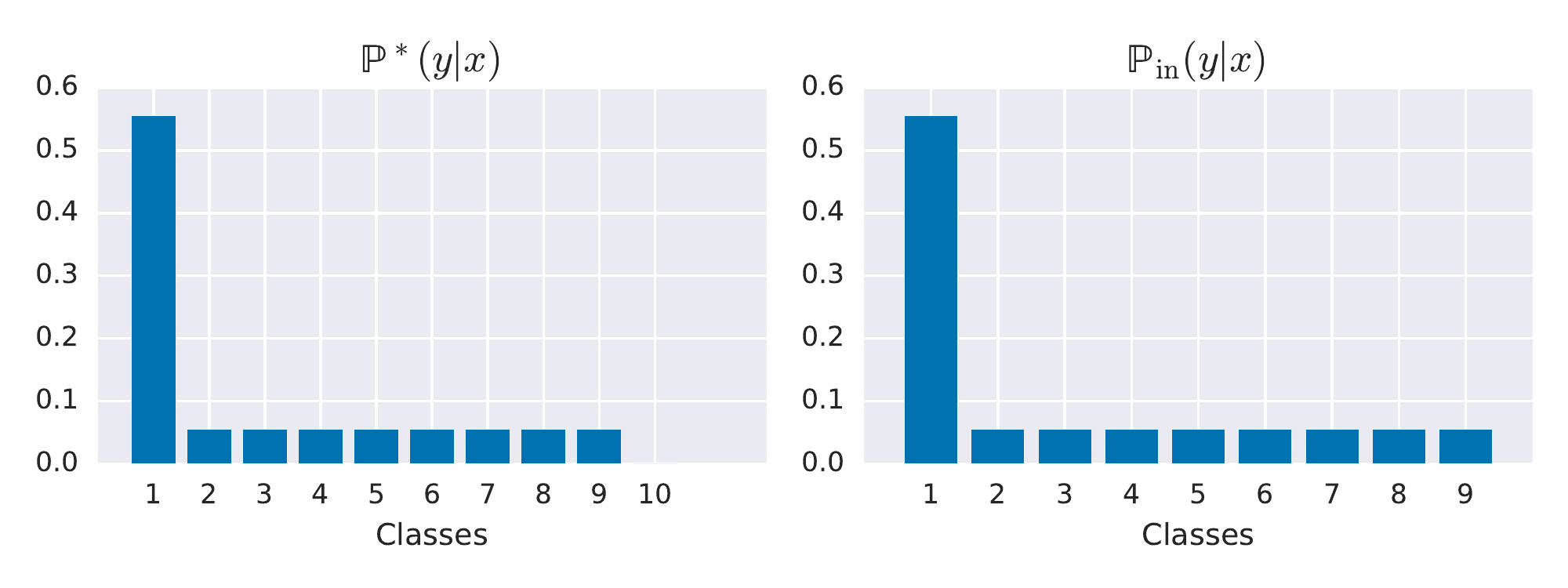}
        }
    }
    \vspace{-10pt}
    \caption{Examples of two open-set classification settings (a) and (b) with $L=10$ classes, where the inlier class distributions $\PTr( y \mid x )= \frac{\mathbb{P}_{\rm te}( y \mid x )}{\mathbb{P}_{\rm te}( y \neq 10 \mid x )}$ over the first 9 classes are identical, but the unknown class density $\Pr^*(10|x)$ is significantly different. Consequently, the MSP baseline, which relies only on the inlier class probabilities, will output the same rejection decision for both settings, whereas the Bayes-optimal classifier, which rejects by thresholding $\Pr^*(10|x)$, may output different decisions for the two settings.  }
    \label{fig:chow-fail-example}
\end{figure*}


%
\begin{figure*}[!t]
    \centering
    
    \resizebox{\linewidth}{!}{
        \subfigure[Uniform outlier distribution $\Pr_{\rm out}$.]{
            \includegraphics[scale=0.1825]{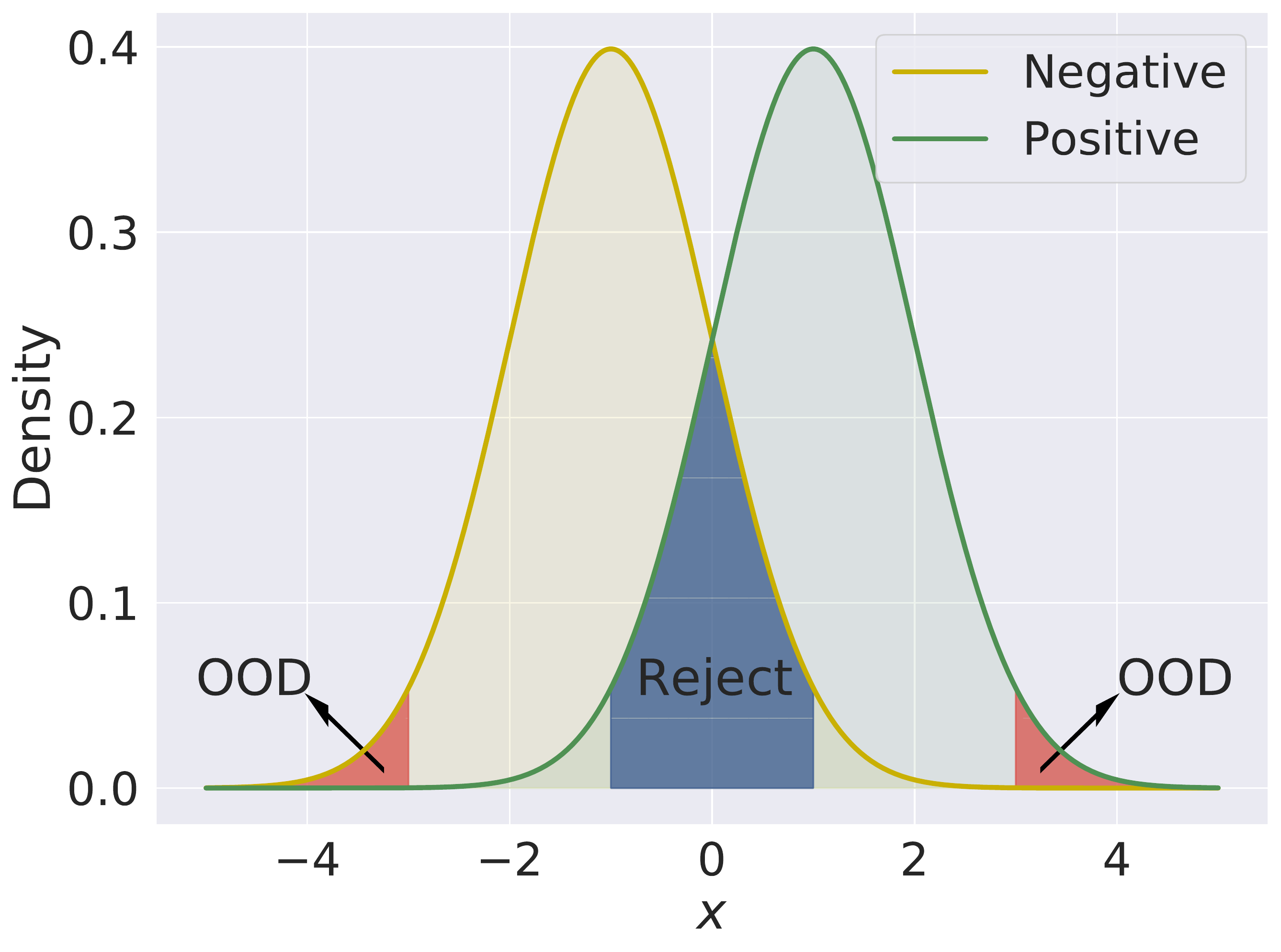}%
        }%
        \qquad
        \subfigure[Open-set classification.]{
            \label{fig:chow-fail-example-open-set}
            \includegraphics[scale=0.27]{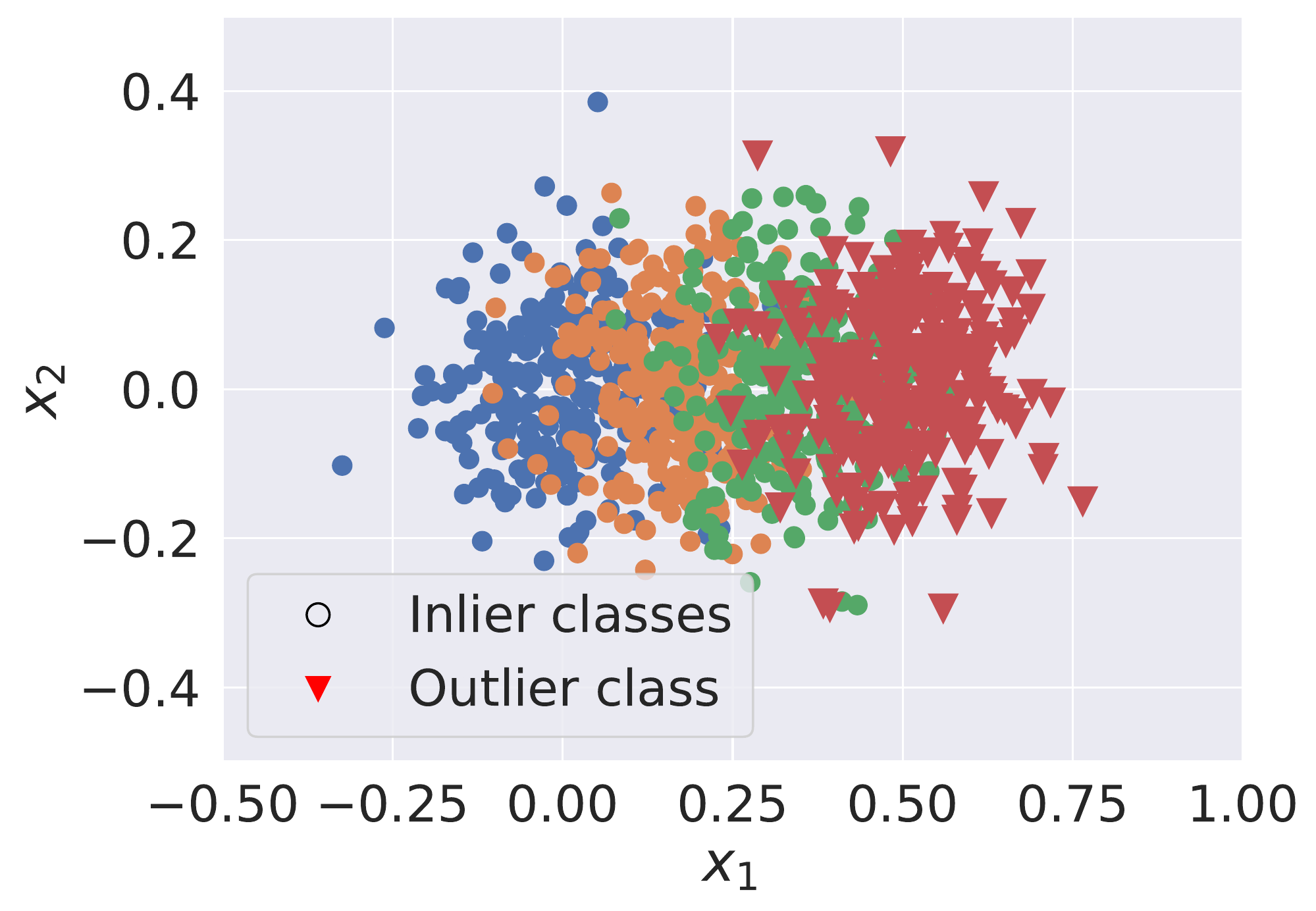}
            \includegraphics[scale=0.27]{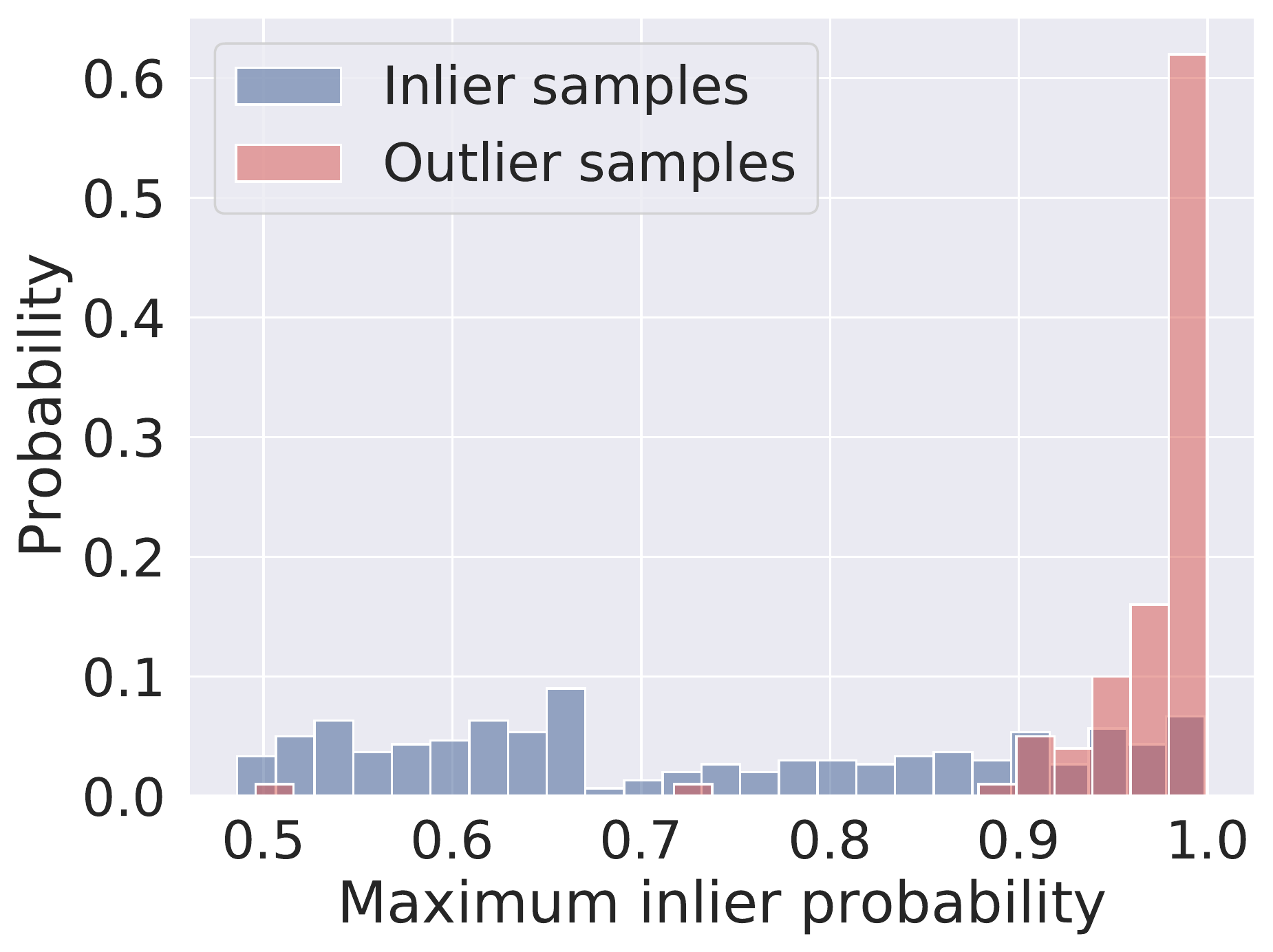}
        }%
    }
    
    \vspace{-5pt}
    \caption{
Example of two settings 
where the maximum softmax probability (MSP) baseline fails for OOD detection.
Setting {\bf (a)} considers
\emph{low-density OOD detection},
where
positive and negative samples drawn from a one-dimensional Gaussian distribution.
Samples \emph{away} from the origin will have $\Pr( x ) \sim 0$, 
and are thus outliers under the Bayes-optimal OOD detector.
However,
the MSP baseline will deem
samples \emph{near} the origin to be outliers,
as these have maximal 
$\max_{y} \Pr( y \mid x )$.
This illustrates the distinction between abstentions favoured by L2R 
(low label certainty)
and 
OOD detection 
(low density).
Setting {\bf (b)} considers    
\emph{open-set classification}
where there are $L = 4$ total classes, with the fourth class 
(denoted by ${\color{red} \blacktriangledown}$)
assumed to comprise outliers not seen during training.
    Each class-conditional is an isotropic Gaussian (left).
    Note that
    the maximum \emph{inlier} class-probability $\PTr( y \mid x )$ scores OOD samples significantly \emph{higher} than ID samples (right).    
    Thus, 
    the MSP baseline,
    which
    declares samples with low $\max_{y} \PTr( y \mid x )$ as outliers,
    will perform poorly.
    }
    \label{fig:chow-fail-synthetic-example}
    \vspace{-5pt}
\end{figure*}

\subsection{Illustration of maximum logit failure for open-set classification}
\label{app:expts-max-logit}

For the same setting as Figure~\ref{fig:chow-fail-synthetic-example},
we show 
in Figure~\ref{fig:max_inlier_logit_synthetic}
the maximum logit computed over the inlier distribution.
As with the maximum probability, the outlier samples tend to get a higher score than the inlier samples.

{For the same reason, rejectors that threshold the margin between the highest and the second-highest probabilities, instead of the maximum class probability, can also fail.
The use of other SC methods such as the cost-sensitive softmax cross-entropy~\citep{Mozannar:2020} may not be successful either, because the optimal solutions for these methods have the same form as MSP.}

\begin{figure}[!t]
    \centering
    \includegraphics[scale=0.325]{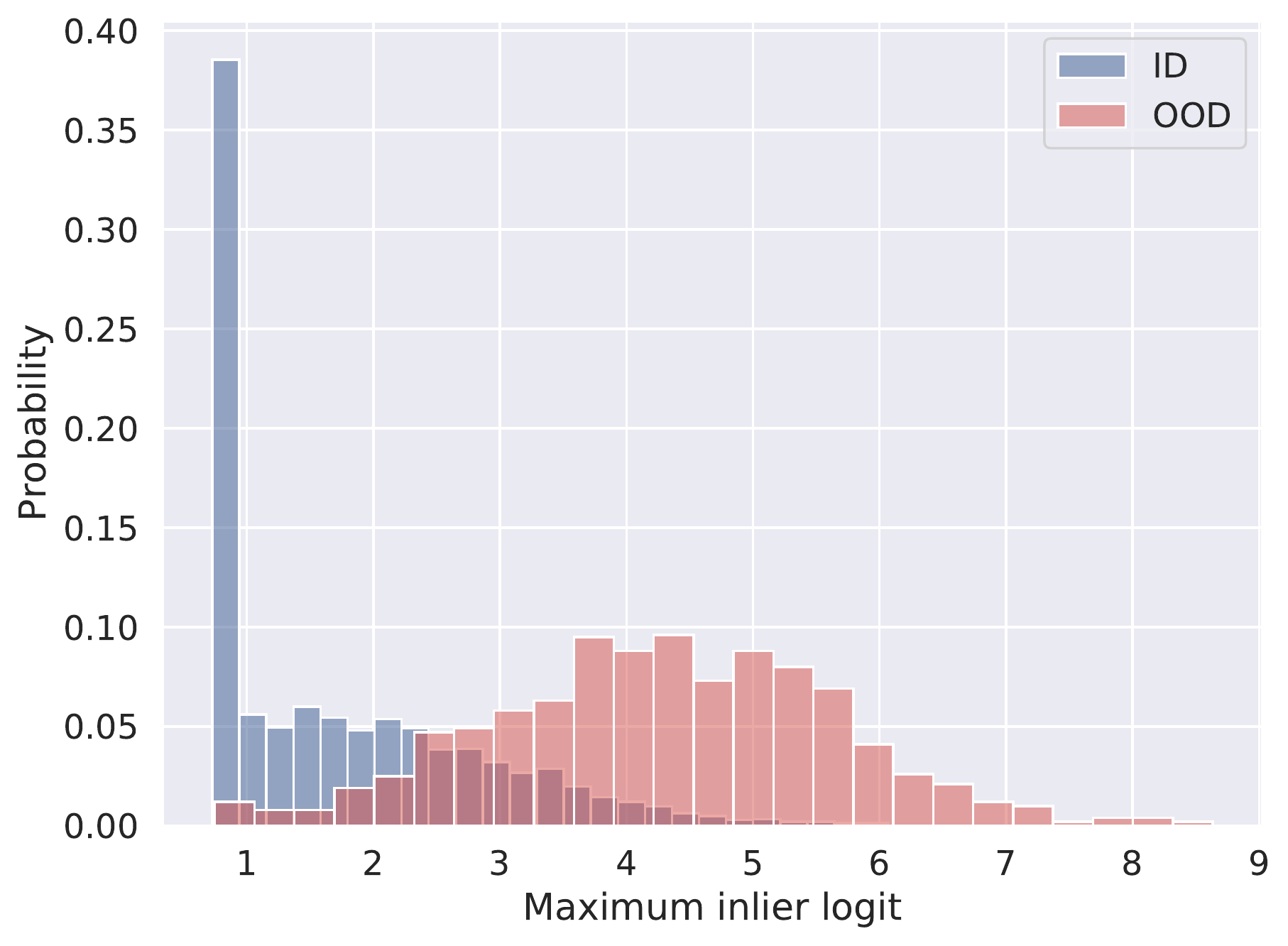}
    \caption{For the same setting as Figure~\ref{fig:chow-fail-synthetic-example},
we show 
the maximum logit computed over the inlier distribution.
As with the maximum probability, the outlier samples tend to get a higher score than the inlier samples.}
    \label{fig:max_inlier_logit_synthetic}
\vspace{-10pt}
\end{figure}

\section{Illustrating the impact of abstention costs}

\subsection{Impact of varying abstention costs $\costin, \costout$}
\label{app:expts-impact-alpha-beta}

Our joint objective that allows for abstentions on both ``hard'' and ``outlier'' samples is controlled by parameters $\costin, \costout$.
These reflect the costs on not correctly abstaining on samples from either class of anomalous sample.
Figure~\ref{fig:varying_alpha} and~\ref{fig:varying_beta} show the impact of varying these parameters while the other is fixed, for the synthetic open-set classification example of Figure~\ref{fig:chow-fail-example-open-set}.
The results are intuitive:
varying $\costin$ tends to favour abstaining on samples that are at the class boundaries,
while varying $\costout$ tends to favour abstaining on samples from the outlier class.
Figure~\ref{fig:varying_alpha_beta} confirms that when \emph{both} $\costin, \costout$ are varied, we achieve abstentions on both samples at the class boundaries, and samples from the outlier class.

\begin{figure}[!t]
    \centering
    
    \resizebox{\linewidth}{!}{
    \includegraphics[scale=0.325]{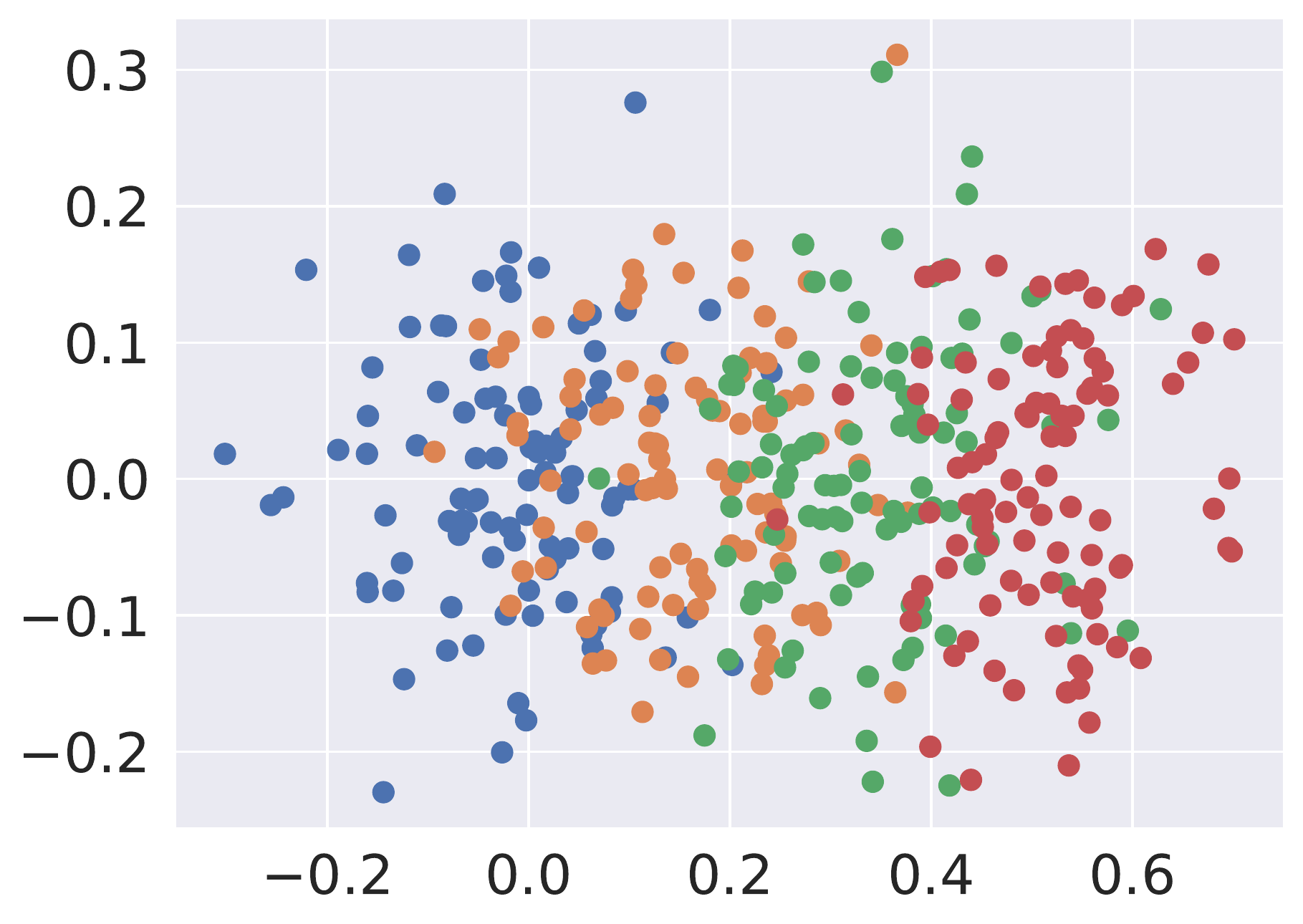}
    \includegraphics[scale=0.325]{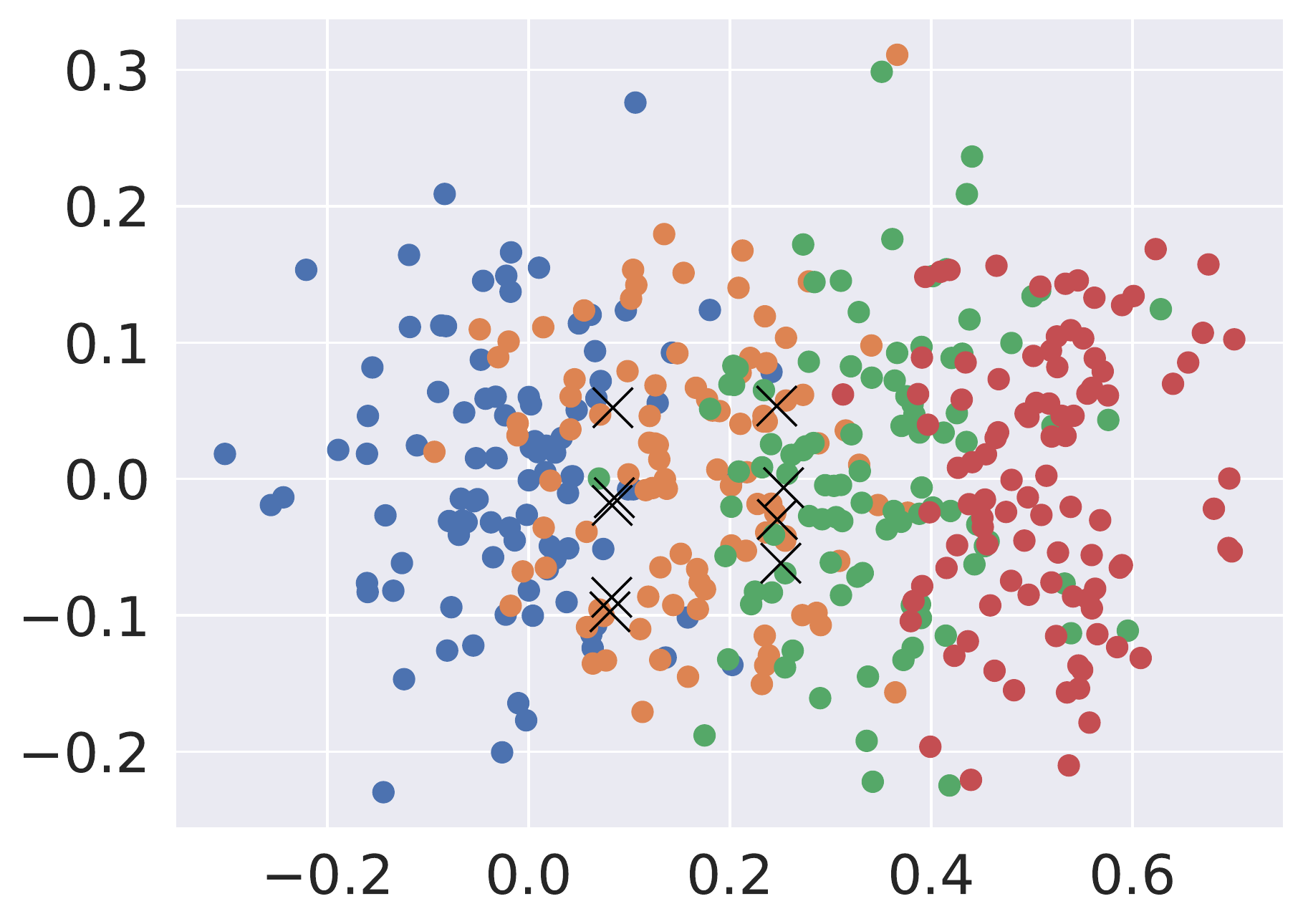}
    \includegraphics[scale=0.325]{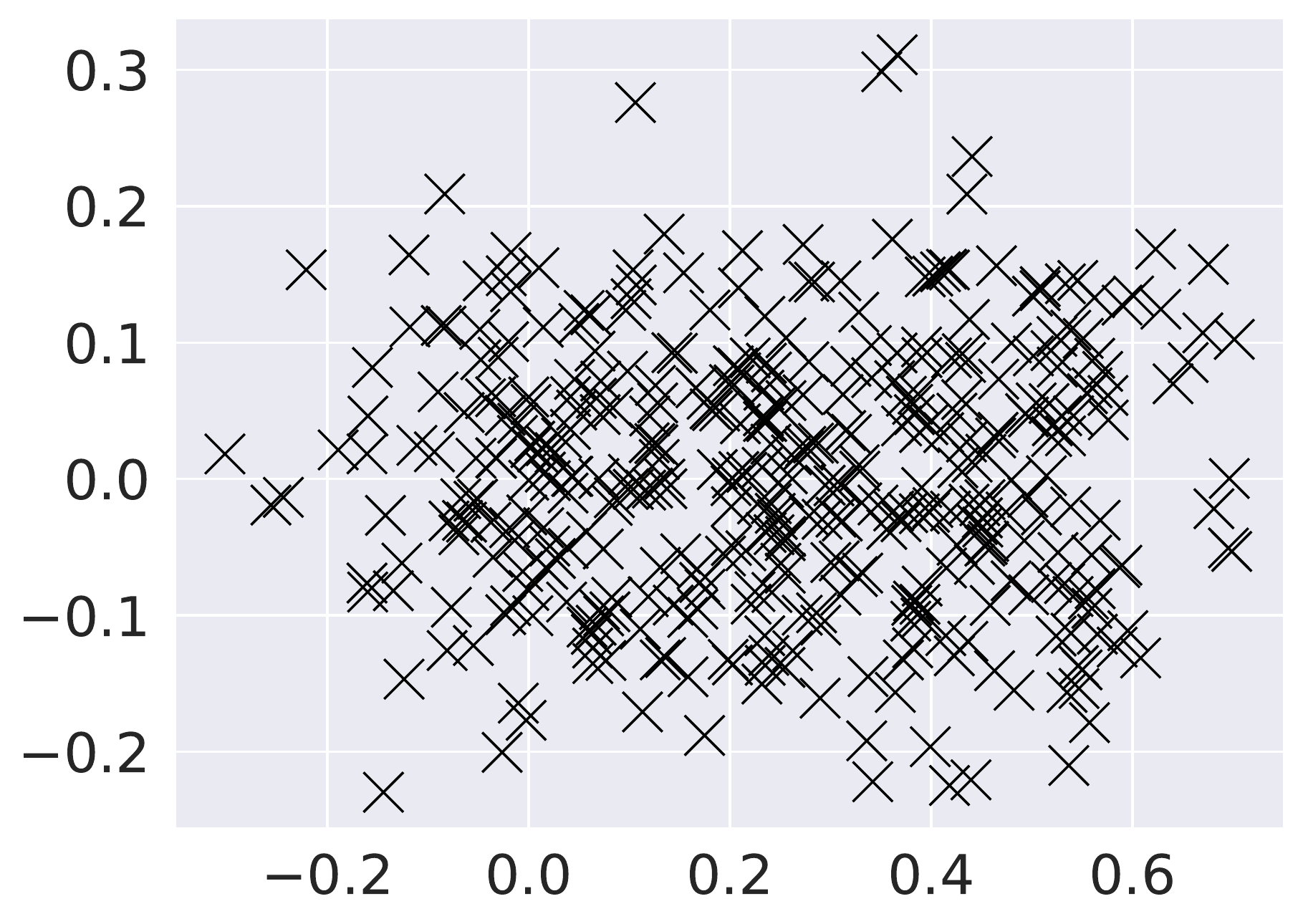}
    }
    
    \caption{Impact of varying $\costin$ for a fixed $\costout = 0.0$.
    The left plot shows the standard dataset, with $\costin = 1.0$.
    For intermediate $\costin = 0.5$ (middle), we abstain 
    (denoting by $\times$)
    only on the samples at the class boundaries.
    For $\costin = 0.0$ (right), we abstain on all samples.}
    \label{fig:varying_alpha}
\end{figure}

\begin{figure}[!t]
    \centering
    
    \resizebox{\linewidth}{!}{
    \includegraphics[scale=0.325]{figs/varying_rejection_threshold_a1.0_b0.0.pdf}
    \includegraphics[scale=0.325]{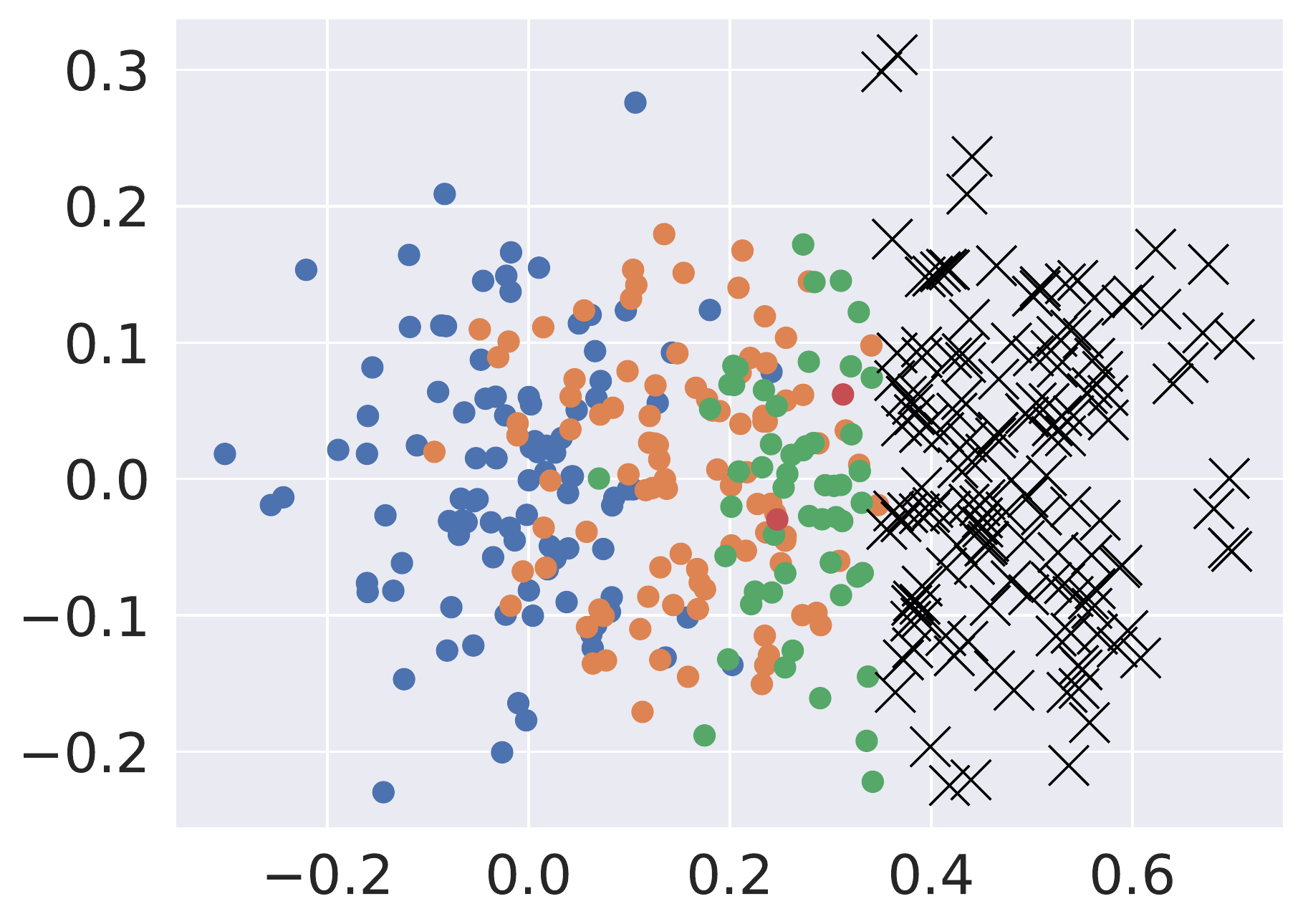}
    \includegraphics[scale=0.325]{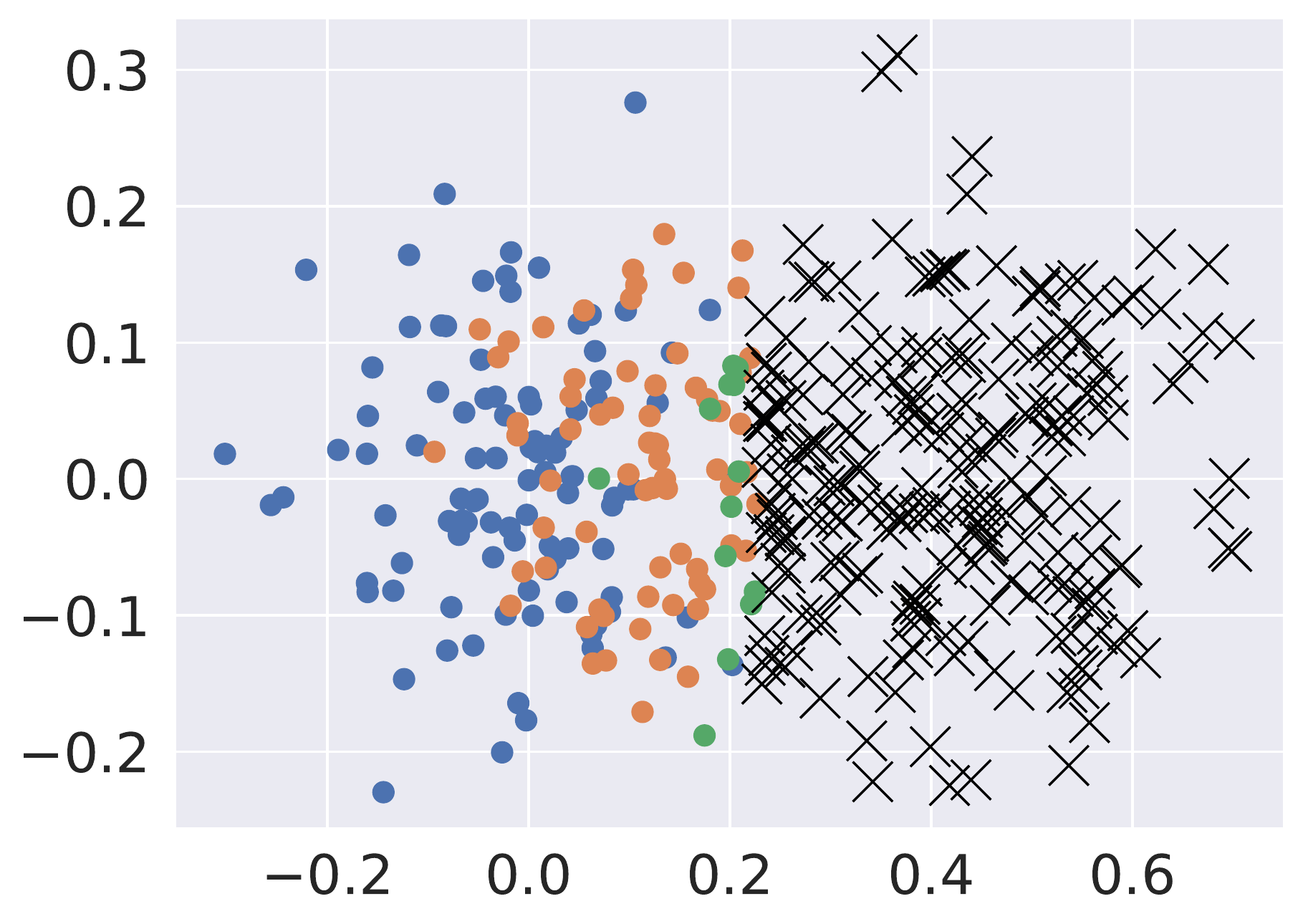}
    }
    
    \caption{Impact of varying $\costout$ for a fixed $\costin = 1.0$.
    The left plot shows the standard dataset, with $\costout = 0.0$.
    For intermediate $\costout = 1.0$ (middle), we abstain 
    (denoting by $\times$)
    only on the samples from the outlier class.
    For larger $\costout = 10.0$ (right), we start abstaining on inlier samples as well.}
    \label{fig:varying_beta}
\end{figure}

\begin{figure}[!t]
    \centering
    
    \resizebox{\linewidth}{!}{
    \includegraphics[scale=0.325]{figs/varying_rejection_threshold_a1.0_b0.0.pdf}
    \includegraphics[scale=0.325]{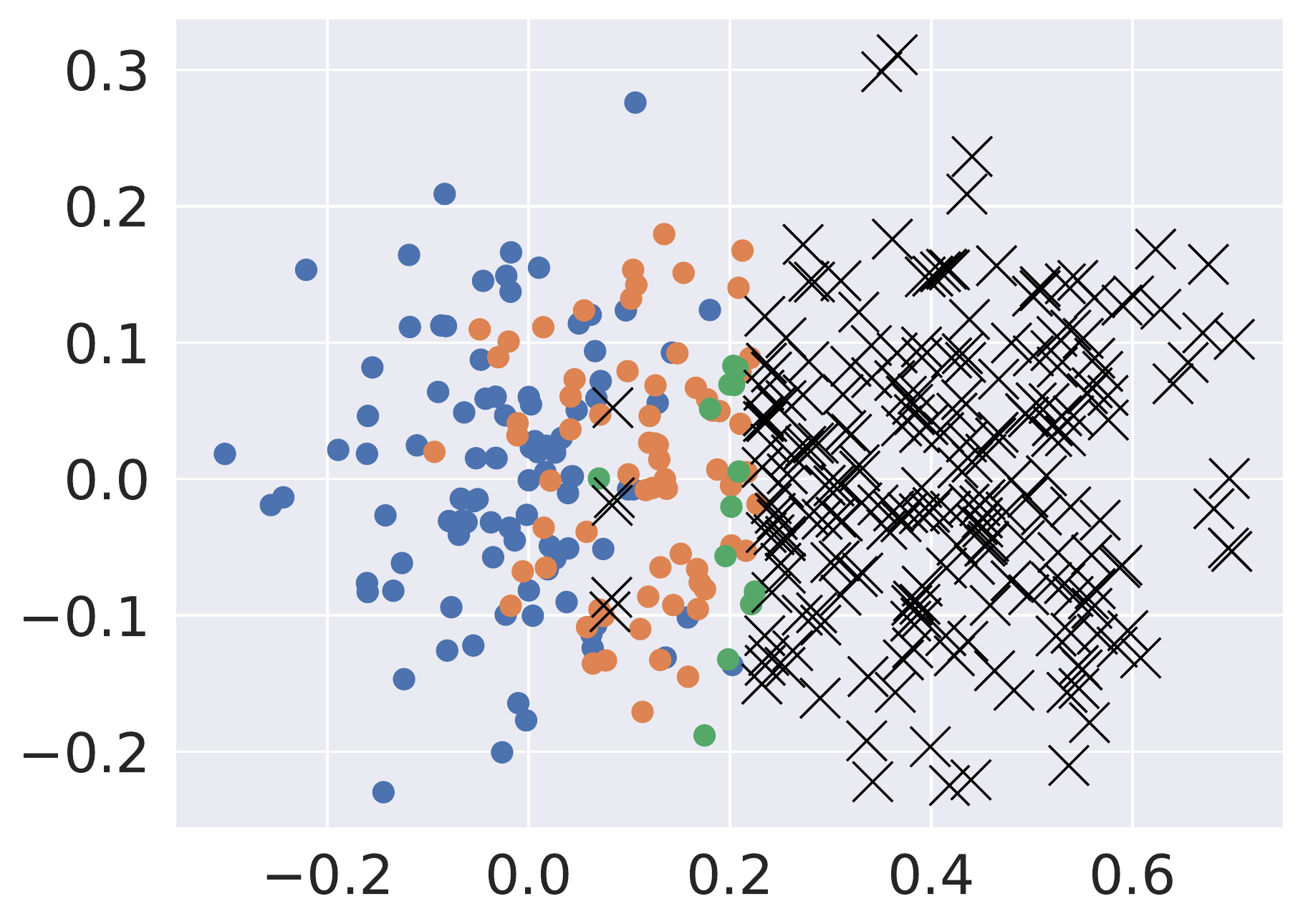}
    \includegraphics[scale=0.325]{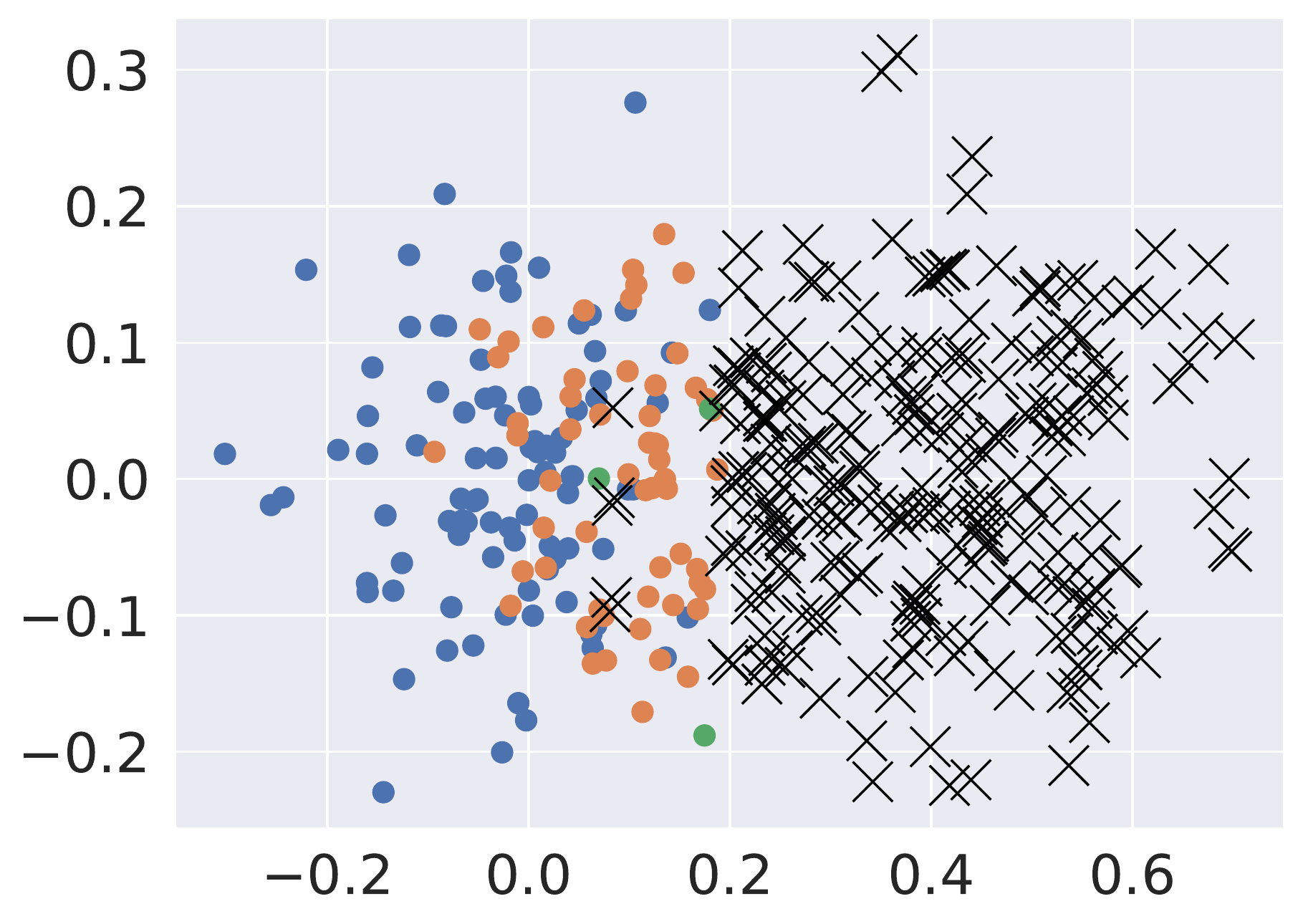}
    }
    
    \caption{Impact of varying both $\costin$ and $\costout$.
    The left plot shows the standard dataset, with $\costin = 1.0, \costout = 0.0$.
    Setting $\costin = 0.5, \costout = 1.0$ (middle)
    and $\costin = 0.5, \costout = 10.0$ (right)
    is shown to favour abstaining 
    (denoting by $\times$)
    on \emph{both} the samples at class boundaries, and the outlier samples.}
    \label{fig:varying_alpha_beta}
    \vspace{-10pt}
\end{figure}

\subsection{Impact of $\costout$ on OOD Detection Performance}
\label{app:expts-impact-beta}

For the same setting as Figure~\ref{fig:chow-fail-synthetic-example},
we consider the OOD detection performance of the
score 
$s( x ) = \max_{y \in [L]} \Pr_{\rm in}( y \mid x ) - \costout \cdot \frac{\Pr_{\rm in}( x )}{\Pr_{\rm out}( x )}$ 
as $\costout$ is varied.
Note that thresholding of this score determines the Bayes-optimal classifier. 
Rather than pick a fixed threshold, we use this score to compute the AUC-ROC for detecting whether a sample is from the outlier class, or not.
As expected, as $\costout$ increases
---
i.e., there is greater penalty on not rejecting an OOD sample
---
the AUC-ROC improves.

\begin{figure}[!t]
    \centering
    \includegraphics[scale=0.3]{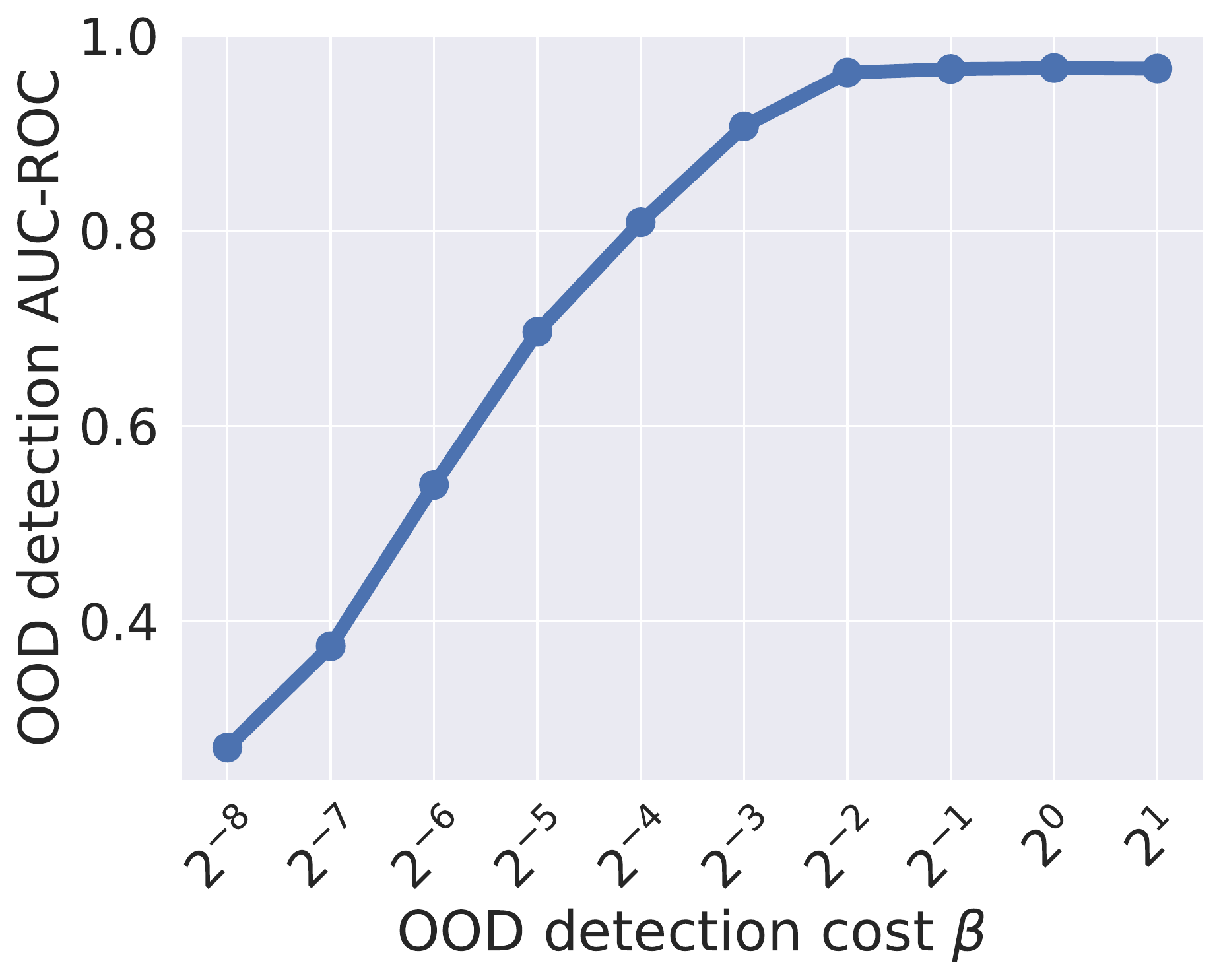}
    \caption{For the same setting as Figure~\ref{fig:chow-fail-synthetic-example},
we consider the OOD detection performance of the
score $s( x ) = \max_{y \in [L]} \Pr_{\rm in}( y \mid x ) - \costout \cdot \frac{\Pr_{\rm in}( x )}{\Pr_{\rm out}( x )}$ as $\costout$ is varied.
Specifically, we use this score to compute the AUC-ROC for detecting whether a sample is from the outlier class, or not.
As expected, as $\costout$ increases, the AUC-ROC improves.}
\vspace{-5pt}
    \label{fig:auc_instance_chow_synthetic}
\end{figure}

\section{Limitations and broader impact}
\label{app:limitations}
Recall that our proposed plug-in rejectors  seek to optimize for overall classification and OOD detection accuracy while keeping the total fraction of abstentions within a limit. However, the improved overall accuracy may come at the cost of poorer performance on smaller sub-groups. For example, \citet{jonesselective} show that Chow's rule or the MSP scorer ``can magnify existing accuracy disparities between various groups within a population, especially in the presence of spurious correlations''. 
It would be of interest to carry out a similar study with the two plug-in based rejectors proposed in this paper, and to  understand how both their inlier classification accuracy and their OOD detection performance varies across sub-groups. It would also be of interest to explore variants of our proposed rejectors that mitigate such disparities among sub-groups.  

Another limitation of our proposed plug-in rejectors is that they are only as good as the estimators we use for the density ratio $\frac{\Pr_{\rm in}(x)}{\Pr_{\rm out}(x)}$. When our estimates of the density ratio are not accurate, the plug-in rejectors are seen to often perform worse than the SIRC baseline that use the same estimates. Exploring better ways  for estimating the density ratio is an important direction for future work.

Beyond SCOD, the proposed rejection strategies are also applicable to the growing literature on adaptive inference \cite{Liu:2020}.  With the wide adoption of large-scale machine learning models with billions of parameters,  it is becoming increasingly important that we are able to perform speed up the inference time for these models. To this end, adaptive inference strategies have gained popularity, wherein one varies the amount of compute the model spends on an example, by for example, exiting early on ``easy'' examples.
The proposed approaches for SCOD may be adapted to equip early-exit models to not only exit early on high-confidence ``easy'' samples, but also exit early on samples that are deemed to be outliers. In the future, it would be interesting to explore the design of such early-exit models that are equipped with an OOD detector to aid in their routing decisions.
\vspace{-5pt}
\end{document}